\documentclass[
  twoside,
]{article}

\usepackage[utf8]{inputenc}
\usepackage[T1]{fontenc}

\usepackage{ifdraft}

\usepackage[accepted]{aistats2025}

\makeatletter
\gdef\@runningauthor{
  Marvin Pförtner,
  Jonathan Wenger,
  Jon Cockayne,
  Philipp Hennig
}
\makeatother

\usepackage{microtype}

\usepackage[inline]{enumitem}

\usepackage[final]{hyperref}
\hypersetup{
  colorlinks=true,
  allcolors=black}
\usepackage[numbered]{bookmark}  %

\usepackage[round,sort&compress]{natbib}
\usepackage{doi}

\usepackage[nottoc, section]{tocbibind}

\usepackage{fontawesome}

\usepackage{siunitx}

\usepackage{wrapfig}

\usepackage{caption}

\usepackage[
  labelformat=simple
]{subcaption}

\usepackage[final]{graphicx}

\graphicspath{%
  {figures/}%
}

\usepackage{xcolor}

\definecolor{TUred}{RGB}{165,30,55}
\definecolor{TUgold}{RGB}{180,160,105}
\definecolor{TUdark}{RGB}{50,65,75}
\definecolor{TUgray}{RGB}{175,179,183}

\definecolor{TUdarkblue}{RGB}{65,90,140}
\definecolor{TUblue}{RGB}{0,105,170}
\definecolor{TUlightblue}{RGB}{80,170,200}
\definecolor{TUlightgreen}{RGB}{130,185,160}
\definecolor{TUgreen}{RGB}{125,165,75}
\definecolor{TUdarkgreen}{RGB}{50,110,30}
\definecolor{TUocre}{RGB}{200,80,60}
\definecolor{TUviolet}{RGB}{175,110,150}
\definecolor{TUmauve}{RGB}{180,160,150}
\definecolor{TUbeige}{RGB}{215,180,105}
\definecolor{TUorange}{RGB}{210,150,0}
\definecolor{TUbrown}{RGB}{145,105,70}

\definecolor{MPLblue}{HTML}{1f77b4}
\definecolor{MPLorange}{HTML}{ff7f0e}
\definecolor{MPLgreen}{HTML}{2ca02c}
\definecolor{MPLred}{HTML}{d62728}
\definecolor{MPLpurple}{HTML}{9467bd}

\definecolor{SNSblue}{rgb}{0.1216, 0.4666, 0.7059}
\definecolor{SNSorange}{rgb}{1.0, 0.4980, 0.0549}
\definecolor{SNSgreen}{rgb}{0.1725, 0.6274, 0.1725}
\definecolor{SNSred}{rgb}{0.84, 0.15, 0.16}
\definecolor{SNSpurple}{rgb}{0.58, 0.40, 0.74}

\definecolor{SNSblue_shaded}{HTML}{8ebad9}
\definecolor{SNSorange_shaded}{HTML}{ffcea3}
\definecolor{SNSgreen_shaded}{HTML}{cae7ca}
\definecolor{SNSred_shaded}{HTML}{ea9293}

\usepackage{tikz}

\usetikzlibrary{positioning}

\usetikzlibrary{arrows,shapes,plotmarks,decorations.pathmorphing,automata,chains,fit,backgrounds,calc,positioning,fadings}

\tikzset{every picture/.style={node distance=2cm,shorten >= 1pt, shorten <= 1pt}}
\tikzset{>=stealth'}
\tikzstyle{graphnode} =
[circle,draw=TUdark,minimum size=22pt,text centered,text width=22pt,inner sep=0pt, text=TUdark]
\tikzstyle{var}   =[graphnode,fill=white]
\tikzstyle{obs}   =[graphnode,fill=TUdarkblue,text=white]
\tikzstyle{act}   =[rectangle,draw=TUdark,text=white,minimum
size=22pt,text centered, text width=22pt,inner sep=0pt]
\tikzstyle{fac}   =[rectangle,draw=TUdark,fill=TUgreen,minimum size=5pt]
\tikzstyle{facprior} =[rectangle,draw=TUdark,fill=TUdark,text=white,minimum size=5pt]
\tikzstyle{edge}  =[draw=TUdark,-]
\tikzstyle{prior} =[rectangle, draw=TUdark, fill=TUdark, minimum size=
5pt, inner sep=0pt]
\tikzstyle{dirprior} = [circle, draw=TUdark, fill=TUdark, minimum
size=5pt, inner sep=0pt]

\tikzfading[name=fade top,bottom color=transparent!0,top color=transparent!75]
\usepackage{booktabs}
\usepackage{multirow}
\usepackage{titletoc}  %
\usepackage{algorithm}
\usepackage[indLines=true]{algpseudocodex}

\algrenewcommand{\algorithmicfunction}{\textbf{fn}}
\newcommand{\just}[2]{\phantom{#1}\mathllap{#2}}
\usepackage{amsmath}
\usepackage{mathtools}

\usepackage{amssymb}
\usepackage{bm}
\usepackage{nicefrac}

\usepackage{amsthm}
\usepackage{thmtools}
\usepackage{thm-restate}

\usepackage{etoolbox}
\usepackage{xparse}

\allowdisplaybreaks    %

\numberwithin{equation}{section}  %

\DeclarePairedDelimiter{\ps}{(}{)}
\DeclarePairedDelimiter{\brks}{[}{]}

\newcommand*{\defeq}{\coloneqq}  %
\newcommand*{\rdefeq}{\eqqcolon}  %

\DeclarePairedDelimiterX{\set}[1]\{\}{%

#1}

\newcommand*{\setsym}[1]{\ensuremath{{\mathbb{#1}}}}

\newcommand*{\R}{\setsym{R}}

\newcommand*{\inv}{^{-1}}

\newcommand*{\spacesym}[1]{{\mathbb{#1}}}  %

\newcommand*{\closure}[1]{\overline{#1}}

\newcommand*{\linspan}[1]{\operatorname{span} \left( #1 \right)}

\newcommand*{\pvec}[1]{
  \begin{pmatrix}
    #1
  \end{pmatrix}
}

\renewcommand*{\vec}[1]{{\bm{#1}}}

\newcommand*{\vb}{\vec{b}}
\newcommand*{\vc}{\vec{c}}
\newcommand*{\vd}{\vec{d}}
\newcommand*{\ve}{\vec{e}}

\newcommand*{\vm}{\vec{m}}

\newcommand*{\vr}{\vec{r}}
\newcommand*{\vs}{\vec{s}}
\newcommand*{\vt}{\vec{t}}
\newcommand*{\vu}{\vec{u}}
\newcommand*{\vv}{\vec{v}}
\newcommand*{\vw}{\vec{w}}
\newcommand*{\vx}{\vec{x}}
\newcommand*{\vy}{\vec{y}}
\newcommand*{\vz}{\vec{z}}

\newcommand*{\vbeta}{\vec{\beta}}

\newcommand*{\vmu}{\vec{\mu}}

\newcommand*{\vxi}{\vec{\xi}}

\NewDocumentCommand{\range}{s m}{%
  \operatorname{ran}
  \IfBooleanTF{#1}{\left(}{(}
  #2
  \IfBooleanTF{#1}{\right)}{)}
}

\NewDocumentCommand{\kernel}{s m}{%
  \operatorname{ker}
  \IfBooleanTF{#1}{\left(}{(}
  #2
  \IfBooleanTF{#1}{\right)}{)}
}

\newcommand*{\pmat}[1]{
  \begin{pmatrix}
    #1
  \end{pmatrix}
}

\newcommand*{\mat}[1]{{\bm{#1}}}

\newcommand*{\mA}{\mat{A}}
\newcommand*{\mB}{\mat{B}}
\newcommand*{\mC}{\mat{C}}

\newcommand*{\mG}{\mat{G}}
\newcommand*{\mH}{\mat{H}}
\newcommand*{\mI}{\mat{I}}

\newcommand*{\mK}{\mat{K}}

\newcommand*{\mM}{\mat{M}}
\newcommand*{\mN}{\mat{N}}

\newcommand*{\mP}{\mat{P}}
\newcommand*{\mQ}{\mat{Q}}

\newcommand*{\mS}{\mat{S}}

\newcommand*{\mV}{\mat{V}}
\newcommand*{\mW}{\mat{W}}
\newcommand*{\mX}{\mat{X}}

\newcommand*{\mZ}{\mat{Z}}

\newcommand*{\mLambda}{\mat{\Lambda}}

\newcommand*{\mSigma}{\mat{\Sigma}}

\makeatletter
\DeclarePairedDelimiterXPP{\@norm}[2]{}{\lVert}{\rVert}{\ifstrempty{#2}{}{_{#2}}}{#1}
\NewDocumentCommand{\norm}{s O{} m O{}}{%
  \IfBooleanTF{#1}{%
    \@norm*{#3}{#4}%
  }{%
    \@norm[#2]{#3}{#4}%
  }%
}
\makeatother

\DeclarePairedDelimiter{\abs}{\lvert}{\rvert}

\newcommand*{\dualop}{'}

\newcommand*{\hilbertsp}[1][H]{\spacesym{#1}}  %

\makeatletter
\DeclarePairedDelimiterXPP{\@inprod}[3]{}{\langle}{\rangle}{\ifstrempty{#3}{}{_{#3}}}{#1, #2}
\NewDocumentCommand{\inprod}{s O{} m m O{}}{%
  \IfBooleanTF{#1}{%
    \@inprod*{#3}{#4}{#5}%
  }{%
    \@inprod[#2]{#3}{#4}{#5}%
  }%
}
\makeatother

\newcommand*{\T}{^\top}

\newcommand*{\pinv}{^\dagger}

\NewDocumentCommand{\cfns}{m o}{%
  C (#1\IfValueT{#2}{, #2})
}

\NewDocumentCommand{\cdfns}{O{0} m o}{%
  C^{#1} (#2\IfValueT{#3}{, #3})
}

\NewDocumentCommand{\holdersp}{m o m}{%
  C^{#1\IfValueT{#2}{, #2}}(\closure{#3})
}

\RenewDocumentCommand{\L}{m o}{%
  \ensuremath{
    L_{#1}
    \IfNoValueF{#2}{\left( #2 \right)}
  }
}

\NewDocumentCommand{\sobolev}{o m o}{%
  \IfNoValueTF{#1}{
    H^{#2}
  }{
    W^{#1,#2}
  }
  \IfNoValueF{#3}{\left( #3 \right)}
}

\NewDocumentCommand{\sobolevtest}{o m o}{%
  \sobolev[#1]{#2}_0
  \IfNoValueF{#3}{\left( #3 \right)}
}

\newcommand*{\rkhs}[1]{\hilbertsp_{#1}}

\newcommand*{\linop}[1]{\mathcal{#1}}

\newcommand*{\linfctls}[1]{{\bm{\linop{#1}}}}

\makeatletter
\DeclarePairedDelimiter{\@evallinop@oparg}{[}{]}
\DeclarePairedDelimiter{\@evallinop@fnarg}{(}{)}
\NewDocumentCommand{\evallinop}{m s O{} m s O{} d()}{%
  #1%
  \IfBooleanTF{#2}{%
    \@evallinop@oparg*{#4}%
  }{%
    \@evallinop@oparg[#3]{#4}%
  }%
  \IfValueT{#7}{%
    \IfBooleanTF{#5}{%
      \@evallinop@fnarg*{#7}%
    }{%
      \@evallinop@fnarg[#6]{#7}%
    }%
  }%
}
\makeatother

\NewDocumentCommand{\linopat}{m s O{} m d()}{%
  \IfBooleanTF{#2}{%
    \evallinop{\linop{#1}}*{#4}(#5)%
  }{%
    \evallinop{\linop{#1}}[#3]{#4}(#5)%
  }%
}

\NewDocumentCommand{\linfctlsat}{m s O{} m d()}{%
  \IfBooleanTF{#2}{%
    \evallinop{\linfctls{#1}}*{#4}(#5)%
  }{%
    \evallinop{\linfctls{#1}}[#3]{#4}(#5)%
  }%
}

\newcommand*{\diff}[2][1]{%
  \mathrm{d} #2\ifstrequal{#1}{1}{}{^{#1}}%
}

\newcommand*{\pdiff}[2][1]{%
  \partial #2\ifstrequal{#1}{1}{}{^{#1}}%
}

\NewDocumentCommand{\deriv}{s O{1} m m}{%
  \frac{%
    \mathrm{d}\ifstrequal{#2}{1}{}{^{#2}}\IfBooleanT{#1}{#4}
  }{%
    \diff[#2]{#3}
  }
  \IfBooleanF{#1}{#4}
}

\NewDocumentCommand{\derivat}{s O{1} m m m}{%
\left.
\IfBooleanTF{#1}{%
  \deriv*[#2]{#3}{#4}
}{%
  \deriv[#2]{#3}{#4}
}
\right|_{#3 = #5}
}

\NewDocumentCommand{\dderiv}{m m}{%
  \partial_{#1} #2
}

\NewDocumentCommand{\dderivat}{m m o m}{%
  \IfNoValueTF{#3}{%
    \dderiv{#1}{#2} \left( #4 \right)
  }{%
    \left.
    \dderiv{#1}{#2}
    \right_{#3 = #4}
  }
}

\NewDocumentCommand{\pderiv}{s O{1} m m}{%
  \frac{%
    \partial\ifstrequal{#2}{1}{}{^{#2}}\IfBooleanT{#1}{#4}
  }{%
    \pdiff[#2]{#3}
  }
  \IfBooleanF{#1}{#4}
}

\NewDocumentCommand{\pderivat}{s O{1} m m o m}{%
\left.
\IfBooleanTF{#1}{%
  \pderiv*[#2]{#3}{#4}
}{%
  \pderiv[#2]{#3}{#4}
}
\right|_{\IfNoValueTF{#5}{#3}{#5} = #6}
}

\NewDocumentCommand{\mpderiv}{s O{1} m m}{%
  \frac{%
    \partial\ifstrequal{#2}{1}{}{^{#2}}\IfBooleanT{#1}{#4}
  }{%
    #3
  }
  \IfBooleanF{#1}{#4}
}

\NewDocumentCommand{\mpderivat}{s O{1} m m m m}{%
\left.
\IfBooleanTF{#1}{%
  \mpderiv*[#2]{#3}{#4}
}{%
  \mpderiv[#2]{#3}{#4}
}
\right|_{#5 = #6}
}

\NewDocumentCommand{\mipderiv}{s m o m o}{%
  \IfNoValueTF{#3}{
    \mathrm{D}^{#2}
    \IfBooleanTF{#1}{\left[ #4 \right]}{#4}
    \IfNoValueF{#5}{\left( #5 \right)}
  }{
    \IfNoValueF{#5}{\left.}
    \frac{%
      \partial^{\lvert #2 \rvert}\IfBooleanT{#1}{#4}
    }{%
      \pdiff[#2]{#3}
    }
    \IfBooleanF{#1}{#4}
    \IfNoValueF{#5}{\right\vert_{#3 = #5}}
  }
}

\NewDocumentCommand{\jacobian}{o m o}{%
  \IfNoValueTF{#1}{%
    \mathrm{D} #2 \IfNoValueF{#3}{\left( #3 \right)}
  }{%
    \left.
      \mathrm{D} #2
    \right|_{#1\IfNoValueF{#3}{= #3}}
}
}

\NewDocumentCommand{\gradient}{o m o}{%
  \IfNoValueTF{#1}{%
    \nabla #2 \IfNoValueF{#3}{\left( #3 \right)}
  }{%
    \left.
      \nabla #2
    \right|_{#1\IfNoValueF{#3}{= #3}}
}
}

\NewDocumentCommand{\divergence}{m o}{%
  \operatorname{div} \left( #1 \right) \IfNoValueF{#2}{\left( #2 \right)}
}

\NewDocumentCommand{\hessian}{o m o}{%
  \IfNoValueTF{#1}{%
    H #2 \IfNoValueF{#3}{\left( #3 \right)}
  }{%
    \left.
      H #2
    \right|_{#1\IfNoValueF{#3}{= #3}}
}
}

\NewDocumentCommand{\laplaceop}{o m o}{%
  \IfNoValueTF{#1}{%
    \Delta #2 \IfNoValueF{#3}{\left( #3 \right)}
  }{%
    \left.
      \Delta #2
    \right|_{#1\IfNoValueF{#3}{= #3}}
}
}

\NewDocumentCommand{\rintegral}{o o m m}{%
  \int\IfNoValueF{#1}{_{#1}}\IfNoValueF{#2}{^{#2}} #4 \,\diff[1]{#3}%
}

\makeatletter

\newcommand*{\@probsymbol}{\mathrm{P}}

\newcommand*{\@given}[1]{%
  \nonscript\:#1\vert
  \allowbreak
  \nonscript\:
  \mathopen{}}

\providecommand*{\given}{}

\DeclarePairedDelimiterXPP{\@prob}[1]{\@probsymbol}{(}{)}{}{%
  \renewcommand*{\given}{\@given{\delimsize}}%
  #1}

\newcommand*{\prob}[1]{\ifblank{#1}{\@probsymbol}{\@prob*{#1}}}

\makeatother

\makeatletter

\DeclarePairedDelimiterX{\@condrv}[1]{.}{.}{%
  \renewcommand*{\given}{\@given{\delimsize}}%
  #1}

\NewDocumentCommand{\condrv}{som}{%
  \IfBooleanTF{#1}{%
    \@condrv*{#3}
  }{%
    \IfNoValueTF{#2}{%
      \begingroup%
      \renewcommand*{\given}{\@given{}}%
      #3%
      \endgroup%
    }{%
      \@condrv[#2]{#3}%
    }
  }
}

\makeatother

\newcommand*{\rvec}[1]{{\bm{\mathrm{#1}}}}

\newcommand*{\rva}{\rvec{a}}
\newcommand*{\rvb}{\rvec{b}}

\newcommand*{\rvq}{\rvec{q}}

\newcommand*{\rvu}{\rvec{u}}

\newcommand*{\rvw}{\rvec{w}}
\newcommand*{\rvx}{\rvec{x}}
\newcommand*{\rvy}{\rvec{y}}

\newcommand*{\rvepsilon}{\rvec{\epsilon}}

\NewDocumentCommand{\expectation}{o o m}{%
  \operatorname{\mathbb{E}}\IfNoValueF{#1}{_{#1\IfNoValueF{#2}{\sim #2}}} \left[ #3 \right]
}
\NewDocumentCommand{\covariance}{o o m m}{%
  \operatorname{Cov}\IfNoValueF{#1}{_{#1\IfNoValueF{#2}{\sim #2}}} \left[ #3, #4 \right]
}
\NewDocumentCommand{\variance}{o o m}{%
  \operatorname{\mathbb{V}}\IfNoValueF{#1}{_{#1\IfNoValueF{#2}{\sim #2}}} \left[ #3 \right]
}

\newcommand*{\gaussian}[2]{{\ensuremath{\operatorname{\mathcal{N}}\left(#1, #2\right)}}}
\newcommand*{\gaussianpdf}[3]{%
  {\ensuremath{\operatorname{\mathcal{N}}\left(#1; #2, #3\right)}}%
}

\newcommand*{\rproc}[1]{{\mathrm{#1}}}
\newcommand*{\morproc}[1]{{\bm{\mathrm{#1}}}}

\newcommand*{\gp}[2]{{\ensuremath{\operatorname{\mathcal{GP}}}\left(#1, #2\right)}}

\NewDocumentCommand{\LkL}{m m o}{%
  #1 #2 \IfValueTF{#3}{#3}{#1}\dualop
}

\declaretheorem[style=plain]{theorem}
\declaretheorem[style=plain,sibling=theorem]{proposition}
\declaretheorem[style=plain,sibling=theorem]{lemma}
\declaretheorem[style=plain,sibling=theorem]{corollary}
\declaretheorem[style=definition,sibling=theorem]{definition}

\declaretheorem[style=remark,sibling=theorem]{remark}

\newcommand*{\iidx}[1][i]{^{(#1)}}  %

\usepackage[
  obeyDraft,
  textsize=tiny,
  figwidth=0.99\linewidth,
]{todonotes}

\newcommand{\gmp}{\rvu}
\newcommand{\gmpval}{\vu}
\newcommand{\gmptarget}{\gmpval^\star}
\newcommand{\gmpdim}{D}
\newcommand{\gmplen}{K}  %
\newcommand{\gmpA}{\mA}
\newcommand{\gmpb}{\vb}
\newcommand{\gmpnoise}{\rvq}
\newcommand{\gmpnoisecov}{\mQ}
\newcommand{\gmpmean}{\vmu}
\newcommand{\gmpcov}{\mSigma}
\newcommand{\gmppstate}[1]{\gmp^-_{#1}}
\newcommand{\gmptstate}[1]{\gmp^+_{#1}}

\newcommand{\obs}{\vy}
\newcommand{\obsrv}{\rvy}
\newcommand{\obsdim}{N}
\newcommand{\obsH}{\mH}
\newcommand{\obsnoise}{\rvepsilon}
\newcommand{\obsnoisecov}{\mLambda}

\newcommand{\cakfpmean}{\hat{\vm}^-}  %
\newcommand{\cakfpdd}{\hat{\mM}^-}  %
\newcommand{\cakfpcov}{\hat{\mP}^-}  %
\newcommand{\cakfmean}{\hat{\vm}}  %
\newcommand{\cakfdd}{\hat{\mM}}  %
\newcommand{\cakfddrank}{r}  %
\newcommand{\cakfcov}{\hat{\mP}}  %
\newcommand{\cakfmaxtddrank}{r^\textnormal{max}}  %
\newcommand{\cakftdd}{\hat{\mM}^+}  %
\newcommand{\cakftddrank}{r^+}  %
\newcommand{\cakftcov}{\hat{\mP}^+}  %
\newcommand{\cakfrepw}{\hat{\vv}}
\newcommand{\cakfrepwdd}{\hat{\mV}}
\newcommand{\cakfw}{\hat{\vw}}
\newcommand{\cakfW}{\hat{\mW}}
\newcommand{\cakfact}{\hat{\vs}}  %
\newcommand{\cakfacts}{\hat{\mS}}  %
\newcommand{\cakfgram}{\hat{\mG}}
\newcommand{\cakfresidual}{\hat{\vr}}

\newcommand{\cakfprojobs}{\check{\obs}}
\newcommand{\cakfprojobsrv}{\check{\obsrv}}
\newcommand{\cakfprojobsdim}{\check{N}}
\newcommand{\cakfmaxprojobsdim}{\check{N}^{\textnormal{max}}}
\newcommand{\cakfprojH}{\check{\obsH}}
\newcommand{\cakfprojobsnoise}{\check{\obsnoise}}
\newcommand{\cakfprojobsnoisecov}{\check{\obsnoisecov}}
\newcommand{\cakfprojgram}{\check{\mG}}
\newcommand{\cakfprojgramlsqrt}{\check{\mV}}

\newcommand{\caksstate}{\hat{\gmp}^{\textnormal{s}}}
\newcommand{\caksmean}{\hat{\vm}^{\textnormal{s}}}  %
\newcommand{\caksdd}{\hat{\mM}^{\textnormal{s}}}  %
\newcommand{\cakscov}{\hat{\mP}^{\textnormal{s}}}  %
\newcommand{\caksw}{\hat{\vw}^{\textnormal{s}}}
\newcommand{\caksW}{\hat{\mW}^{\textnormal{s}}}

\newcommand{\kfpmean}{\vm^-}  %
\newcommand{\kfpcov}{\mP^-}  %
\newcommand{\kfpdd}{\mM^-}  %
\newcommand{\kfmean}{\vm}  %
\newcommand{\kfcov}{\mP}  %
\newcommand{\kfdd}{\mM}
\newcommand{\kfresidual}{\vr}
\newcommand{\kfgram}{\mG}
\newcommand{\kfgramlsqrt}{\mV}
\newcommand{\kfW}{\mW}
\newcommand{\kfgain}{\mK}

\newcommand{\ksmean}{\kfmean^{\textnormal{s}}}
\newcommand{\kscov}{\kfcov^{\textnormal{s}}}
\newcommand{\ksdd}{\kfdd^{\textnormal{s}}}
\newcommand{\ksgain}{\kfgain^{\textnormal{s}}}
\newcommand{\ksw}{\vw^{\textnormal{s}}}
\newcommand{\ksW}{\mW^{\textnormal{s}}}

\newcommand{\kfpstate}{\gmp^{f-}}  %
\newcommand{\kfpobsrv}{\obsrv^{f-}}  %
\newcommand{\kfstate}{\gmp^f}  %
\newcommand{\ksstate}{\gmp^{\textnormal{s}}}  %
\newcommand{\kswrv}{\rvw^{\textnormal{s}}}  %

\newcommand{\cakfpstate}{\hat{\gmp}^{f-}}  %
\newcommand{\cakfstate}{\hat{\gmp}^f} %
\newcommand{\cakfwrv}{\hat{\rvw}}
\newcommand{\cakswrv}{\hat{\rvw}^{\textnormal{s}}}

\newcommand{\tmin}{t_0}
\newcommand{\tmax}{T}
\newcommand{\temporalinputspace}{[\tmin, \tmax]}
\newcommand{\spatialinputspace}{\spacesym{X}}
\newcommand{\spatialinputspacedim}{D_{\spatialinputspace}}
\newcommand{\inputspace}{\spacesym{Z}}

\newcommand{\strtargetfn}{f^\star}
\newcommand{\strobsfn}{y^\star}

\newcommand{\tstrain}{\vt^\textnormal{train}}
\newcommand{\ttrain}[1]{t^\textnormal{train}_{#1}}
\newcommand{\ntstrain}{\gmplen}
\newcommand{\xstrain}[1]{\mX^\textnormal{train}_{#1}}
\newcommand{\xtrain}[2]{\vx^\textnormal{train}_{#1, #2}}
\newcommand{\nxstrain}[1]{\obsdim_{#1}}
\newcommand{\zstrain}{\mZ^\textnormal{train}}
\newcommand{\ztrain}[1]{\vz^\textnormal{train}_{#1}}
\newcommand{\ystrain}{\vy^\textnormal{train}}
\newcommand{\ntraindata}{\obsdim}

\newcommand{\tsall}{\vt}
\newcommand{\ntsall}{N_\tsall}
\newcommand{\xsall}{\mX}
\newcommand{\nxsall}{N_\xsall}
\newcommand{\zsall}{\mZ}
\newcommand{\nzsall}{N_\zsall}

\newcommand{\gpprior}{\rproc{f}}
\newcommand{\meanfn}{\mu}
\newcommand{\covfn}{\Sigma}

\newcommand{\noisescale}{\sigma}
\newcommand{\noisycovfn}{\covfn^\noisescale}

\newcommand{\postmeanfn}[1][\strobsfn]{\bar{\meanfn}^{#1}}
\newcommand{\postcovfn}{\bar{\covfn}}

\newcommand{\outputdim}{D'}

\newcommand{\stsgmp}{\morproc{f}}
\newcommand{\stsgmpmeanfn}{\vmu}
\newcommand{\stsgmpcovfn}{\mSigma}
\newcommand{\stsgmpmeanfntime}{\vmu^t}
\newcommand{\stsgmpmeanfnspace}{\mu^\vx}
\newcommand{\stsgmpcovfntime}{\mSigma^t}
\newcommand{\stsgmpcovfnspace}{\Sigma^\vx}

\newcommand{\meanfntime}{\mu^t}
\newcommand{\meanfnspace}{\mu^\vx}
\newcommand{\covfntime}{\Sigma^t}
\newcommand{\covfnspace}{\Sigma^\vx}

\newcommand{\stsgmpapproxpostmeanfn}[1][\strobsfn]{\hat{\stsgmpmeanfn}^{#1}}
\newcommand{\stsgmpapproxpostcovfn}{\hat{\stsgmpcovfn}}

\newcommand{\mactions}{\mS}
\newcommand{\nactions}{\cakfprojobsdim}

\newcommand{\projystrain}{\check{\vy}^\textnormal{train}}
\newcommand{\projystrainrv}{\check{\rvy}^\textnormal{train}}

\newcommand{\approxpostmeanfn}[1][\strobsfn]{\hat{\meanfn}^{#1}}
\newcommand{\approxpostcovfn}{\hat{\covfn}}

\newcommand{\idxdtime}{k}

\NewDocumentCommand{\matfree}{m s O{} g}{%
  \textcolor{TUred}{#1}%
  \IfValueT{#4}{%
    \IfBooleanTF{#2}{%
      \brks*{#4}%
    }{%
      \brks[#3]{#4}%
    }%
  }
}

\usepackage{glossaries}
\usepackage{glossaries-extra}

\makenoidxglossaries

\newglossary*{gmp}{Dynamics Model}

\newglossaryentry{gmp:dim}{
  type={gmp},
  name={\ensuremath{\gmpdim}},
  description={Dimension of the state space},
}

\newglossaryentry{gmp:len}{
  type={gmp},
  name={\ensuremath{\gmplen}},
  description={Total number of states},
}

\newglossaryentry{gmp:state}{
  type={gmp},
  name={\ensuremath{\gmp_\idxdtime}},
  description={Unobserved state at time step $\idxdtime$},
}

\newglossaryentry{gmp:mean}{
  type={gmp},
  name={\ensuremath{\gmpmean_\idxdtime}},
  description={Prior mean of the state at time step $\idxdtime$},
}

\newglossaryentry{gmp:cov}{
  type={gmp},
  name={\ensuremath{\gmpcov_\idxdtime}},
  description={Prior covariance of the state at time step $\idxdtime$},
}

\newglossaryentry{gmp:A}{
  type={gmp},
  name={\ensuremath{\gmpA_\idxdtime}},
  description={Transition matrix from time step $\idxdtime$ to time step $\idxdtime + 1$},
}

\newglossaryentry{gmp:b}{
  type={gmp},
  name={\ensuremath{\gmpb_\idxdtime}},
  description={Transition offset from time step $\idxdtime$ to time step $\idxdtime + 1$},
}

\newglossaryentry{gmp:noise}{
  type={gmp},
  name={\ensuremath{\gmpnoise_\idxdtime}},
  description={Process noise from time step $\idxdtime$ to time step $\idxdtime + 1$},
}

\newglossaryentry{gmp:noisecov}{
  type={gmp},
  name={\ensuremath{\gmpnoisecov_\idxdtime}},
  description={Process noise covariance matrix from time step $\idxdtime$ to time step $\idxdtime + 1$},
}

\newglossary*{obs}{Observation Model}

\newglossaryentry{obs:dim}{
  type={obs},
  name={\ensuremath{\obsdim_k}},
  description={Number of observations (dimension of the observation vector) at time step $\idxdtime$},
}

\newglossaryentry{obs:vec}{
  type={obs},
  name={\ensuremath{\obs_\idxdtime}},
  description={Observation vector at time step $\idxdtime$},
}

\newglossaryentry{obs:rv}{
  type={obs},
  name={\ensuremath{\obsrv_\idxdtime}},
  description={Random variable encoding the belief about the observations at time step $\idxdtime$},
}

\newglossaryentry{obs:H}{
  type={obs},
  name={\ensuremath{\obsH_\idxdtime}},
  description={Observation matrix at time step $\idxdtime$},
}

\newglossaryentry{obs:noise}{
  type={obs},
  name={\ensuremath{\obsnoise_\idxdtime}},
  description={Observation noise at time step $\idxdtime$},
}

\newglossaryentry{obs:noisecov}{
  type={obs},
  name={\ensuremath{\obsnoisecov_\idxdtime}},
  description={Observation noise covariance matrix at time step $\idxdtime$},
}

\newglossary*{kf}{Kalman Filter}

\newglossaryentry{kf:pmean}{
  type={kf},
  name={\ensuremath{\kfpmean_\idxdtime}},
  description={Predictive filter mean at time step $\idxdtime$},
}

\newglossaryentry{kf:pdd}{
  type={kf},
  name={\ensuremath{\kfpdd_\idxdtime}},
  description={Left square root of the downdate term in the predictive filter covariance at time step $\idxdtime$},
}

\newglossaryentry{kf:pcov}{
  type={kf},
  name={\ensuremath{\kfpcov_\idxdtime}},
  description={Predictive filter covariance at time step $\idxdtime$},
}

\newglossaryentry{kf:mean}{
  type={kf},
  name={\ensuremath{\kfmean_\idxdtime}},
  description={Updated filter mean at time step $\idxdtime$},
}

\newglossaryentry{kf:dd}{
  type={kf},
  name={\ensuremath{\kfdd_\idxdtime}},
  description={Left square root of the downdate term in the updated filter covariance at time step $\idxdtime$},
}

\newglossaryentry{kf:cov}{
  type={kf},
  name={\ensuremath{\kfcov_\idxdtime}},
  description={Updated filter covariance at time step $\idxdtime$},
}

\newglossaryentry{kf:residual}{
  type={kf},
  name={\ensuremath{\kfresidual_\idxdtime}},
  description={Prediction-measurement residual at time step $\idxdtime$},
}

\newglossaryentry{kf:gram}{
  type={kf},
  name={\ensuremath{\kfgram_\idxdtime}},
  description={Innovation matrix at time step $\idxdtime$},
}

\newglossaryentry{kf:gramlsqrt}{
  type={kf},
  name={\ensuremath{\kfgramlsqrt_\idxdtime}},
  description={Left square root of the inverse innovation matrix at time step $\idxdtime$},
}

\newglossaryentry{kf:W}{
  type={kf},
  name={\ensuremath{\kfW_\idxdtime}},
  description={Filter ``covariance message'' propagated from time step $\idxdtime$ to time step $\idxdtime + 1$ in the inverse-free RTS smoother},
}

\newglossaryentry{kf:gain}{
  type={kf},
  name={\ensuremath{\kfgain_\idxdtime}},
  description={Kalman gain at time step $\idxdtime$},
}

\newglossary*{cakf}{Computation-Aware Kalman Filter}

\newglossaryentry{hatted}{
  type={cakf},
  name={\ensuremath{\hat{(\cdot)}}},
  description={``Hatted'' quantities are associated with the CAKF/CAKS.
      Typically, these are counterparts of quantities in the standard Kalman filter / RTS smoother.
    },
}

\newglossaryentry{cakf:pmean}{
  type={cakf},
  name={\ensuremath{\cakfpmean_\idxdtime}},
  description={Predictive CAKF mean at time step $\idxdtime$},
}

\newglossaryentry{cakf:pdd}{
  type={cakf},
  name={\ensuremath{\cakfpdd_\idxdtime}},
  description={Left square root of the low-rank downdate representing the predictive CAKF covariance at time step $\idxdtime$},
}

\newglossaryentry{cakf:pcov}{
  type={cakf},
  name={\ensuremath{\cakfpcov_\idxdtime}},
  description={Predictive CAKF covariance at time step $\idxdtime$},
}

\newglossaryentry{cakf:mean}{
  type={cakf},
  name={\ensuremath{\cakfmean_\idxdtime}},
  description={Updated CAKF mean at time step $\idxdtime$},
}

\newglossaryentry{cakf:dd}{
  type={cakf},
  name={\ensuremath{\cakfdd_\idxdtime}},
  description={Left square root of the low-rank downdate representing the updated CAKF covariance at time step $\idxdtime$},
}

\newglossaryentry{cakf:ddrank}{
  type={cakf},
  name={\ensuremath{\cakfddrank_\idxdtime}},
  description={Rank of the downdate representing the updated CAKF covariance at time step $\idxdtime$},
}

\newglossaryentry{cakf:cov}{
  type={cakf},
  name={\ensuremath{\cakfcov_\idxdtime}},
  description={Updated CAKF covariance at time step $\idxdtime$},
}

\newglossaryentry{cakf:projobsdim}{
  type={cakf},
  name={\ensuremath{\cakfprojobsdim_\idxdtime}},
  description={Dimension of the projected observation vector (or equivalently number of actions) at time step $\idxdtime$},
}

\newglossaryentry{cakf:acts}{
  type={cakf},
  name={\ensuremath{\cakfacts_\idxdtime}},
  description={Action matrix whose columns (the actions) span the low-dimensional subspace onto which the CAKF projects the observation at time step $\idxdtime$},
}

\newglossaryentry{cakf:act}{
  type={cakf},
  name={\ensuremath{\cakfact_\idxdtime\iidx}},
  description={$i$-th action at time step $\idxdtime$ given by the $i$-th column of $\cakfacts_\idxdtime$},
}

\newglossaryentry{cakf:projobs}{
  type={cakf},
  name={\ensuremath{\cakfprojobs_\idxdtime}},
  description={Projected observation vector at time step $\idxdtime$},
}

\newglossaryentry{cakf:projobsrv}{
  type={cakf},
  name={\ensuremath{\cakfprojobsrv_\idxdtime}},
  description={Random variable modeling the belief about the projected observation vector at time step $\idxdtime$},
}

\newglossaryentry{cakf:projH}{
  type={cakf},
  name={\ensuremath{\cakfprojH_\idxdtime}},
  description={Projected observation matrix at time step $\idxdtime$},
}

\newglossaryentry{cakf:projobsnoise}{
  type={cakf},
  name={\ensuremath{\cakfprojobsnoise_\idxdtime}},
  description={Projected observation noise at time step $\idxdtime$},
}

\newglossaryentry{cakf:projobsnoisecov}{
  type={cakf},
  name={\ensuremath{\cakfprojobsnoisecov_\idxdtime}},
  description={Projected observation noise covariance matrix at time step $\idxdtime$},
}

\newglossaryentry{cakf:projgram}{
  type={cakf},
  name={\ensuremath{\cakfprojgram_\idxdtime}},
  description={Projected innovation matrix at time step $\idxdtime$},
}

\newglossaryentry{cakf:projgramlsqrt}{
  type={cakf},
  name={\ensuremath{\cakfprojgramlsqrt_\idxdtime}},
  description={Left square root of the projected inverse innovation matrix at time step $\idxdtime$},
}

\newglossaryentry{cakf:w}{
  type={cakf},
  name={\ensuremath{\cakfw_\idxdtime}},
  description={CAKF ``mean message'' propagated from time step $\idxdtime$ to time step $\idxdtime + 1$ in the CAKS},
}

\newglossaryentry{cakf:W}{
  type={cakf},
  name={\ensuremath{\cakfW_\idxdtime}},
  description={CAKF ``covariance message'' propagated from time step $\idxdtime$ to time step $\idxdtime + 1$ in the CAKS},
}

\newglossaryentry{cakf:tdd}{
  type={cakf},
  name={\ensuremath{\cakftdd_\idxdtime}},
  description={Left square root of the truncated low-rank downdate defining the truncated CAKF covariance at time step $\idxdtime$},
}

\newglossaryentry{cakf:tddrank}{
  type={cakf},
  name={\ensuremath{\cakftddrank_\idxdtime}},
  description={Rank of the truncated downdate defining the truncated CAKF covariance at time step $\idxdtime$},
}

\newglossaryentry{cakf:tcov}{
  type={cakf},
  name={\ensuremath{\cakftcov_\idxdtime}},
  description={Truncated CAKF covariance at time step $\idxdtime$},
}

\newglossaryentry{gmp:tstate}{
  type={cakf},
  name={\ensuremath{\gmptstate{\idxdtime}}},
  description={Additional state in the augmented state-space model of the CAKF modeling computational noise due to downdate truncation ``infinitesimally after'' time step $\idxdtime$},
}

\newglossary*{ks}{Rauch--Tung--Striebel Smoother}

\newglossaryentry{ks:mean}{
  type={ks},
  name={\ensuremath{\ksmean_\idxdtime}},
  description={Smoother mean at time step $\idxdtime$},
}

\newglossaryentry{ks:cov}{
  type={ks},
  name={\ensuremath{\kscov_\idxdtime}},
  description={Smoother covariance at time step $\idxdtime$},
}

\newglossaryentry{ks:w}{
  type={ks},
  name={\ensuremath{\ksw_\idxdtime}},
  description={Smoother ``mean message'' propagated from time step $\idxdtime$ to time step $\idxdtime - 1$ in the inverse-free RTS smoother},
}

\newglossaryentry{ks:W}{
  type={ks},
  name={\ensuremath{\ksW_\idxdtime}},
  description={Smoother ``covariance message'' propagated from time step $\idxdtime$ to time step $\idxdtime - 1$ in the inverse-free RTS smoother},
}

\newglossaryentry{ks:gain}{
  type={ks},
  name={\ensuremath{\ksgain_\idxdtime}},
  description={Smoother gain at time step $\idxdtime$},
}

\newglossary*{caks}{Computation-Aware Rauch--Tung--Striebel Smoother}

\newglossaryentry{caks:mean}{
  type={caks},
  name={\ensuremath{\caksmean_\idxdtime}},
  description={CAKS mean at time step $\idxdtime$},
}

\newglossaryentry{caks:dd}{
  type={caks},
  name={\ensuremath{\caksdd_\idxdtime}},
  description={Left square root of the low-rank downdate representing the CAKS covariance at time step $\idxdtime$},
}

\newglossaryentry{caks:cov}{
  type={caks},
  name={\ensuremath{\cakscov_\idxdtime}},
  description={CAKS covariance at time step $\idxdtime$},
}

\newglossaryentry{caks:w}{
  type={caks},
  name={\ensuremath{\caksw_\idxdtime}},
  description={CAKS ``mean message'' propagated from time step $\idxdtime$ to time step $\idxdtime - 1$},
}

\newglossaryentry{caks:W}{
  type={caks},
  name={\ensuremath{\caksW_\idxdtime}},
  description={(Truncated) CAKS ``covariance message'' propagated from time step $\idxdtime$ to time step $\idxdtime - 1$},
}

\newglossary*{samp}{Matheron Sampling}

\newglossaryentry{kf:pstate}{
  type={samp},
  name={\ensuremath{\kfpstate_\idxdtime}},
  description={Random variable with distribution $\gaussian{\kfpmean_\idxdtime}{\kfpcov_\idxdtime}$},
}

\newglossaryentry{kf:pobsrv}{
  type={samp},
  name={\ensuremath{\kfpobsrv_\idxdtime}},
  description={Random variable with distribution $\gaussian{\obsH_\idxdtime \kfpmean_\idxdtime}{\kfgram_\idxdtime}$},
}

\newglossaryentry{kf:state}{
  type={samp},
  name={\ensuremath{\kfstate_\idxdtime}},
  description={Random variable with distribution $\gaussian{\kfmean_\idxdtime}{\kfcov_\idxdtime}$},
}

\newglossaryentry{ks:wrv}{
  type={samp},
  name={\ensuremath{\kswrv_\idxdtime}},
  description={Random variable propagating the ``smoother message'' from time step $\idxdtime$ to time step $\idxdtime - 1$ during inverse-free posterior sampling},
}

\newglossaryentry{ks:state}{
  type={samp},
  name={\ensuremath{\ksstate_\idxdtime}},
  description={Random variable with distribution $\gaussian{\ksmean_\idxdtime}{\kscov_\idxdtime}$},
}

\newglossary*{casamp}{Computation-Aware Matheron Sampling}

\newglossaryentry{cakf:pstate}{
  type={casamp},
  name={\ensuremath{\cakfpstate_\idxdtime}},
  description={Random variable with distribution $\gaussian{\cakfpmean_\idxdtime}{\cakfpcov_\idxdtime}$},
}

\newglossaryentry{cakf:wrv}{
  type={casamp},
  name={\ensuremath{\cakfwrv_\idxdtime}},
  description={Random variable propagating the ``filter message'' from time step $\idxdtime$ to time step $\idxdtime + 1$ during CAKF/CAKS sampling},
}

\newglossaryentry{cakf:state}{
  type={casamp},
  name={\ensuremath{\cakfstate_\idxdtime}},
  description={Random variable with distribution $\gaussian{\cakfmean_\idxdtime}{\cakfcov_\idxdtime}$},
}

\newglossaryentry{caks:wrv}{
  type={casamp},
  name={\ensuremath{\cakswrv_\idxdtime}},
  description={Random variable propagating the ``smoother message'' from time step $\idxdtime$ to time step $\idxdtime - 1$ during CAKS sampling},
}

\newglossaryentry{caks:state}{
  type={casamp},
  name={\ensuremath{\caksstate_\idxdtime}},
  description={Random variable with distribution $\gaussian{\caksmean_\idxdtime}{\cakscov_\idxdtime}$},
}

\newglossary*{str}{Spatiotemporal GP Regression}

\newglossaryentry{str:temporalinputspace}{
  type={str},
  name={\ensuremath{\temporalinputspace}},
  description={Temporal domain},
}

\newglossaryentry{str:spatialinputspace}{
  type={str},
  name={\ensuremath{\spatialinputspace}},
  description={Spatial domain},
}

\newglossaryentry{str:inputspace}{
  type={str},
  name={\ensuremath{\inputspace}},
  description={Input domain $\inputspace = \temporalinputspace \times \spatialinputspace$},
}

\newglossaryentry{str:targetfn}{
  type={str},
  name={\ensuremath{\strtargetfn}},
  description={Unknown target function},
}

\newglossaryentry{str:obsfn}{
  type={str},
  name={\ensuremath{\strobsfn}},
  description={Noisy observed function},
}

\newglossaryentry{str:ttrain}{
  type={str},
  name={\ensuremath{\ttrain{\idxdtime}}},
  description={$\idxdtime$-th time step at which training data is available},
}

\newglossaryentry{str:xstrain}{
  type={str},
  name={\ensuremath{\xstrain{\idxdtime}}},
  description={Vector of spatial points $\xtrain{\idxdtime}{n}$ at which training data is available at time step $\ttrain{\idxdtime}$},
}

\newglossaryentry{str:zstrain}{
  type={str},
  name={\ensuremath{\zstrain}},
  description={Vector of ordered pairs $(\ttrain{\idxdtime}, \xtrain{\idxdtime}{n})$ of training input points},
}

\newglossaryentry{str:ystrain}{
  type={str},
  name={\ensuremath{\ystrain}},
  description={Vector containing all training targets},
}

\newglossaryentry{str:ntraindata}{
  type={str},
  name={\ensuremath{\ntraindata}},
  description={Total number of training data points},
}

\newglossaryentry{str:xsall}{
  type={str},
  name={\ensuremath{\xsall}},
  description={Vector containing all (unique) spatial training and test points from all time steps},
}

\newglossaryentry{str:nxsall}{
  type={str},
  name={\ensuremath{\nxsall}},
  description={Total number of unique spatial training and test points over all time steps},
}

\newglossaryentry{str:gpprior}{
  type={str},
  name={\ensuremath{\gpprior}},
  description={Spatiotemporal Gaussian process prior for the unknown target function $\strtargetfn$},
}

\newglossaryentry{str:meanfn}{
  type={str},
  name={\ensuremath{\meanfn}},
  description={Mean function of $\gpprior$},
}

\newglossaryentry{str:covfn}{
  type={str},
  name={\ensuremath{\covfn}},
  description={Covariance function of $\gpprior$},
}

\newglossaryentry{str:rkhs}{
  type={str},
  name={\ensuremath{\rkhs{\covfn}}},
  description={Reproducing kernel Hilbert space (RKHS) associated with $\covfn$},
}

\newglossaryentry{str:noisescale}{
  type={str},
  name={\ensuremath{\noisescale}},
  description={Observation noise scale},
}

\newglossaryentry{str:noisycovfn}{
  type={str},
  name={\ensuremath{\noisycovfn}},
  description={Covariance function of the noisy observed process},
}

\newglossaryentry{str:postmeanfn}{
  type={str},
  name={\ensuremath{\postmeanfn}},
  description={Posterior mean function corresponding to the observed function $\strobsfn$},
}

\newglossaryentry{str:postcovfn}{
  type={str},
  name={\ensuremath{\postcovfn}},
  description={Posterior covariance function},
}

\newglossary*{stsgmp}{Space-Time Separable Gauss(--Markov) Processes}

\newglossaryentry{stsgmp}{
  type={stsgmp},
  name={\ensuremath{\stsgmp}},
  description={Space-time separable Gaussian (or Gauss--Markov) process (STSG(M)P). Typically obtained by combining $\gpprior$ and $\outputdim - 1$ of its time derivatives in a multi-output GP},
}

\newglossaryentry{stsgmp:outputdim}{
  type={stsgmp},
  name={\ensuremath{\outputdim}},
  description={Output dimension of the STSG(M)P},
}

\newglossaryentry{stsgmp:meanfn}{
  type={stsgmp},
  name={\ensuremath{\stsgmpmeanfn}},
  description={Mean function of the STSG(M)P},
}

\newglossaryentry{stsgmp:meanfntime}{
  type={stsgmp},
  name={\ensuremath{\stsgmpmeanfntime}},
  description={Temporal factor of the mean function of the STSG(M)P},
}

\newglossaryentry{stsgmp:meanfnspace}{
  type={stsgmp},
  name={\ensuremath{\stsgmpmeanfnspace}},
  description={Spatial factor of the mean function of the STSG(M)P},
}

\newglossaryentry{stsgmp:covfn}{
  type={stsgmp},
  name={\ensuremath{\stsgmpcovfn}},
  description={Covariance function of the STSG(M)P},
}

\newglossaryentry{stsgmp:covfntime}{
  type={stsgmp},
  name={\ensuremath{\stsgmpcovfntime}},
  description={Temporal factor of the covariance function of the STSG(M)P},
}

\newglossaryentry{stsgmp:covfnspace}{
  type={stsgmp},
  name={\ensuremath{\stsgmpcovfnspace}},
  description={Spatial factor of the covariance function of the STSG(M)P},
}

\newglossary*{itergp}{Iteratively-Approximated Gaussian Process Regression}

\newglossaryentry{itergp:actions}{
  type={itergp},
  name={\ensuremath{\mactions}},
  description={Matrix of actions},
}

\newglossaryentry{itergp:nactions}{
  type={itergp},
  name={\ensuremath{\nactions}},
  description={Number of actions},
}

\newglossaryentry{itergp:projystrain}{
  type={itergp},
  name={\ensuremath{\projystrain}},
  description={Projected training targets},
}

\newglossaryentry{itergp:projystrainrv}{
  type={itergp},
  name={\ensuremath{\projystrainrv}},
  description={Random variable modeling the (prior) belief about the projected training targets},
}

\newglossaryentry{itergp:approxpostmeanfn}{
  type={itergp},
  name={\ensuremath{\approxpostmeanfn}},
  description={Approximate posterior mean function corresponding to the observed function $\strobsfn$},
}

\newglossaryentry{itergp:approxpostcovfn}{
  type={itergp},
  name={\ensuremath{\approxpostcovfn}},
  description={Approximate posterior covariance function},
}

\usepackage[
  capitalize,
  nameinlink,
  noabbrev,
]{cleveref}

\makeatletter
\if@cref@capitalise
  \crefname{assumption}{Assumption}{Assumptions}
\else
  \crefname{assumption}{assumption}{assumptions}
\fi
\makeatother

\begin{document}
  \twocolumn[
    \aistatstitle{Computation-Aware Kalman Filtering and Smoothing}

    \aistatsauthor{%
      Marvin Pförtner$^1$
      \And Jonathan Wenger$^2$
      \And Jon Cockayne$^3$
      \And Philipp Hennig$^1$
    }

    \aistatsaddress{
      \\
      $^1$ Tübingen AI Center, University of Tübingen
      \hspace{1em}
      $^2$ Columbia University
      \hspace{1em}
      $^3$ University of Southampton
    }
  ]

  \begin{abstract}
    Kalman filtering and smoothing are the foundational mechanisms for efficient inference in Gauss--Markov models.
However, their time and memory complexities scale prohibitively with the size of the state space.
This is particularly problematic in spatiotemporal regression problems, where the state dimension scales with the number of spatial observations.
Existing approximate frameworks leverage low-rank approximations of the covariance matrix.
But since they do not model the error introduced by the computational approximation, their predictive uncertainty estimates can be overly optimistic.
In this work, we propose a probabilistic numerical method for inference in high-dimensional Gauss--Markov models which mitigates these scaling issues.
Our matrix-free iterative algorithm leverages GPU acceleration and crucially enables a tunable trade-off between computational cost and predictive uncertainty.
Finally, we demonstrate the scalability of our method on a large-scale climate dataset.

  \end{abstract}

  \section{INTRODUCTION}
\label{sec:introduction}
From language modeling to robotics to climate science, many application domains of machine learning generate data that are correlated in time.
By describing the underlying temporal dynamics via a \emph{state-space model} (SSM), the sequential structure can be leveraged to perform efficient inference.
In machine learning, state-space models are widely used in reinforcement learning \citep{Hafner2019LearningLatent}, as well as in deep \citep{Gu2023MambaLinearTime} and probabilistic \citep{Sarkka2023BayesianFiltering} sequence modeling.
For example, suppose we aim to forecast temperature as a function of time from a set of \(\ntstrain\) observations.
Standard regression approaches have cubic cost \(\mathcal{O}{(\ntstrain^3)}\) in the number of data points \citep{RasmussenGaussianProcessesMachine2006}.
Instead, one can leverage the temporal structure of the problem by representing it as a state-space model and performing Bayesian filtering and smoothing, which has \emph{linear} time complexity \(\mathcal{O}(\ntstrain)\) \citep{Sarkka2023BayesianFiltering}.

\paragraph{Challenges of a Large State-Space Dimension}
However, if the latent state has more than a few dimensions, inference in a state-space model can quickly become prohibitive.
The overall computational cost is linear in time, but \emph{cubic} in the size of the state space \(\gmpdim\) with a \emph{quadratic} memory requirement.
Returning to the example above, suppose temperatures are given at a set of \(\nxsall\) spatial measurement locations around the globe.
Assuming a non-zero correlation in temperature between those locations, the computational cost \(\mathcal{O}(\ntstrain \cdot \gmpdim^3)\) with $\gmpdim = \mathcal{O}(\nxsall)$ quickly becomes prohibitive.
In response, many approximate filtering and smoothing algorithms have been proposed, e.g., based on sampling \citep[e.g.,][]{Evensen1994EnKF}, Krylov subspace methods \citep{Bardsley2011}, sketching \citep{Berberidis2017}, and dynamical-low-rank approximation \citep{Schmidt2023RankReducedKalman}.
All of these methods inevitably introduce approximation error, which is \emph{not} accounted for in the uncertainty estimates of the resulting posterior distributions.

\begin{figure*}
  \centering
  \begin{subcaptionblock}{0.48\textwidth}
    \centering
    \caption{$\gmpdim = \num{14640}$}
    \label{fig:12-512}
    \vspace{.5em}
    \begin{minipage}[b]{0.48\linewidth}
      \centering
      {\footnotesize Mean}\\
      [0.5em]
      \includegraphics[width=\linewidth]{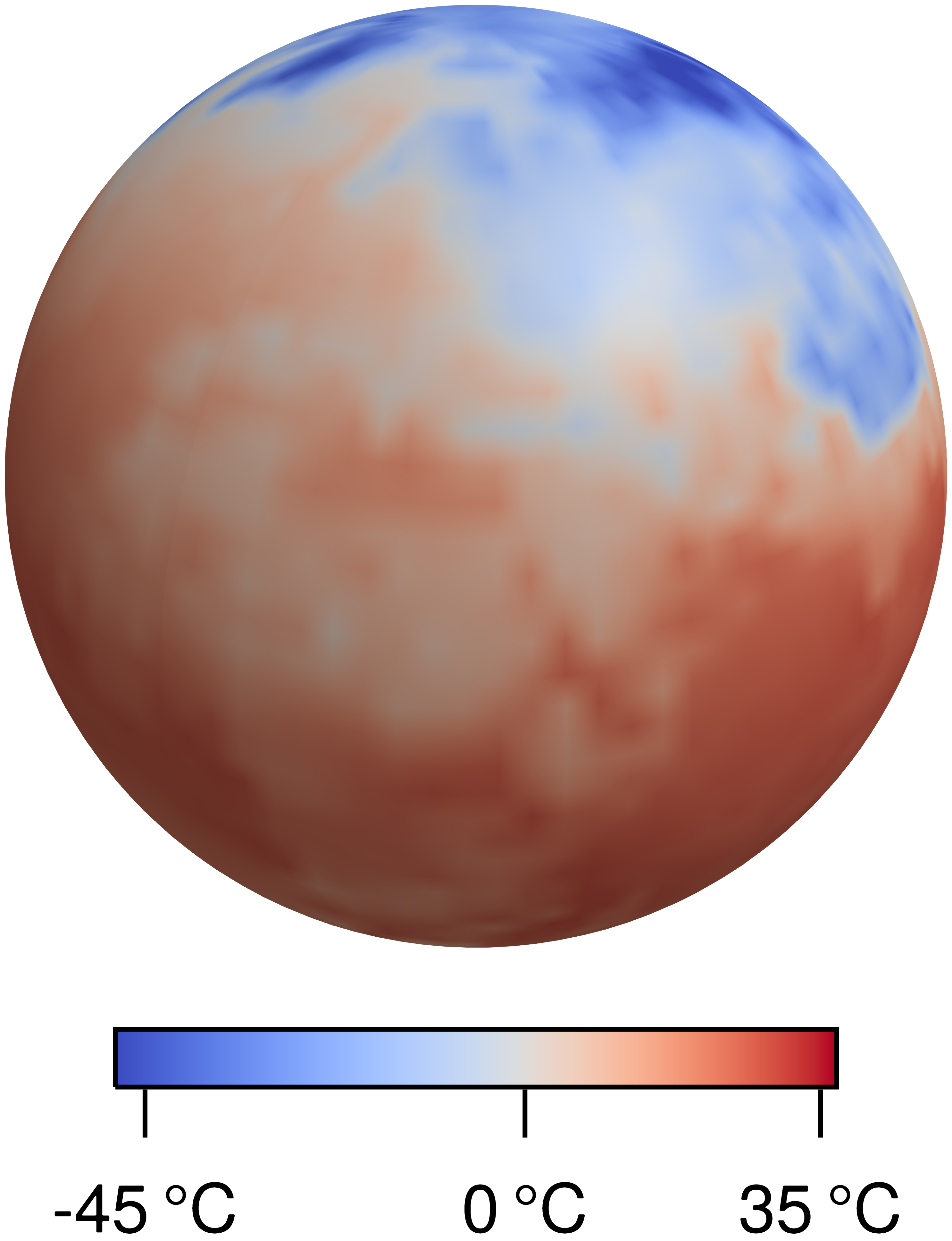}%
    \end{minipage}
    \hfill%
    \begin{minipage}[b]{0.48\linewidth}
      \centering
      {\footnotesize Standard Deviation}\\
      [0.5em]
      \includegraphics[width=\linewidth]{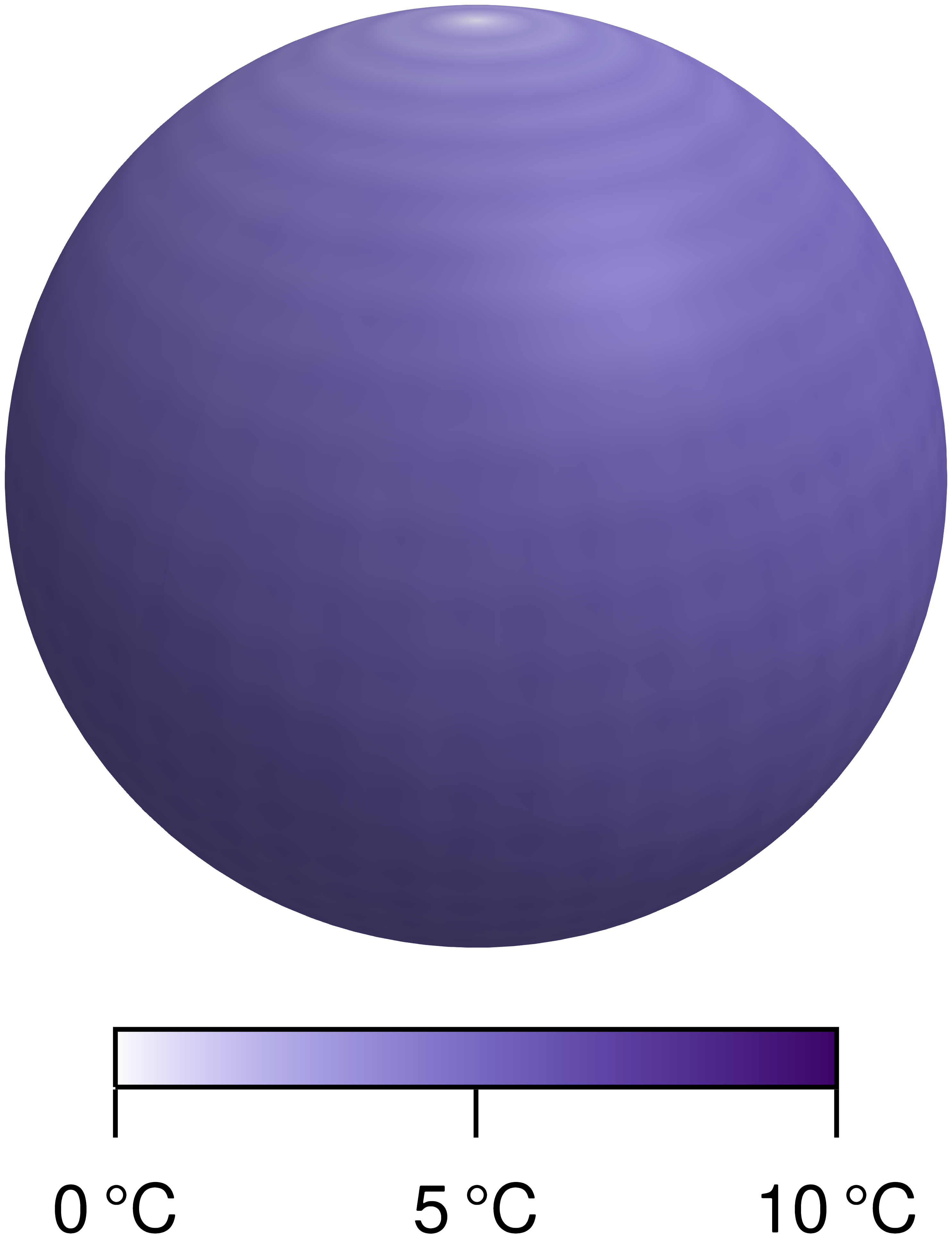}%
    \end{minipage}
  \end{subcaptionblock}
  \hfill%
  \begin{subcaptionblock}{0.48\textwidth}
    \centering
    \caption{$\gmpdim = \num{231360}$}
    \label{fig:3-64}
    \vspace{.5em}
    \begin{minipage}[b]{0.48\linewidth}
      \centering
      {\footnotesize Mean}\\
      [0.5em]
      \includegraphics[width=\linewidth]{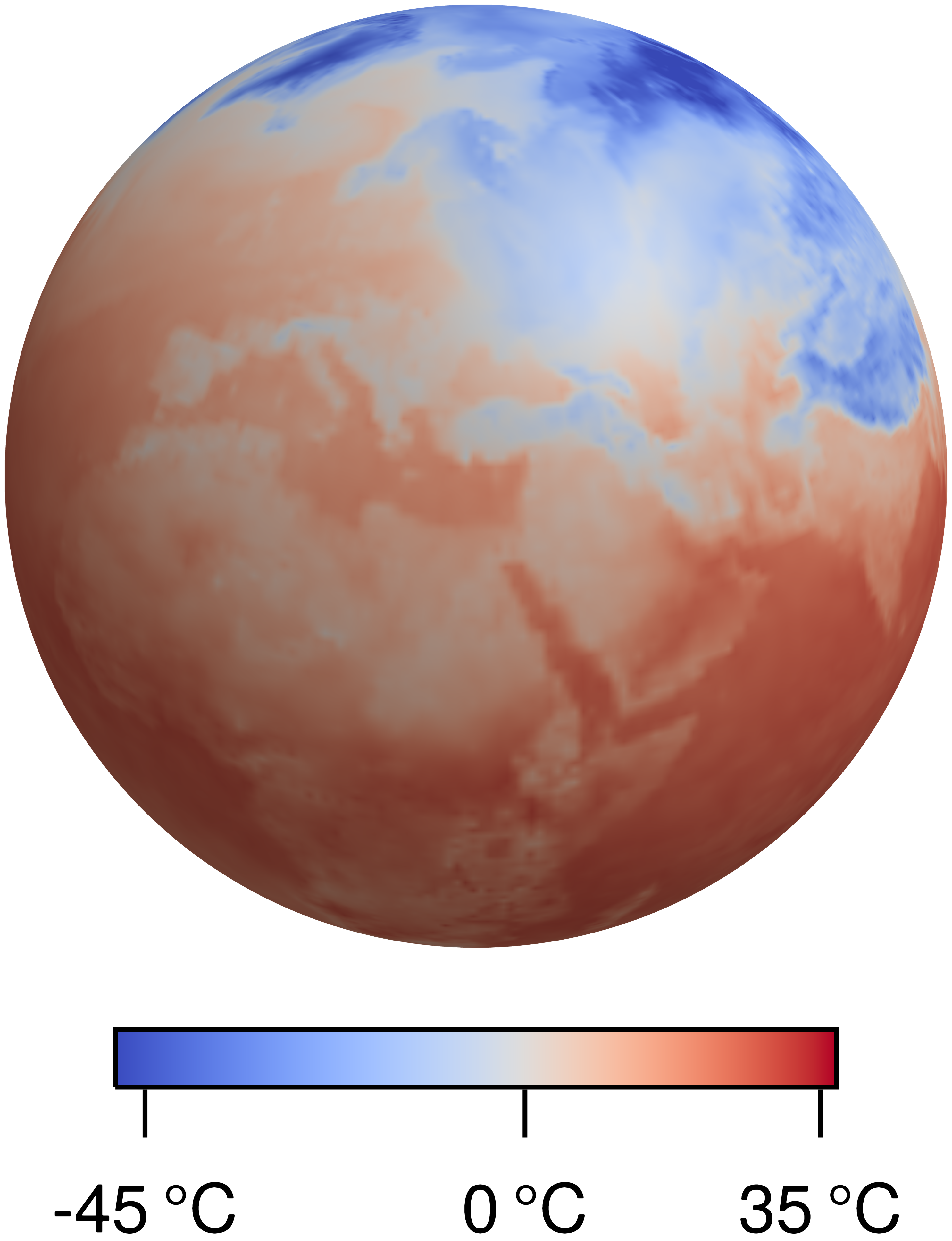}%
    \end{minipage}
    \hfill%
    \begin{minipage}[b]{0.48\linewidth}
      \centering
      {\footnotesize Standard Deviation}\\
      [0.5em]
      \includegraphics[width=\linewidth]{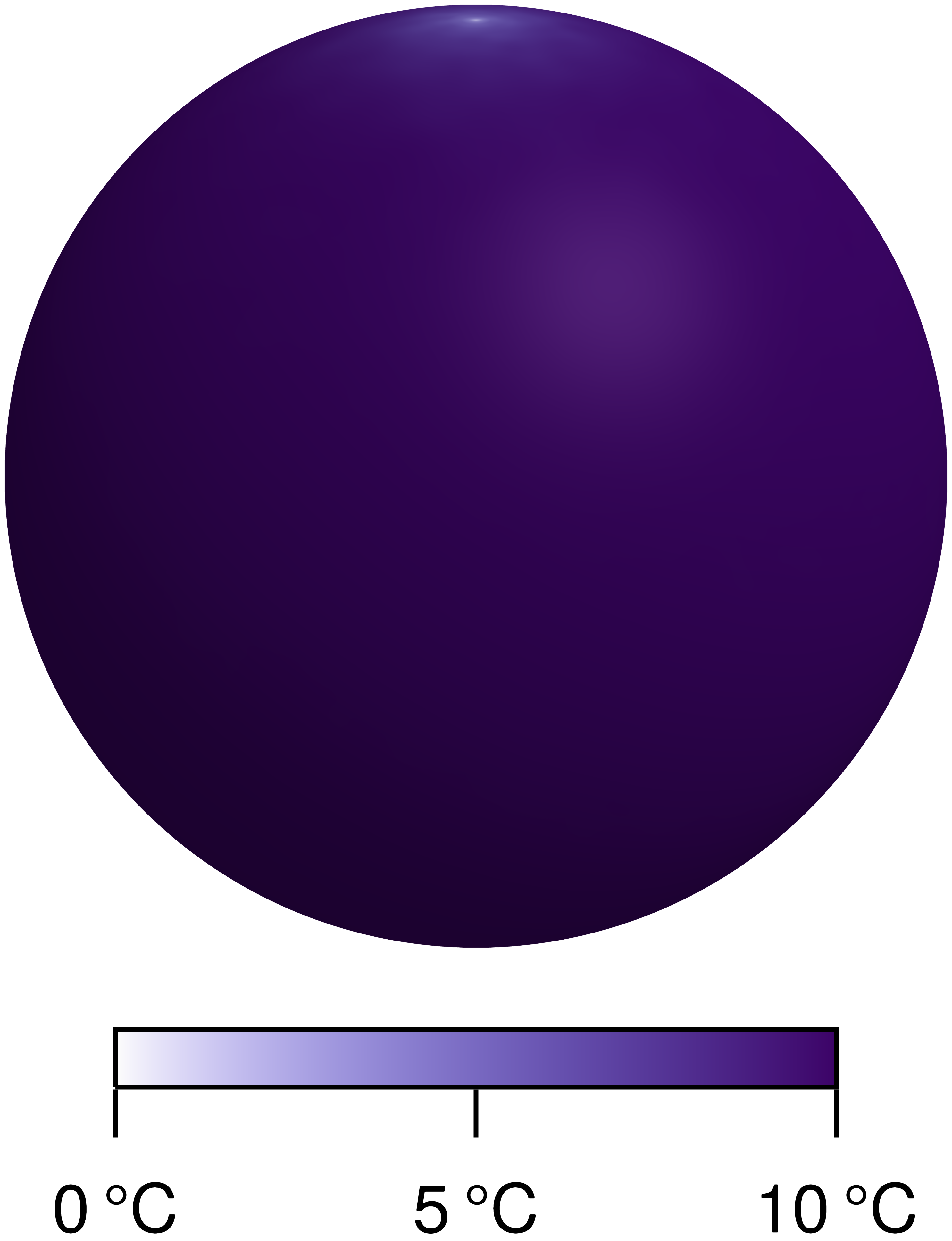}%
    \end{minipage}
  \end{subcaptionblock}
  \caption{
    Spatio-temporal Gaussian process regression of Earth's surface temperature using the ERA5 dataset \citep{Hersbach2020ERA5Global} via computation-aware filtering and smoothing for two different values of state-space dimension $\gmpdim$.
    Kalman filtering and RTS smoothing would require in excess of $\qty{1.17}{\tebi\byte}$ of memory to generate \cref{fig:3-64}.
    Our novel algorithms scale to larger state-space dimension $\gmpdim$ with lower time and memory costs, resolving finer detail and achieving better predictive performance.
  }
  \label{fig:visual-abstract}
  \vspace{-1em} %
\end{figure*}

\paragraph{Computation-Aware Filtering and Smoothing}
In this work, we introduce \emph{computation-aware Kalman filters} (CAKFs) and \emph{smoothers} (CAKSs): novel approximate versions of the Kalman filter and Rauch--Tung--Striebel (RTS) smoother.
Approximations are introduced both to reduce the computational cost through low-dimensional projection of the data (\cref{sec:projection}) and memory burden through covariance truncation (\cref{sec:truncation}).
Alongside their prediction for the underlying dynamics, they return a combined uncertainty estimate quantifying both epistemic uncertainty \emph{and} approximation error.
\Cref{fig:visual-abstract} showcases our approach on a large-scale spatiotemporal regression problem.

\textbf{Contribution}
We introduce the CAKF and CAKS, novel filtering and smoothing algorithms that are:
\begin{enumerate}[itemsep=-0.5ex,topsep=0ex,partopsep=0ex,label=(\arabic*)]
  \item \emph{iterative} and \emph{matrix-free}, and can fully leverage modern parallel hardware (i.e., GPUs);
  \item \emph{more efficient} both in time and space than their standard versions (see \cref{sec:computational-complexity}); and
  \item \emph{computation-aware}, i.e., they come with theoretical guarantees for their uncertainty estimates which capture the inevitable approximation error (\cref{thm:pointwise-error-cakf}).
\end{enumerate}
We demonstrate the scalability of our approach empirically on climate data with up to \(\ntstrain \cdot \nxsall\approx\num{4}\mathrm{M}\) observations and state-space dimension \(\gmpdim\approx230\mathrm{k}\).

  \section{BACKGROUND}
\label{sec:background}
Many temporal processes can be modeled with a linear-Gaussian state-space model, in which exact Bayesian inference can be done efficiently using the Kalman filter and Rauch--Tung--Striebel smoother.

\subsection{Bayesian Inference in Linear-Gaussian State-Space Models}
\label{sec:lgssm}
In the following, we want to infer the values of the \emph{states} $\gls{gmp:state} \in \R^{\gls{gmp:dim}}$ of an unobserved discrete-time Gauss--Markov process $\set{\gmp_\idxdtime}_{\idxdtime = 0}^{\gls{gmp:len}}$ (or a discretized continuous-time Gauss--Markov process with $\gmp_\idxdtime = \gmp(\ttrain{\idxdtime})$; see \cref{sec:temporal-interpolation}) defined by the \emph{dynamics model}
\(
\gmp_{\idxdtime + 1}
= \gls{gmp:A} \gmp_\idxdtime + \gls{gmp:b} + \gls{gmp:noise}
\)
with $\gmp_0 \sim \gaussian{\gmpmean_0}{\gmpcov_0}$ and $\gmpnoise_\idxdtime \sim \gaussian{\vec{0}}{\gls{gmp:noisecov}}$ from a given set of noisy \emph{observations} $\set{\gls{obs:vec}}_{\idxdtime = 1}^\gmplen$ made through the \emph{observation model}
\(
\gls{obs:rv}
\defeq \gls{obs:H} \gmp_\idxdtime + \gls{obs:noise} \in \R^{\gls{obs:dim}}.
\)
with $\obsnoise_\idxdtime \sim \gaussian{\vec{0}}{\gls{obs:noisecov}}$.
Collectively, the dynamics and observation models are referred to as a \emph{linear-Gaussian state-space model} (LGSSM).
The \emph{initial state} $\gmp_0$, the \emph{process noise} $\set{\gmpnoise_\idxdtime}_{\idxdtime = 0}^{\gmplen - 1}$, and the \emph{observation noise} $\set{\obsnoise_\idxdtime}_{\idxdtime = 1}^\gmplen$ are pairwise independent.
One can show that $\gmp_\idxdtime \sim \gaussian{\gls{gmp:mean}}{\gls{gmp:cov}}$ with $\gmpmean_{\idxdtime + 1} = \gmpA_\idxdtime \gmpmean_\idxdtime + \gmpb_\idxdtime$ and $\gmpcov_{\idxdtime + 1} = \gmpA_\idxdtime \gmpcov_\idxdtime \gmpA_\idxdtime\T + \gmpnoisecov_\idxdtime$.
The Kalman filter is an algorithm for computing conditional distributions of the form $\condrv{\gmp_\idxdtime \given \obsrv_{1:\idxdtime} = \obs_{1:\idxdtime}}$, $\idxdtime = 1, \dots, \gmplen$.
It alternates recursively between computing the moments
\begin{align*}
  \gls{kf:pmean} & \defeq \gmpA_{\idxdtime-1} \kfmean_{\idxdtime-1} + \gmpb_{\idxdtime-1}                             \\
  \gls{kf:pcov}  & \defeq \gmpA_{\idxdtime-1} \kfcov_{\idxdtime-1} \gmpA_{\idxdtime-1}\T + \gmpnoisecov_{\idxdtime-1}
\end{align*}
of
\(
\gmp_\idxdtime \mid \obsrv_{1:\idxdtime-1} = \obs_{1:\idxdtime - 1}
\sim \gaussian{\kfpmean_\idxdtime}{\kfpcov_\idxdtime}
\)
in the \emph{predict step} and the moments
\begin{subequations}
  \label{eq:kf-update-step}
  \begin{align}
    \gls{kf:mean} & \defeq \kfpmean_\idxdtime + \kfpcov_\idxdtime \obsH_\idxdtime \kfgram_\idxdtime^{-1} (\obsrv_\idxdtime - \obsH_\idxdtime \kfpmean_\idxdtime) \label{eq:kalman_filtering_mean} \\
    \gls{kf:cov}  & \defeq \kfpcov_\idxdtime - \kfpcov_\idxdtime \obsH_\idxdtime \kfgram_\idxdtime^{-1} \obsH_\idxdtime\T \kfpcov_\idxdtime \label{eq:kalman_filtering_cov}
  \end{align}
\end{subequations}
of
\(
\gmp_\idxdtime \mid \obsrv_{1:\idxdtime} = \obs_{1:\idxdtime}
\sim \gaussian{\kfmean_\idxdtime}{\kfcov_\idxdtime}
\)
in the \emph{update step}, where $\gls{kf:gram} \defeq \obsH_\idxdtime \kfpcov_\idxdtime \obsH_\idxdtime\T + \obsnoisecov_\idxdtime$ is the \emph{innovation matrix}.
The conditional distributions computed by the filter are mainly useful for forecasting.
Interpolation requires the full Bayesian posterior
\(
\condrv{\gmp_\idxdtime \given \obsrv_{1:\gmplen} = \obs_{1:\gmplen}}
\sim \gaussian{\ksmean_\idxdtime}{\kscov_\idxdtime}.
\)
Its moments can be computed from the filter moments via the Rauch--Tung--Striebel (RTS) smoother recursion
\begin{align*}
  \gls{ks:mean} & \defeq \kfmean_\idxdtime + \ksgain_\idxdtime (\ksmean_{\idxdtime + 1} - \kfpmean_{\idxdtime + 1})                    \\
  \gls{ks:cov}  & \defeq \kfcov_\idxdtime + \ksgain_\idxdtime (\kscov_{\idxdtime + 1} - \kfpcov_{\idxdtime + 1}) (\ksgain_\idxdtime)\T
\end{align*}
with $\gls{ks:gain} \defeq \kfcov_\idxdtime \gmpA_\idxdtime\T (\kfpcov_{\idxdtime + 1})\inv$, $\ksmean_{\gmplen} = \kfmean_{\gmplen}$, and $\kscov_{\gmplen} = \kfcov_{\gmplen}$.
See \citet{Sarkka2023BayesianFiltering} for an in-depth introduction to Kalman filtering and RTS smoothing.

\subsection{Spatiotemporal Regression}
A major application of state-space models is spatiotemporal Gaussian Process (GP) regression \citep{Hartikainen2010KalmanFiltering,Sarkka2013SpatiotemporalLearning}.
Suppose we aim to learn a function \(\gls{str:targetfn} : \gls{str:temporalinputspace} \times \gls{str:spatialinputspace} \to \R\) from training data \(\{((\gls{str:ttrain}, \gls{str:xstrain}), \obs_\idxdtime)\}_{\idxdtime=1}^{\ntstrain}\) with \(\ttrain{\idxdtime} \in \temporalinputspace\), \(\xstrain{\idxdtime} \in \spatialinputspace^{\obsdim_\idxdtime}\), and \(\obs_\idxdtime \in \R^{\obsdim_\idxdtime}\).
We assume a GP prior \(\gls{str:gpprior} \sim \gp{\gls{str:meanfn}}{\gls{str:covfn}}\) for $\strtargetfn$.
If the multi-output GP $\gls{stsgmp}(t, \vx) \defeq (\partial_t^i \gpprior(t, \vx))_{i = 0}^{\outputdim - 1} \in \R^{\gls{stsgmp:outputdim}}$ defined by \(\gpprior\) and \(\outputdim - 1\) of its time derivatives is a \emph{space-time separable Gauss--Markov process}\footnotemark (STSGMP), then one can construct an equivalent linear-Gaussian state-space model (see \cref{lem:discretized-stsgmp-is-markov}) with \(\gmp_\idxdtime \defeq \stsgmp(\ttrain{\idxdtime}, \xsall) \in \R^{\outputdim \times \nxsall}\), where \(\gls{str:xsall} \in \spatialinputspace^{\gls{str:nxsall}}\) such that $\xstrain{\idxdtime} \subset \xsall$ for all $\idxdtime = 1, \dotsc, \gmplen$.
\footnotetext{%
  See \cref{sec:st-gpr-ssm}.
  While not every GP prior induces an STSGMP, a broad class of common spatiotemporal models do (see \cref{rem:gp-priors-as-stsgmps} for details).
}
Therefore, assuming that $\xstrain{\idxdtime} = \xsall$ and $\obsdim_\idxdtime = \nxsall$, the computational cost of spatiotemporal GP regression can be reduced from \(\mathcal{O}(\ntstrain^3\nxsall^3)\) to \(\mathcal{O}(\ntstrain\nxsall^3)\) via Bayesian filtering and smoothing.

  \section{COMPUTATION-AWARE KALMAN FILTERING}
\label{sec:matrix-free-kalman-filtering}
While filtering and smoothing are efficient in time, they scale prohibitively with the dimension \(\gmpdim\) of the state space.
Direct implementations of the Kalman filter incur two major computational challenges that are addressed with the CAKF:
\begin{enumerate}[label=(C\arabic*)]
  \item \label{ch:statecov-space}
        The state covariances $\kfpcov_\idxdtime, \kfcov_\idxdtime \in \R^{\gmpdim \times \gmpdim}$ need to be stored in memory, requiring $\mathcal{O}(\gmpdim^2)$ space.
  \item \label{ch:gram-inv}
        The inversion of the innovation matrix $\kfgram_\idxdtime \in \R^{\obsdim_\idxdtime \times \obsdim_\idxdtime}$ costs $\mathcal{O}(\obsdim_\idxdtime^3)$ time and $\mathcal{O}(\obsdim_\idxdtime^2)$ space.
\end{enumerate}
Both of these quickly become prohibitive if $\gmpdim$ and/or $\obsdim_\idxdtime$ is large.
To mitigate these costs, we apply iterative, matrix-free linear algebra in the Kalman recursion.

\subsection{From Matrix-y to Matrix-Free}
\label{sec:projection}
We start by noting that the Kalman filter's update step at time $\idxdtime$ conditions the predictive belief $\condrv{\gmp_\idxdtime \mid \obsrv_{1:\idxdtime-1} = \obs_{1:\idxdtime - 1}}$ on the observation that $\obsrv_\idxdtime = \obs_\idxdtime$.
To reduce both the time and memory complexity of the update step, we project both sides of the observation onto a low-dimensional subspace:
\(
\gls{cakf:projobsrv}
\defeq \cakfacts_\idxdtime\T \obsrv_\idxdtime
= \cakfacts_\idxdtime\T \obs_\idxdtime
\rdefeq \gls{cakf:projobs},
\)
where $\gls{cakf:acts} \in \R^{\obsdim_\idxdtime \times \gls{cakf:projobsdim}}$ with $\cakfprojobsdim_\idxdtime \ll \obsdim_\idxdtime$.
The corresponding modified observation model then reads
\begin{equation}
  \label{eqn:proj-observation}
  \cakfprojobsrv_\idxdtime = \underbracket[0.1ex]{\cakfacts_\idxdtime\T \obsH_\idxdtime}_{\rdefeq \gls{cakf:projH}} \gmp_\idxdtime + \underbracket[0.1ex]{\cakfacts_\idxdtime\T \obsnoise_\idxdtime}_{\rdefeq \gls{cakf:projobsnoise}} \in \R^{\cakfprojobsdim_\idxdtime},
\end{equation}
where $\cakfprojobsnoise_\idxdtime \sim \gaussian{\vec{0}}{\cakfprojobsnoisecov_\idxdtime}$ with $\gls{cakf:projobsnoisecov} := \cakfacts_\idxdtime\T \obsnoisecov_\idxdtime \cakfacts_\idxdtime$.
Consequently, the modified filtering equations can be obtained from \eqref{eq:kf-update-step} by substituting $\obs_\idxdtime \mapsto \cakfprojobs_\idxdtime$, $\obsH_\idxdtime \mapsto \cakfprojH_\idxdtime$ and $\obsnoisecov_\idxdtime \mapsto \cakfprojobsnoisecov_\idxdtime$.
The inversion of the innovation matrix $\gls{cakf:projgram} \in \R^{\cakfprojobsdim \times \cakfprojobsdim}$ in the projected update step then costs $\mathcal{O}(\cakfprojobsdim_\idxdtime^3)$ time and $\mathcal{O}(\cakfprojobsdim_\idxdtime^2)$ memory, which solves \ref{ch:gram-inv}.

Since $\cakfacts_\idxdtime$ is not square, the projection results in a loss of information and the filtering moments $\set{\gls{cakf:mean}, \gls{cakf:cov}}_{\idxdtime = 0}^{\gmplen}$ and $\set{\gls{cakf:pmean}, \gls{cakf:pcov}}_{\idxdtime = 0}^{\gmplen}$ obtained from the Kalman recursion with the modified observation model \labelcref{eqn:proj-observation} are only approximations of the corresponding moments from the unmodified Kalman recursion.
We can choose the columns of $\cakfacts_\idxdtime$, the \emph{actions}, such that they retain the most informative parts of the observation, keeping the approximation error small; more on this in \cref{sec:conjugate_policy}.
Moreover, we show in \cref{sec:theory} that the approximation error in the state mean will be accounted for by an increase in the corresponding state covariance and hence our inference procedure is \emph{computation-aware} (in the sense of \cite{Wenger2022PosteriorComputational}).
In essence, this is because we made the projection onto $\cakfacts_\idxdtime$ part of the modified observation model \labelcref{eqn:proj-observation}, i.e., the likelihood accounts for the fact that we do not observe the data $\obs_\idxdtime$ in the orthogonal complement of $\linspan{\cakfacts_\idxdtime}$.
The resulting posterior is sometimes called a \emph{partial posterior}, and such posteriors are known to provide sensible uncertainty quantification under certain technical assumptions \citep{Cockayne2022Calibrated}.

Since $\cakfprojobsdim_\idxdtime \ll \obsdim_\idxdtime$, one can show that the updated state covariance $\cakfcov_\idxdtime$ under the modified observation model differs from the corresponding predictive state covariance $\cakfpcov_\idxdtime$ by a low-rank downdate
\(
\cakfcov_\idxdtime = \cakfpcov_\idxdtime - \cakfpcov_\idxdtime \cakfprojH_\idxdtime\T \cakfprojgramlsqrt_\idxdtime (\cakfpcov_\idxdtime \cakfprojH_\idxdtime\T \cakfprojgramlsqrt_\idxdtime)\T,
\)
where $\gls{cakf:projgramlsqrt} \in \R^{\cakfprojobsdim_\idxdtime \times \cakfprojobsdim_\idxdtime}$ with $\cakfprojgramlsqrt_\idxdtime \cakfprojgramlsqrt_\idxdtime\T = \cakfprojgram_\idxdtime\inv$.
It turns out that the recursion for the state covariances in the Kalman filter is compatible with the low-rank downdate structure.
More precisely, in \cref{prop:kf-ddcov}, we show that %
\begin{align*}
  \cakfpcov_\idxdtime & = \gmpcov_\idxdtime - \cakfpdd_\idxdtime (\cakfpdd_\idxdtime)\T, \\
  \cakfcov_\idxdtime  & = \gmpcov_\idxdtime - \cakfdd_\idxdtime \cakfdd_\idxdtime\T,
\end{align*}
where \(\gls{cakf:pdd} = \gmpA_{\idxdtime - 1} \cakfdd_{\idxdtime - 1}\) and \(\gls{cakf:dd}  = (\cakfpdd_\idxdtime \ \ \ \cakfpcov_\idxdtime \cakfprojH_\idxdtime\T \cakfprojgramlsqrt_\idxdtime)\).
Incidentally, this observation solves \ref{ch:statecov-space}: When implementing the CAKF using the recursions from \cref{prop:kf-ddcov}, we only need access to matrix-vector products $\gmpcov_\idxdtime \vv$, $\gmpA_\idxdtime \vv$, $\obsH_\idxdtime\T \vv$, and $\obsnoisecov_\idxdtime \vv$ with $\cakfprojobsdim_\idxdtime + 1$ vectors $\vv$.
In many cases, such matrix-vector products can be efficiently implemented or accurately approximated in a ``matrix-free'' fashion, i.e., without needing to store the matrix in memory, at cost (much) less than $\mathcal{O}(\gmpdim^2)$ space and sometimes less than $\mathcal{O}(\gmpdim^2)$ time (though matrix-free time complexity is often higher for dense matrices).
For instance, this is possible if $\set{\gmp_\idxdtime}_{\idxdtime = 0}^{\gmplen}$ is a discretization of a continuous-time space-time separable Gauss--Markov process with known covariance function, in which case an expression for the entries $(\gmpcov_k)_{ij}$ of the state covariance is known.
We emphasize the matrix-free implementation of our algorithm by highlighting matrices that are not instantiated in memory (e.g., $\matfree{\cakfpcov_\idxdtime}$)
in \cref{alg:mfkf,alg:projected_update,alg:mfks}.
Assuming the rank of the downdates in \cref{prop:kf-ddcov} is small (c.f.~\cref{sec:truncation}), such a matrix-free implementation of a Kalman filter incurs linear memory cost per time step.

Pseudocode for the procedure outlined above can be found in \cref{alg:mfkf,alg:projected_update}.
In the algorithm, the \textsc{IsMissing} line is included to allow a user to make predictions for intermediate states $k$ for which there is no associated data.
The choice of $\cakfacts_\idxdtime$ is given by a state-dependent \textsc{Policy}.
In \cref{alg:projected_update}, $\cakfacts_\idxdtime$ is selected in batch through a single call to \textsc{Policy} at the beginning of each update step.
This is mostly presented for clarity; in practice we implement the update step as shown in \cref{alg:update_pls}.
\Cref{alg:update_pls} can be derived as successive conditioning of $\cakfpstate_\idxdtime$ on the events $\inprod{\cakfact_\idxdtime^{(i)}}{\obsrv_\idxdtime}[2] = \inprod{\cakfact_\idxdtime^{(i)}}{\obs_\idxdtime}[2]$ for $i = 1, \dotsc, \cakfprojobsdim_\idxdtime$, where the actions $\gls{cakf:act}$ form the columns of $\cakfacts_\idxdtime$.
One can show that this is equivalent to conditioning on $\cakfacts_\idxdtime\T \obsrv_\idxdtime = \cakfacts_\idxdtime\T \obs_\idxdtime$.
However, such a sequential selection of the actions through calls to \textsc{Policy} that can access the current state of the iteration (e.g., through data residuals) allows the actions to adapt to the problem more effectively.

\begin{algorithm}
  \caption{Computation-Aware Kalman Filter (CAKF)}
  \label{alg:mfkf}
  \begin{algorithmic}[0]
    \small
    \Function{Filter}{$\cakfmean_0, \{ \gmpcov_\idxdtime, \gmpA_\idxdtime, \gmpb_\idxdtime, \obsH_\idxdtime, \obsnoisecov_\idxdtime, \obs_\idxdtime \}_{\idxdtime = 0}^{\gmplen}$}
      \State $\cakftdd_0 \gets (\quad) \in \R^{\gmpdim \times 0}$
      \For{$\idxdtime = 1, \dotsc, \gmplen$}
        \State $\just{\cakfpdd_\idxdtime}{\cakfpmean_\idxdtime} \gets \matfree{\gmpA_{\idxdtime - 1}}{\cakfmean_{\idxdtime - 1}} + \gmpb_{\idxdtime - 1}$ \Comment{Predict}
        \State $\cakfpdd_\idxdtime \gets \matfree{\gmpA_{\idxdtime - 1}}{\cakftdd_{\idxdtime - 1}}$
        \If{\(\neg \Call{IsMissing}{\obs_\idxdtime}\)} %
          \State \(
          \cakfmean_\idxdtime, \cakfdd_\idxdtime
          \gets
          \Call{Update}{
            \cakfpmean_\idxdtime, \cakfpdd_\idxdtime,
            \dotsc %
          }
          \)
        \Else
          \State $\cakfmean_\idxdtime, \cakfdd_\idxdtime \gets \cakfpmean_\idxdtime, \cakfpdd_\idxdtime$
        \EndIf
        \State $\cakftdd_\idxdtime \gets \Call{Truncate}{\cakfdd_\idxdtime}$
      \EndFor
      \State \textbf{return} \( \{ \cakfmean_\idxdtime, \cakfdd_\idxdtime \}_{\idxdtime = 0}^{\gmplen} \)
    \EndFunction
  \end{algorithmic}
\end{algorithm}

\begin{algorithm}
  \caption{CAKF Update Step}
  \label{alg:projected_update}
  \begin{algorithmic}[0]
    \small
    \vspace{-.55ex}
    \Function{Update}{$\cakfpmean, \cakfpdd, \gmpcov, \obsH, \obsnoisecov, \obs$}
      \State $\just{\cakfprojH\T}{\matfree{\cakfpcov}} \gets \matfree{\gmpcov - \cakfpdd (\cakfpdd)\T}$
      \State $\just{\cakfprojH\T}{\cakfacts} \gets \Call{Policy}{\cakfpmean, \matfree{\cakfpcov}, \dotsc}$
      \State $\just{\cakfprojH\T}{\cakfprojH\T} \gets \matfree{\obsH\T}{\cakfacts}$ %
      \State $\just{\cakfprojH\T}{\cakfprojobsnoisecov} \gets \cakfacts\T \matfree{\obsnoisecov}{\cakfacts}$
      \State $\just{\cakfprojH\T}{\cakfprojobs} \gets \cakfacts\T \obs$
      \State $\just{\cakfprojH\T}{\cakfprojgram} \gets \cakfprojH \matfree{\cakfpcov}{\cakfprojH\T} + \cakfprojobsnoisecov$
      \State $\just{\cakfprojH\T}{\cakfprojgramlsqrt} \gets \Call{lsqrt}{\cakfprojgram\pinv}$
      \State $\just{\cakfprojH\T}{\glslink{cakf:w}{\cakfw}} \gets \cakfprojH\T \cakfprojgram\pinv (\cakfprojobs - \cakfprojH \cakfpmean)$
      \State $\just{\cakfprojH\T}{\glslink{cakf:W}{\cakfW}} \gets \cakfprojH\T \cakfprojgramlsqrt$
      \State $\just{\cakfprojH\T}{\cakfmean} \gets \cakfpmean + \matfree{\cakfpcov}{\cakfw}$
      \State \(
      \just{\cakfprojH\T}{\cakfdd}
      \gets
      \begin{pmatrix}
        \cakfpdd & \matfree{\cakfpcov}{\cakfW}
      \end{pmatrix}
      \)
      \State \textbf{return} $(\cakfmean, \cakfdd)$
    \EndFunction
  \end{algorithmic}
\end{algorithm}

\subsection{Downdate Truncation} \label{sec:truncation}

While the algorithm is matrix-free in the sense of not needing to compute and store $\gmpdim \times \gmpdim$ matrices, the accumulation of the downdate matrices $\cakfdd_\idxdtime$ results in an $\mathcal{O}(\gmpdim \sum_{l = 1}^\idxdtime \cakfprojobsdim_l)$ memory cost at step $\idxdtime$, which can easily exceed $\mathcal{O}(\gmpdim^2)$ as $\idxdtime$ grows.
To address this, we introduce an optimal truncation of the downdate matrices in the \textsc{Truncate} procedure to control the memory requirements of the algorithm.

Consider the square root $\cakfdd_\idxdtime \in \R^{\gmpdim \times \gls{cakf:ddrank}}$ of a belief covariance downdate.
We truncate the downdate matrices by selecting $\gls{cakf:tddrank} \le \cakfddrank_\idxdtime$ (typically $\cakftddrank_\idxdtime \ll \cakfddrank_\idxdtime$), $\gls{cakf:tdd} \in \R^{\gmpdim \times \cakftddrank_\idxdtime}$, and $\mN_\idxdtime \in \R^{\gmpdim \times (\cakfddrank_\idxdtime - \cakftddrank_\idxdtime)}$ such that
\(
\cakfdd_\idxdtime \cakfdd_\idxdtime\T = \cakftdd_\idxdtime \ps{\cakftdd_\idxdtime}\T + \mN_\idxdtime \mN_\idxdtime\T
\)
and approximating $\cakfdd_\idxdtime \approx \cakftdd_\idxdtime$ as well as
\(
\cakfcov_\idxdtime
\approx \gls{cakf:tcov}
\defeq \gmpcov_\idxdtime - \cakftdd_\idxdtime \ps{\cakftdd_\idxdtime}\T.\)
Noting that
\(
\cakftcov_\idxdtime =\cakfcov_\idxdtime + \mN_\idxdtime \mN_\idxdtime\T
\),
we realise that the truncation of the downdate can be interpreted as the addition of independent \emph{computational uncertainty} \citep{Wenger2022PosteriorComputational}: additional noise $\rvq^\text{comp}_\idxdtime \sim \gaussian{\vec{0}}{\mN_\idxdtime \mN_\idxdtime\T}$ added to the posterior covariance to account for uncertainty due to incomplete computation, in this case, truncation.
To represent this, we augment the dynamics model with additional prior states $\smash{\gmptstate{\idxdtime}, \gmppstate{\idxdtime}}$ as visualized in \cref{fig:pgm}, such that
\(
\gls{gmp:tstate} \defeq \gmp_\idxdtime + \smash{\rvq^\text{comp}_\idxdtime},
\)
\(
\smash{\gmppstate{\idxdtime + 1}} \defeq \gmpA_\idxdtime \smash{\gmptstate{\idxdtime}} + \gmpb_\idxdtime + \gmpnoise_\idxdtime,
\)
and
\(
\gmp_\idxdtime \defeq \smash{\gmppstate{\idxdtime}}.
\)
Even though the truncation leads to a further approximation of the state beliefs, this approximation will be \emph{conservative}, which directly follows from the computational noise interpretation above.

We truncate by computing a singular-value decomposition of $\cakfdd_\idxdtime \cakfdd_\idxdtime\T$, and dropping the subspace corresponding to the smallest singular vectors.
By the Eckart--Young--Mirsky theorem \citep{Mirsky1960}, this truncation is optimal with respect to all unitarily invariant matrix norms.
The effect of rank truncation is that at most $\mathcal{O}(d \cakftddrank_\idxdtime)$ memory is required to store the downdate matrices, and that the cost of computing matrix-vector products with the truncated covariance $\cakftcov_\idxdtime$ is at most $\mathcal{O}(\rho_\idxdtime + d \cakftddrank_\idxdtime)$, where $\rho_\idxdtime$ is the cost of computing a matrix-vector product with $\gmpcov_\idxdtime$.

\subsection{Choice of Policy} \label{sec:conjugate_policy}

It remains to specify a \textsc{Policy} defining the actions \(\cakfacts_\idxdtime\).
This can have a significant impact on the algorithm, both from the perspective of how close the CAKF states are to the states of the true Kalman filter and in terms of its computational cost.
Heuristically, we would like to make $\cakfprojobsdim_k$ as small as possible while keeping $\gaussian{\cakfmean_\idxdtime}{\cakfcov_\idxdtime}$ close to $\gaussian{\kfmean_\idxdtime}{\kfcov_\idxdtime}$.
We discuss and compare a number of natural policy choices in more detail in \Cref{sec:details-policy-choice}.
In the experiments in \cref{sec:experiments} we exclusively use Lanczos/CG-based directions, corresponding to choosing the current residual $\smash{\cakfresidual_\idxdtime\iidx}$ as the action in iteration $i$ of \cref{alg:update_pls}, i.e., $\textsc{Policy}(i, \dotsc) = \smash{\cakfresidual_\idxdtime\iidx}$.
We found these to perform well empirically compared to other choices (see \Cref{fig:work-precision-policies}), and similar policies have been found effective in related work \citep{Cockayne2019BayesCG,Wenger2022PosteriorComputational}.

  \section{COMPUTATION-AWARE RTS SMOOTHING}
\label{sec:caks}
\begin{algorithm}
  \caption{Computation-Aware RTS Smoother (CAKS)}
  \label{alg:mfks}
  \begin{algorithmic}[0]
    \small
    \Function{Smooth}{$\{ \dotsc, \cakfmean_\idxdtime, \cakfdd_\idxdtime, \cakfw_\idxdtime, \cakfW_\idxdtime, \dotsc \}_{\idxdtime = 1}^{\gmplen}$}
      \State $\just{\caksW_{\gmplen}}{\caksw_{\gmplen}} \gets \cakfw_{\gmplen}$
      \State $\caksW_{\gmplen} \gets \cakfW_{\gmplen}$
      \For{$\idxdtime = \gmplen-1, \dotsc, 0$}
        \State $\just{\caksW_{\gmplen}}{\gls{caks:mean}} \gets \cakfmean_\idxdtime + \matfree{\cakfcov_\idxdtime \gmpA_\idxdtime\T}{\caksw_{\idxdtime+1}}$
        \State $\just{\caksW_{\gmplen}}{\gls{caks:dd}} \gets
          \begin{pmatrix}
            \cakfdd_\idxdtime & \matfree{\cakfcov_\idxdtime \gmpA_\idxdtime\T}{\caksW_{\idxdtime+1}}
          \end{pmatrix}
        $
        \State $\just{\caksW_{\gmplen}}{\gls{caks:w}} \gets \cakfw_\idxdtime + \matfree{(\mI - \cakfW_\idxdtime \cakfW_\idxdtime\T \cakfpcov_\idxdtime) \gmpA_\idxdtime\T}{\caksw_{\idxdtime+1}}$
        \State $\just{\caksW_{\gmplen}}{\caksW_\idxdtime} \gets
          \begin{pmatrix}
            \cakfW_\idxdtime &
            \matfree{(\mI - \cakfW_\idxdtime \cakfW_\idxdtime\T \cakfpcov_\idxdtime) \gmpA_\idxdtime\T}{\caksW_{\idxdtime+1}}
          \end{pmatrix}
        $
        \State $\just{\caksW_{\gmplen}}{\gls{caks:W}} \gets \Call{Truncate}{\caksW_\idxdtime}$
      \EndFor
      \State \textbf{return} \( \{ \caksmean_\idxdtime, \caksdd_\idxdtime \}_{\idxdtime = 0}^{\gmplen} \)
    \EndFunction
  \end{algorithmic}
\end{algorithm}

If the state-space dimension $\gmpdim$ is large, naive implementations of the RTS smoother face similar challenges to those outlined for the Kalman filter in \cref{sec:matrix-free-kalman-filtering}.
This is due to the \emph{smoother gain} matrices $\ksgain_\idxdtime$ defined in \cref{sec:lgssm} needing to be stored and inverted at $\mathcal{O}(\gmpdim^3)$ time and $\mathcal{O}(\gmpdim^2)$ memory cost.
Fortunately, we can apply a similar strategy to \cref{sec:projection} to make the smoother matrix-free.
Specifically, in \cref{prop:inverse-free-smoother} we show that the mean and covariance of the RTS smoother can be computed from quantities precomputed in the Kalman filter, i.e., without the need to compute any additional inverses:
\(
\ksmean_\idxdtime = \kfmean_\idxdtime + \kfcov_\idxdtime \gmpA_\idxdtime\T \ksw_\idxdtime
\)
and
\(
\kscov_\idxdtime = \gmpcov_\idxdtime - \ksdd_\idxdtime(\ksdd_\idxdtime)\T
\)
with
\(
\ksdd_\idxdtime =
\begin{pmatrix}
  \kfdd_\idxdtime & \kfcov_\idxdtime \gmpA_\idxdtime\T \ksW_\idxdtime
\end{pmatrix}
,
\)
where recursive expressions for $\ksw_\idxdtime$ and $\ksW_\idxdtime$ are given in \cref{eq:smooth_w,eq:smooth_W}.
Hence, just as for the filtering covariances, the smoother covariances take the form of a downdated prior covariance.
The terms $\ksw_\idxdtime$ and $\ksW_\idxdtime$ can be efficiently computed from quantities cached in \cref{alg:mfkf,alg:projected_update} without materializing any $\mathcal{O}(\gmpdim^2)$ matrices in memory.
Applying \cref{prop:inverse-free-smoother}\footnote{We would like to point out that \cref{prop:inverse-free-smoother} may be of independent interest since it is (to the best of our knowledge) a novel result about the RTS smoother that can be used as an alternative to the standard RTS smoothing recursions for increased numerical stability.
} in matrix-free form to the modified state-space model of the CAKF introduced in \cref{sec:matrix-free-kalman-filtering} yields \cref{alg:mfks} -- the \emph{computation-aware RTS smoother} (CAKS).

While the cost of filtering is reduced by the low-dimensional projection of the data, the same does not hold for the smoother.
Examining \cref{alg:projected_update} we see that $\cakfpcov_\idxdtime$ only appears in a product with $\cakfprojH$, whereas in \cref{alg:mfks} products with $\cakfpcov_\idxdtime$ appear directly.
It is also necessary to truncate the directions $\caksW_\idxdtime$ accumulated over the course of the smoother to mitigate a further $\mathcal{O}(\gmpdim \sum_{l = k}^{\gmplen} \cakfprojobsdim_l)$ storage cost.
This is implemented using the same procedure %
as in \cref{sec:truncation}.

In \cref{sec:matheron-sampling} we use Matheron's rule to derive an algorithm for sampling from the posterior process, i.e., the CAKS states $\caksstate_\idxdtime$, pseudocode for which is given in \cref{alg:cakf-caks-sampler}.
Counterintuitively, \cref{alg:cakf-caks-sampler} can draw samples from the smoother states without the requirement to run the CAKS, since all necessary quantities have already been computed in the CAKF.

In \cref{sec:temporal-interpolation} we show that both the CAKF and CAKS can also be applied if the dynamics model is a continuous-time Gauss--Markov process.
In this case, the CAKS provides efficient access to intermediate posterior states $\condrv{\gmp(t) \given \obsrv_{1:\gmplen} = \obs_{1:\gmplen}}$ for arbitrary $t \in \mathbb{T}$.

  \section{THEORETICAL ANALYSIS}
\label{sec:theory}

\vspace{-0.5em} %

\subsection{Computational Complexity}
\label{sec:computational-complexity}
\vspace{-0.5em} %
As mentioned in \cref{sec:matrix-free-kalman-filtering}, the CAKF and CAKS assume that we can efficiently evaluate matrix-vector products with $\gmpcov_\idxdtime$, $\gmpA_\idxdtime$, $\obsH_\idxdtime\T$, and $\obsnoisecov_\idxdtime$, without synthesizing the matrices in memory.
More precisely, we assume that matrix-vector products with these matrices can be computed with a memory complexity linear in the larger of their two dimensions and with the same worst-case time complexity.

\textbf{Filtering}
The CAKF predict step at time $\idxdtime$ costs at most
\(
\mathcal{O}(\gmpdim^2 \cakftddrank_{\idxdtime - 1})
\)
time and
\(
\mathcal{O}(\gmpdim \cakftddrank_{\idxdtime - 1})
\)
memory.
The CAKF update step at time $\idxdtime$ costs at most
\(
\mathcal{O} \ps[\big]{
  (
  \gmpdim \obsdim_\idxdtime
  + \obsdim_\idxdtime^2
  + \gmpdim^2
  ) \cakfprojobsdim_\idxdtime
}
\)
time and
\(
\mathcal{O} \ps[\big]{
  (
  \obsdim_\idxdtime
  + \gmpdim
  )
  \cakfprojobsdim_\idxdtime
}
\)
memory.
The SVD downdate truncation has a time complexity of at most $\mathcal{O}(\gmpdim (\cakftddrank_{\idxdtime - 1} + \cakfprojobsdim_\idxdtime)^2)$.

\textbf{Smoothing}
CAKS\;iteration\;$\idxdtime$\;costs %
\(
\mathcal{O} \ps[\big]{
  \gmpdim (
  \gmpdim
  + \cakfprojobsdim_\idxdtime
  ) \cakftddrank_{\idxdtime + 1}
  + \gmpdim \cakfprojobsdim_\idxdtime^2
}
\)
time and
\(
\mathcal{O} \ps[\big]{
  (
  \gmpdim
  + \cakfprojobsdim_\idxdtime
  ) \cakftddrank_{\idxdtime + 1}
}
\)
memory.

\textbf{Simplified Complexities}
In practice, especially for spatiotemporal GP regression, it virtually always holds that $\gmpdim = \mathcal{O}(\obsdim_\idxdtime)$.
With this assumption, the time and memory complexities of the CAKF update step simplify to $\mathcal{O}(\gmpdim^2 \cakfprojobsdim_\idxdtime)$ and $\mathcal{O}(\gmpdim \cakfprojobsdim_\idxdtime)$, respectively.
Similarly, iteration $\idxdtime$ of the smoother then costs $\mathcal{O}(\gmpdim^2 (\cakftddrank_{\idxdtime + 1} + \cakfprojobsdim_\idxdtime))$ time and $\mathcal{O}(\gmpdim \cakftddrank_{\idxdtime + 1})$ memory.
It is also sometimes desirable to set $\cakftddrank_{\idxdtime} \le \cakfmaxtddrank$ and $\cakfprojobsdim_\idxdtime \le \cakfmaxprojobsdim$, i.e., uniform in $\idxdtime$, with $\cakfmaxtddrank = \mathcal{O}(\cakfmaxprojobsdim)$.
In this case, running both CAKF and CAKS for $\gmplen$ time steps results in \textbf{worst-case time and memory complexities of \( \mathcal{O} \ps[\big]{\gmplen \gmpdim^2 \cakfmaxprojobsdim} \) and \( \mathcal{O} \ps[\big]{ \gmplen \gmpdim \cakfmaxprojobsdim}\)}.

\subsection{Error Bound for Spatiotemporal Regression}

It is important to understand the impact of the approximations made by the CAKF and CAKS on the resulting predictions.
So far we have argued informally, that the additional uncertainty of the CAKS captures the approximation error.
We will now make this statement rigorous for the case of spatiotemporal regression.

\begin{restatable}[Pointwise Worst-Case Prediction Error]{theorem}{thmErrorBoundCAKF}
  \label{thm:pointwise-error-cakf}
  Let \(\inputspace = \temporalinputspace \times \spatialinputspace\) and define a space-time separable Gauss--Markov process \(\stsgmp \sim \gp{\stsgmpmeanfn}{\stsgmpcovfn}\) such that its first component \(\gpprior \defeq \stsgmp_0 \sim \gp{\meanfn}{\covfn}\) defines a prior for the latent function \(\strtargetfn \in \rkhs{\covfn}\) generating the data, assumed to be an element of the RKHS defined by \(\gls{str:rkhs}\).
  Given observation noise \(\gls{str:noisescale}^2\geq 0\), let \(\gls{str:obsfn}(\cdot) \in \rkhs{\noisycovfn}\) be the observed process with \(\gls{str:noisycovfn}(\vz, \vz') \defeq \covfn(\vz, \vz') + \sigma^2 \delta(\vz, \vz')\).
  Given training inputs \(\gls{str:zstrain} \in \inputspace^{\gls{str:ntraindata}}\) and targets \(\gls{str:ystrain} = \strobsfn(\zstrain) \in \R^{\ntraindata}\), let \(\caksmean(\vz \mid \strobsfn)\) and \(\cakscov(\vz)\) be the mean and variance of the CAKS for an arbitrary test input \(\vz = (t, \vx) \in \inputspace \setminus \zstrain\).
  Then
  \begin{equation}
    \label{eqn:pointwise-error-cakf}
    \sup\limits_{y \in \rkhs{\noisycovfn} \setminus \set{0}} \frac{\abs{y(\vz) - \caksmean(\vz \mid y)_0}}{\norm{y}_{\rkhs{\noisycovfn}}} = \sqrt{\cakscov(\vz)_{0, 0} + \sigma^2}.
  \end{equation}
  If \(\noisescale^2=0\), this also holds for training inputs \(\vz \in \zstrain\).
\end{restatable}

The proof can be found in \cref{sec:error-bound}.
\Cref{thm:pointwise-error-cakf} says that the (relative) worst-case error of the CAKS's posterior mean \(\caksmean(\cdot \mid \strobsfn)_0\) computed for data from the data-generating process \(\strobsfn(\cdot)\) is tightly bounded by its predictive standard deviation \((\cakscov(\vz)_{0, 0} + \sigma^2)^{\nicefrac{1}{2}}\) (assuming no truncation).
Importantly, this guarantee is of the same form as the one satisfied by the \emph{exact} posterior predictive \(\gp{\gls{str:postmeanfn}}{\gls{str:postcovfn} + \sigma^2\delta}\) for the same prior (see Prop.~3.8 of \citet{Kanagawa2018GaussianProcesses}), %
except for the corresponding \emph{approximations}.
In this sense, \textbf{\Cref{thm:pointwise-error-cakf} makes the nomenclature \emph{computation-aware} rigorous, since both the error due to finite data and the inevitable approximation error incurred by using \(\caksmean_0 \approx \postmeanfn\) for prediction is quantified by its uncertainty estimate.
}
Finally, truncation only increases the marginal variance of the CAKS, which leads us to conject that the same guarantee as in \cref{eqn:pointwise-error-cakf} holds with inequality for truncation.

  \section{RELATED WORK}
Reducing the cost of Kalman filtering and smoothing in the high-dimensional regime is a fundamental problem.
A large family of methods accelerates Kalman filtering by truncating state covariance matrices.
This includes the \emph{ensemble Kalman filter} (EnKF) \citep{Evensen1994EnKF} and its variants, as well as the \emph{reduced-rank Kalman filter} (RRKF) \citep{Schmidt2023RankReducedKalman}.
However in contrast to this work, these approaches truncate the full state covariance rather than downdates, which can lead to overconfident uncertainty estimates \citep[Appendix E]{Schmidt2023RankReducedKalman}.
Some authors also propose dimension reduction techniques for the state space \citep[e.g.,][]{Solonen2016DimensionReduction} with the notable exception of \citet{Berberidis2017}, who focus on data dimension reduction, as we do here.
\citet{Bardsley2011} use a similar Lanczos-inspired methodology; in certain settings, this is equivalent to low-dimensional projections of the data. %
The main application of high-dimensional filtering and smoothing considered in \cref{sec:theory,sec:experiments} is to spatiotemporal GP regression.
This connection was first expounded in \citet{Hartikainen2010KalmanFiltering,Sarkka2013SpatiotemporalLearning} and generalized to a wider class of covariance functions in \citet{Todescato2020EfficientSpatiotemporal}.
These works focus on discretizing the GP to obtain a state-space model.
One can also apply the Kalman filter directly in the infinite-dimensional setting, as proposed by \citet{Sarkka2012InfiniteDimensionalKalman,Solin2013InfinitedimensionalBayesian}.

The CAKF is a probabilistic numerical method \citep{Hennig2015ProbabilisticNumerics,Cockayne2019BayesianProbabilistic,Hennig2022ProbabilisticNumerics}.
In particular, \cref{alg:update_pls} is closely related to the literature on \emph{probabilistic linear solvers} \citep{Hennig2015LinearSolvers,Cockayne2019BayesCG,wenger2020problinsolve}, which frequently employ the Lanczos process.
The idea of using such solvers for GP regression was explored in \citet{Wenger2022PosteriorComputational}, which first proposed the construction of computation-aware solvers; the sense in which the CAKF is computation-aware is slightly different, in that the truncation of the covariance downdates also plays a role.
\citet{Tatzel2023AcceleratingGeneralized} explore similar ideas in the context of Bayesian generalized linear models.

  \begin{figure}[t]
  \includegraphics[width=\linewidth]{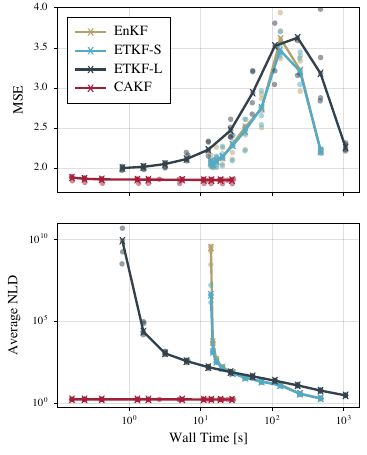}
  \caption{
    Comparison of the CAKF, the EnKF, and two variants of the ETKF on on-model data with state-space dimension $\gmpdim = \num{20000}$ while varying the rank parameters that govern the computational budget of the algorithms.
    The CAKF significantly outperforms the other filter variants for high-dimensional state spaces.
  }
  \label{fig:ensemble-on-model-wall-time}
  \vspace{-1em} %
\end{figure}

\begin{figure*}
  \vspace{-1em} %
  \begin{subfigure}[t]{0.49\textwidth}
    \centering
    \includegraphics[width=\textwidth]{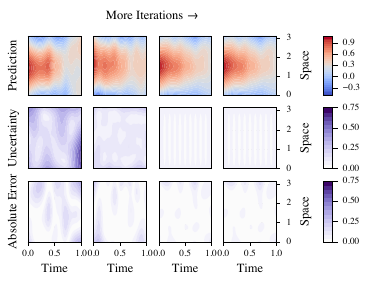}
    \caption{CAKS}
    \label{fig:toy-data-iterations}
  \end{subfigure}
  ~
  \begin{subfigure}[t]{0.49\textwidth}
    \centering
    \includegraphics[width=\textwidth]{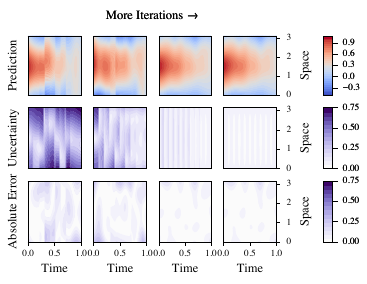}
    \caption{CAKS + Truncation}
    \label{fig:toy-data-iterations-truncation}
  \end{subfigure}
  \vspace{-0.5em}  %
  \caption{
    Predictive mean, predictive standard deviation, and pointwise absolute error for an increasing maximal number of iterations \(\cakfmaxprojobsdim \ge \cakfprojobsdim_\idxdtime\) per time step on a synthetic spatiotemporal regression problem.
  }
  \label{fig:toy-data}
  \vspace{-1em} %
\end{figure*}

\begin{figure*}[t]
  \vspace{-1em} %
  \centering
  \includegraphics[width=\textwidth]{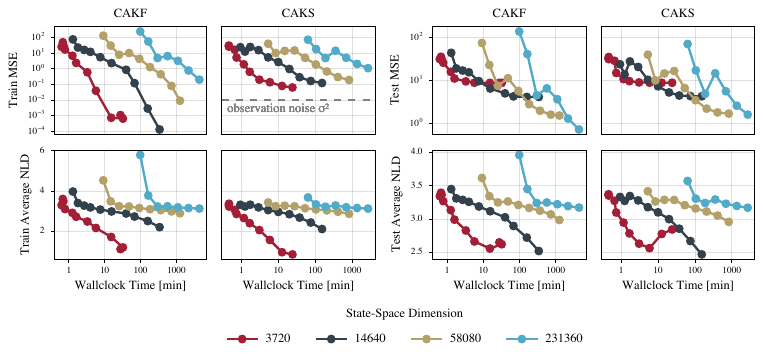}
  \vspace{-2em}  %
  \caption{
    \emph{Work-precision diagrams for the CAKF and CAKS on the ERA5 climate dataset.}
    The plot shows the mean squared error (MSE) and average negative log density (NLD) of the computation-aware filter and smoother for different problem sizes (i.e., state-space dimension) and number of iterations on the train and test set.
    The predictive error measured by test MSE \emph{decreases} with larger problem sizes, while the test NLD \emph{increases}.
    This is because we limit the computational budget and thus run \emph{fewer} iterations for larger problems, i.e., we trade reduced computation cost for increased uncertainty.
  }
  \label{fig:work-precision-era5}
  \vspace{-1em} %
\end{figure*}

\section{EXPERIMENTS}
\label{sec:experiments}
To evaluate the computation-aware Kalman filter and smoother empirically, we apply it to spatiotemporal regression problems with synthetic data and a large-scale dataset from the geosciences.

\textbf{Model}
All experiments will use Gaussian process priors $\gpprior \sim \gp{0}{\covfn}$ with space-time separable covariance functions
\(
\covfn((t, \vx), (t', \vx')) = \covfntime(t, t')\covfnspace(\vx, \vx'),
\)
where \(\covfntime\) is chosen such that the prior can be represented by an equivalent STSGMP (see \cref{sec:st-gpr-ssm}).
We will assume that the data is corrupted by i.i.d.~Gaussian measurement noise with standard deviation \(\lambda\).
See \cref{sec:appendix-experiments} for details on the model hyperparameters for the respective experiments.

\textbf{Evaluation}
We measure the performance of each method via the mean squared error (MSE) (\(\downarrow\)) of its predictive mean as well as via the average negative log density (NLD) (\(\downarrow\)), which additionally takes uncertainty quantification into account.
See \cref{sec:appendix-metrics} for details on the evaluation metrics.

\textbf{Implementation}
A flexible and efficient implementation of the CAKF and CAKS, including Matheron sampling and support for spatiotemporal modeling, is available as an open-source Julia library at
\begin{center}
  \href{https://github.com/marvinpfoertner/ComputationAwareKalmanExperiments.jl}{\faGithub\ / marvinpfoertner / ComputationAwareKalmanExperiments.jl}.
\end{center}
When applying the CAKF and CAKS to separable spatiotemporal Gauss--Markov models, the main performance bottleneck is the computation of matrix-vector products with the prior's state covariance, since this involves a multiplication with a large kernel Gram matrix $\covfnspace(\xsall, \xsall)$.
Our Julia implementation includes a custom CUDA kernel for multiplying with Gramians generated by covariance kernels \emph{without} materializing the matrix in memory.\footnotemark
\footnotetext{For reference, we observed up to a 600-fold speedup over the default CPU implementation when multiplying a $9600\!\times\!9600$ kernel Gram matrix of a three-dimensional Matérn($\nicefrac{3}{2}$) kernel with a $9600\!
    \times\!128$ matrix.}

\textbf{Hardware}
All experiments were run on a single dedicated machine equipped with an Intel i7-8700K CPU with \qty{32}{\giga\byte} of RAM and an NVIDIA GeForce RTX 2080 Ti GPU with \qty{11}{\giga\byte} of VRAM.

\subsection{Comparison to Other Methods}
We compare the performance of the CAKF/CAKS both to the standard Kalman filter and RTS smoother, as well as ensemble Kalman filters.

\textbf{Data}
To isolate the effect of the algorithm on the prediction, we sample on-model datasets from the prior.
To this end, we discretize the prior on regular grids in both time and space and draw a sample from the resulting discretized Gauss--Markov process.
We pick a random subset of these points as training data which are subseqently corrupted by additive Gaussian measurement noise.

\subsubsection{Comparison to EnKF and ETKF}
\label{sec:experiment-ensemble-on-model}
We start by comparing the CAKF against its main competitors: ensemble Kalman filters.
Among those, we consider the ensemble Kalman filter (EnKF) \citep{Evensen1994EnKF} and two variants of the ensemble transform Kalman filter (ETKF) \citep{Bishop2001AdaptiveSampling}.
The variants of the ETKF differ only in their initialization and prediction steps.
Namely, the first variant (ETKF-S) uses the same sampling-based initialization and prediction steps as the EnKF, while the second variant (ETKF-L) uses the Lanczos process for both (see \cref{sec:appendix-on-model-data}).
We vary the rank parameter $r = 1, 2, 4, 8, \dotsc, 1024$ that governs the computational budget of the algorithms, starting at $r = 2$ for the EnKF and the ETKF-S and at $r = 1$ for the ETKF-L and the CAKF (and we only run the latter up to $r = 512$).
For a fair comparison, we set both the maximal number of iterations as well as the truncation rank of the CAKF to the same constant value, i.e., $\cakftddrank_\idxdtime = \cakfmaxprojobsdim_\idxdtime = r$.
Note that the scalability issues of the EnKF and ETKF-S outlined in \cref{sec:comparison-enkf} limit the state-space dimension of this problem.
The results are visualized in \cref{fig:ensemble-on-model-wall-time,fig:ensemble-on-model-rank}.

\textbf{Interpretation}
The CAKF consistently outperforms the ensemble methods in terms of MSE and average NLD across all ranks.
While the difference in MSE is comparatively small, the CAKF achieves a significantly lower average NLD.
This indicates that the uncertainty quantification of the CAKF is considerably more accurate than that of the ensemble methods.

\subsubsection{Comparison to Kalman Filter and RTS Smoother}
As the CAKF and CAKS are approximations to the Kalman filter and RTS smoother, respectively, it is important to assess the approximation error.
To this end, we compute the errors in the mean and covariance estimates provided by the CAKF and CAKS with varying rank parameters $\cakftddrank_\idxdtime = \cakfmaxprojobsdim_\idxdtime \in \set{1, 16, 32, 64}$ as compared to the respective mean and covariance estimates produced by the Kalman filter and RTS smoother on a sufficiently small problem for which we can run these.
The results can be found in \cref{fig:error-dynamics}.

\textbf{Interpretation}
As expected, the error introduced by the CAKF and CAKS decreases with increasing rank.
Moreover, both the error in the mean and the error in the covariance are bounded and stable over time.

\subsection{Impact of Truncation}
To visualize the effect of the downdate truncation, we consider a synthetic dataset generated by adding Gaussian measurement noise to the target function
\(
\strtargetfn(t, x) \defeq \sin(x) \exp(-t).
\)
In \cref{fig:toy-data}, we illustrate the predictive mean of the CAKS, its predictive standard deviation, and the corresponding pointwise (absolute) error, both with (\cref{fig:toy-data-iterations}) and without (\cref{fig:toy-data-iterations-truncation}) truncation.
Each column corresponds to a larger number of iterations \(\cakfmaxprojobsdim\).
Here, the truncation rank is chosen as \(\cakftddrank_\idxdtime = \min(2\cakfprojobsdim_\idxdtime, \cakftddrank_\idxdtime)\).

\textbf{Interpretation}
For an increasing number of iterations, the error in the predictive mean \(\caksmean(\vz)_0\) decreases and its uncertainty \(\cakscov(\vz)_{0,0}\) reduces correspondingly.
When the belief is (optimally) truncated to save memory, the uncertainty tends to increase as \cref{fig:toy-data-iterations-truncation} illustrates.
Notice the trade-off between the degree of truncation and the impact on the prediction.
Finally, \cref{fig:toy-data} also illustrates that the uncertainty bounds the error in the predictive mean as shown in \cref{thm:pointwise-error-cakf}.

\subsection{Large-Scale Climate Dataset}
\label{sec:climate-dataset}
To demonstrate that our approach scales to large, real-world problems, we use the CAKS to interpolate earth surface temperature data over time using an STSGMP prior on the sphere.

\textbf{Data}
We consider the ``$\qty{2}{\meter}$-temperature'' variable from the ERA5 global reanalysis dataset \citep{Hersbach2020ERA5Global}.
The data reside on a 1440\(\times\)721 spatial latitude-longitude grid with an hourly temporal resolution.
For our experiments, we selected the first $\qty{48}{\hour}$ of 2022 with a temporal stride of $\qty{1}{\hour}$, i.e., $\ntstrain = 48$.
To show the effect of different problem sizes on our algorithms, we downsample the dataset by factors of 3, 6, 12, and 24 along both spatial dimensions using nearest neighbor downsampling.
A regular subgrid consisting of \qty{25}{\percent} of the points in the downsampled dataset is used for testing, while the remaining points are used as a training set.
The total number of spatial points and the number of spatial training points for each downsampled version of the dataset can be found in \cref{tbl:era5-sizes}.

\textbf{Evaluation}
We run the CAKF and the CAKS for the four different problem sizes (spatial downsampling factors of 3, 6, 12, and 24) corresponding to increasing state-space dimension (see \cref{tbl:era5-sizes}), for a maximum of \(\approx \num{4000000}\) training datapoints.
For each problem size, we vary the computational budget, defined by the number of actions $\cakfprojobsdim_\idxdtime$ and the rank $\cakftddrank_\idxdtime = \min(2 \cakfprojobsdim_\idxdtime, \cakfddrank_\idxdtime)$ of the downdates after truncation.
As the $\cakfprojobsdim_\idxdtime$ increases, the CAKF and CAKS approach the standard Kalman filter and RTS smoother.
A comparison to these methods on this problem is not practically realisable, and nor is comparison to ensemble Kalman filters; see \cref{sec:comparison-enkf} for a detailed explanation.
For the smallest problem, we use up to \(\cakfprojobsdim_\idxdtime = 2^{10}\) actions, while for the largest problem, we use up to \(\cakfprojobsdim_\idxdtime=2^6\).
The experimental results are visualized in a work-precision diagram in \cref{fig:work-precision-era5}.

\Cref{fig:3-64} was generated by running the CAKF and the CAKS with a spatial downsampling factor of 3, corresponding to a state-space dimension of $\num{231360}$ and $\approx \num{4000000}$ total training data points.
The number of actions is set to $\cakfprojobsdim_\idxdtime = 64$ and the maximal rank after truncation is set to $2 \cakfprojobsdim_\idxdtime = 128 \ge \cakftddrank_\idxdtime$.

\textbf{Interpretation}
As we increase the number of actions, i.e., the computational budget, the MSE and average NLD improve for both the CAKF and CAKS.
As the state-space dimension increases, inference becomes more computationally demanding and the CAKF and CAKS take longer to compute the posterior marginals.
However, with more data, both improve their generalization performance as measured by the MSE.
To stay within a fixed upper limit on the time and memory budget, we constrain the number of iterations \(\cakfprojobsdim_\idxdtime\) more as the problem size increases, which results in an increase in average NLD.
This is an example of the aforementioned trade-off between reduced computational resources and increased uncertainty.
  \section{CONCLUSION}
\label{sec:conclusion}
Kalman filtering and smoothing enable efficient inference in linear-Gaussian state-space models from a set of noisy observations.
However, in many practical applications such as spatiotemporal regression, the latent state is high-dimensional.
This results in prohibitive computational demands.
In this work, we introduced computation-aware versions of the Kalman filter and smoother, which significantly reduce the time and memory complexity, while quantifying their inevitable approximation error via an appropriate increase in predictive uncertainty.
Our experiments show that the CAKF and CAKS significantly outperform their main competitors, ensemble Kalman filtering methods, already on problems with moderately large state-space dimension.
This is mostly due to the fact that, unlike the ensemble methods, the CAKF and CAKS account for their inherent approximation error by design.
Further the CAKF and CAKS scale to significantly larger problems as evidenced by our benchmark experiment on a real-world climate dataset.
A natural next step is to extend our approach such that model selection via evidence maximization becomes possible.
Since the CAKF and CAKS are performing exact inference in a modified linear Gaussian state-space model, this is in theory directly possible by exploiting known techniques for the vanilla filter and smoother \citep[Sec.~16.3.2,][]{Sarkka2023BayesianFiltering}, however, the need for truncation complicates this.

  \acknowledgments{MP and PH gratefully acknowledge financial support by the European Research Council through ERC StG Action 757275 / PANAMA and ERC CoG Action 101123955 / ANUBIS; the DFG Cluster of Excellence “Machine Learning - New Perspectives for Science”, EXC 2064/1, project number 390727645; the German Federal Ministry of Education and Research (BMBF) through the Tübingen AI Center (FKZ: 01IS18039A); the DFG SPP 2298 (Project HE 7114/5-1); and the Carl Zeiss Foundation (project "Certification and Foundations of Safe Machine Learning Systems in Healthcare"); as well as funds from the Ministry of Science, Research and Arts of the State of Baden-Württemberg.
The authors thank the International Max Planck Research School for Intelligent Systems (IMPRS-IS) for supporting MP.
JW was supported by the Gatsby Charitable Foundation (GAT3708), the Simons Foundation (542963), the NSF AI Institute for Artificial and Natural Intelligence (ARNI: NSF DBI 2229929) and the Kavli Foundation.
JC is supported by EPSRC grant EP/Y001028/1.
}

  \bibliographystyle{unsrtnat}
  \bibliography{references}

  \newpage
  \appendix

\numberwithin{figure}{section}
\numberwithin{table}{section}
\numberwithin{algorithm}{section}

\numberwithin{definition}{section}
\numberwithin{theorem}{section}
\numberwithin{proposition}{section}
\numberwithin{lemma}{section}
\numberwithin{corollary}{section}
\numberwithin{remark}{section}

\onecolumn

\thispagestyle{empty}
\aistatstitle{%
  Computation-Aware Kalman Filtering and Smoothing: \\%
  Supplementary Materials%
}

The supplementary materials contain derivations for our theoretical framework and proofs for the mathematical statements in the main text.
We also provide implementation specifics and describe our experimental setup in more detail.

\startcontents[sections]
\vspace{1em}
\printcontents[sections]{l}{1}{\setcounter{tocdepth}{3}}
\vspace{3em}

\section{NOTATION}
\label{sec:notation}
\setglossarystyle{long3col}
\setlength{\glsdescwidth}{0.8\linewidth}

\printunsrtglossary[type=gmp]
\printunsrtglossary[type=obs]
\printunsrtglossary[type=kf]
\printunsrtglossary[type=cakf]
\printunsrtglossary[type=ks]
\printunsrtglossary[type=caks]
\printunsrtglossary[type=samp]
\printunsrtglossary[type=casamp]
\printunsrtglossary[type=str]
\printunsrtglossary[type=stsgmp]
\printunsrtglossary[type=itergp]

\section{DERIVATION OF THE ALGORITHM}
\label{sec:kf-rts}
\begin{figure*}
  \centering
  \begin{tikzpicture}
  \def\dx{2.0}

  \filldraw[TUgray, semitransparent] (-3.75*\dx,-2) rectangle (3.5*\dx,-1);
  \filldraw[TUgray, semitransparent] (-3.75*\dx,-0.5) rectangle (3.5*\dx,0.5);
  \filldraw[TUgray, semitransparent] (-3.75*\dx,1.5) rectangle (3.5*\dx,2.5);

  \node[anchor=south west,text=TUdark] at (-3.75*\dx,-2) {Observations};
  \node[anchor=south west,text=TUdark] at (-3.75*\dx,-0.5) {Dynamics};
  \node[anchor=south west,text=TUdark] at (-3.75*\dx,1.5) {Computation};

  \node[var] (xm) at (-2*\dx,0) {$\gmp_{\idxdtime-1}$};
  \node[var] (xt) at (0,0) {$\gmp_\idxdtime$};
  \node[var] (xp) at (2*\dx,0) {$\gmp_{\idxdtime+1}$};

  \node[obs,TUdark,text=white] (ym) at (-2*\dx,-1.5) {$\cakfprojobsrv_{\idxdtime-1}$};
  \node[obs,TUdark,text=white] (yt) at (0,-1.5) {$\cakfprojobsrv_\idxdtime$};
  \node[obs,TUdark,text=white] (yp) at (2*\dx,-1.5) {$\cakfprojobsrv_{\idxdtime+1}$};

  \node[text=TUdark,anchor=west] at (ym.east) {$\in \R^{\cakfprojobsdim_{\idxdtime-1}}$};
  \node[text=TUdark,anchor=west] at (yt.east) {$\in \R^{\cakfprojobsdim_{\idxdtime}}$};
  \node[text=TUdark,anchor=west] at (yp.east) {$\in \R^{\cakfprojobsdim_{\idxdtime+1}}$};

  \node[anchor=south east] at (3.5*\dx,-0.5) {$\R^\gmpdim$};
  \node[anchor=south east] at (3.5*\dx,1.5) {$\R^\gmpdim$};

  \node[var, text=TUdark, draw=TUdark] (xmm) at (-2*\dx-0.3*\dx,2) {$\gmppstate{\idxdtime-1}$};
  \node[var, text=TUdark, draw=TUdark] (xmt) at (0-0.3*\dx,2) {$\gmppstate{\idxdtime}$};
  \node[var, text=TUdark, draw=TUdark] (xmp) at (2*\dx-0.3*\dx,2) {$\gmppstate{\idxdtime+1}$};

  \node[var] (xtm) at (-2*\dx+0.3*\dx,2) {$\gmptstate{\idxdtime-1}$};
  \node[var] (xtt) at (0.3*\dx,2) {$\gmptstate{\idxdtime}$};
  \node[var] (xtp) at (2*\dx+0.3*\dx,2) {$\gmptstate{\idxdtime+1}$};

  \draw[edge,->] (xmm) -- (xm);
  \draw[edge,->] (xmt) -- (xt);
  \draw[edge,->] (xmp) -- (xp);

  \draw[edge,->] (xm) -- (ym);
  \draw[edge,->] (xt) -- (yt);
  \draw[edge,->] (xp) -- (yp);

  \draw[edge,->] (xm) -- (xtm);
  \draw[edge,->] (xt) -- (xtt);
  \draw[edge,->] (xp) -- (xtp);

  \draw[edge,->] (xtm) -- (xmt);
  \draw[edge,->] (xtt) -- (xmp);

  \draw[bend right,->, TUred,dashed] (xmm) edge node[pos=0.3,below,anchor=north east]{\footnotesize $(\cakfpmean_{\idxdtime-1},\cakfpdd_{\idxdtime-1})$} (xm);
  \draw[bend right,->, TUred,dashed] (xmt) edge node[pos=0.3,below,anchor=north east]{\footnotesize $(\cakfpmean_t,\cakfpdd_t)$} (xt);
  \draw[bend right,->, TUred,dashed] (xmp) edge node[pos=0.3,below,anchor=north east]{\footnotesize $(\cakfpmean_{\idxdtime+1},\cakfpdd_{\idxdtime+1})$} (xp);

  \draw[bend right,->, TUgreen,dashed] (xm) edge node[pos=0.6,above,anchor=south west]{\footnotesize $(\cakfmean_{\idxdtime-1},\cakfdd_{\idxdtime-1})$} (xtm);
  \draw[bend right,->, TUgreen,dashed] (xt) edge node[pos=0.6,above,anchor=south west]{\footnotesize $(\cakfmean_\idxdtime,\cakfdd_\idxdtime)$} (xtt);
  \draw[bend right,->, TUgreen,dashed] (xp) edge node[pos=0.6,above,anchor=south west]{\footnotesize $(\cakfmean_{\idxdtime+1},\cakfdd_{\idxdtime+1})$} (xtp);

  \draw[bend left,->, TUblue,dashed] (xtm) edge node[midway,above,anchor=south]{\footnotesize $(\cakfmean_{\idxdtime-1},\cakftdd_{\idxdtime-1})$} (xmt);
  \draw[bend left,->, TUblue,dashed] (xtt) edge node[midway,above,anchor=south]{\footnotesize $(\cakfmean_\idxdtime,\cakftdd_\idxdtime)$} (xmp);

  \draw[->,TUviolet, dashed] (xp) edge node[midway,below,anchor=north]{\footnotesize $(\caksw_{\idxdtime+1},\caksW_{\idxdtime+1})$} (xt);
  \draw[->,TUviolet, dashed] (xt) edge node[midway,below,anchor=north]{\footnotesize $(\caksw_{\idxdtime},\caksW_{\idxdtime})$} (xm);

  \draw[edge,path fading=west,->] ([xshift=5pt, yshift=5pt]-2*\dx-0.3*\dx,2) (xmm)+(-\dx,0) -- (xmm);
  \draw[edge,path fading=east,->] ([xshift=5pt, yshift=5pt]+2*\dx-0.3*\dx,2) (xtp) -- +(1*\dx,0);

\end{tikzpicture}
  \caption{%
    Probabilistic graphical model for the computation-aware Kalman filter and RTS smoother.
    Solid arrows and circles define the joint generative model (i.e., the posterior computed by filter and smoother).
    Dashed arrows visualize the information flow between nodes, with the corresponding ``messages'' in parentheses.
  }
  \label{fig:pgm}
\end{figure*}

\begin{definition}[Linear-Gaussian State-Space Model]
  \label{def:lgssm}
  A \emph{linear-Gaussian state-space model} (LGSSM) is a pair $(\set{\gls{gmp:state}}_{\idxdtime = 0}^{\gls{gmp:len}}, \set{\gls{obs:rv}}_{\idxdtime = 1}^\gmplen)$ of discrete-time stochastic processes defined by
  \begin{align*}
    \gmp_\idxdtime     & \defeq \glslink{gmp:A}{\gmpA_{\idxdtime - 1}} \gmp_{\idxdtime - 1} + \glslink{gmp:b}{\gmpb_{\idxdtime - 1}} + \glslink{gmp:noise}{\gmpnoise_{\idxdtime - 1}} \in \R^{\gls{gmp:dim}},\qquad \\
    \obsrv_{\idxdtime} & \defeq \gls{obs:H} \gmp_\idxdtime + \vc_\idxdtime + \gls{obs:noise} \in \R^{\gls{obs:dim}},\qquad
  \end{align*}
  where
  \begin{align*}
    \gmp_0                    & \sim \gaussian{\gmpmean_0}{\gmpcov_0}                                         \\
    \gmpnoise_{\idxdtime - 1} & \sim \gaussian{\vec{0}}{\glslink{gmp:noisecov}{\gmpnoisecov_{\idxdtime - 1}}} \\
    \obsnoise_\idxdtime       & \sim \gaussian{\vec{0}}{\gls{obs:noisecov}}
  \end{align*}
  are pairwise independent.
\end{definition}

\subsection{Filtering}
\begin{theorem}[Kalman Filter]
  \label{thm:kalman-filter}
  Let $(\gmp, \obsrv)$ be the LGSSM from \cref{def:lgssm}.
  Fix $\set{\gls{obs:vec}}_{\idxdtime = 1}^\gmplen$ with $\obs_\idxdtime \in \R^{\obsdim_\idxdtime}$.
  Then
  \begin{equation*}
    \condrv{\gmp_\idxdtime \given \obsrv_{1:\idxdtime - 1} = \obs_{1:\idxdtime - 1}}
    \sim
    \gaussian{\kfpmean_\idxdtime}{\kfpcov_\idxdtime},
  \end{equation*}
  where
  \begin{align*}
    \gls{kf:pmean} & \defeq \gmpA_{\idxdtime - 1} \kfmean_{\idxdtime - 1} + \gmpb_{\idxdtime - 1},                               \\
    \gls{kf:pcov}  & \defeq \gmpA_{\idxdtime - 1} \kfcov_{\idxdtime - 1} \gmpA_{\idxdtime - 1}\T + \gmpnoisecov_{\idxdtime - 1},
  \end{align*}
  and
  \begin{equation*}
    \condrv{\gmp_\idxdtime \given \obsrv_{1:\idxdtime} = \obs_{1:\idxdtime}}
    \sim
    \gaussian{\kfmean_\idxdtime}{\kfcov_\idxdtime},
  \end{equation*}
  where $\kfmean_0 = \gmpmean_0$, $\kfcov_0 = \gmpcov_0$, and
  \begin{align*}
    \gls{kf:mean}     & \defeq \kfpmean_\idxdtime + \kfgain_\idxdtime \kfresidual_\idxdtime,                 \\
    \gls{kf:cov}      & \defeq \kfpcov_\idxdtime - \kfgain_\idxdtime \kfgram_\idxdtime \kfgain_\idxdtime\T,  \\
    \intertext{for $\idxdtime = 1, \dotsc, \gmplen$ with}
    \gls{kf:residual} & \defeq \obs_\idxdtime - \obsH_\idxdtime \kfpmean_\idxdtime - \vc_\idxdtime,          \\
    \gls{kf:gram}     & \defeq \obsH_\idxdtime \kfpcov_\idxdtime \obsH_\idxdtime\T + \obsnoisecov_\idxdtime, \\
    \gls{kf:gain}     & \defeq \kfpcov_\idxdtime \obsH_\idxdtime\T \kfgram_\idxdtime\inv.
  \end{align*}
\end{theorem}

\begin{proposition}[Downdate-Form Kalman Filter]
  \label{prop:kf-ddcov}
  The Kalman state covariances can equivalently be computed via
  \begin{align*}
    \kfpcov_\idxdtime & = \gmpcov_\idxdtime - \kfpdd_\idxdtime (\kfpdd_\idxdtime)\T, \\
    \kfcov_\idxdtime  & = \gmpcov_\idxdtime - \kfdd_\idxdtime \kfdd_\idxdtime\T,
  \end{align*}
  where $\kfdd_0 \defeq \pvec{\ } \in \R^{\gmpdim \times 0}$ and
  \begin{align*}
    \gls{kf:pdd}
     & \defeq \gmpA_{\idxdtime - 1} \kfdd_{\idxdtime - 1}, \\
    \gls{kf:dd}
     & \defeq
    \begin{pmatrix}
      \kfpdd_\idxdtime & \kfpcov_\idxdtime \kfW_\idxdtime
    \end{pmatrix}
  \end{align*}
  with $\gls{kf:W} \defeq \obsH_\idxdtime\T \kfgramlsqrt_\idxdtime$, and $\gls{kf:gramlsqrt} \kfgramlsqrt_\idxdtime\T = \kfgram_\idxdtime\inv$ for $k = 1, \dotsc, \gmplen$.
\end{proposition}
\begin{proof}
  For $\idxdtime = 0$, we find that
  \begin{equation*}
    \kfcov_0
    = \gmpcov_0
    = \gmpcov_0 - \mat{0}_{\gmpdim \times \gmpdim}
    = \gmpcov_0 - \kfdd_0 \kfdd_0\T.
  \end{equation*}
  Now let $1 \le k \le \gmplen$ and assume that the statement holds for $k - 1$.
  Then
  \begin{align*}
    \kfpcov_\idxdtime
     & = \gmpA_{\idxdtime - 1} \kfcov_{\idxdtime - 1} \gmpA_{\idxdtime - 1}\T + \gmpnoisecov_{\idxdtime - 1}                                                                                                \\
     & = \gmpA_{\idxdtime - 1} \gmpcov_{\idxdtime - 1} \gmpA_{\idxdtime - 1}\T + \gmpnoisecov_{\idxdtime - 1} - \gmpA_{\idxdtime - 1} \kfdd_{\idxdtime - 1} \kfdd_{\idxdtime - 1}\T \gmpA_{\idxdtime - 1}\T \\
     & = \gmpcov_\idxdtime - \gmpA_{\idxdtime - 1} \kfdd_{\idxdtime - 1} (\gmpA_{\idxdtime - 1} \kfdd_{\idxdtime - 1})\T                                                                                    \\
     & = \gmpcov_\idxdtime - \kfpdd_\idxdtime (\kfpdd_\idxdtime)\T,
  \end{align*}
  and
  \begin{align*}
    \kfcov_\idxdtime
     & = \kfpcov_\idxdtime - \kfpcov_\idxdtime \obsH_\idxdtime\T \kfgram_\idxdtime\inv \obsH_\idxdtime \kfpcov_\idxdtime                                                                   \\
     & = \gmpcov_\idxdtime - \kfpdd_\idxdtime (\kfpdd_\idxdtime)\T - \kfpcov_\idxdtime \obsH_\idxdtime\T \kfgramlsqrt_\idxdtime \kfgramlsqrt_\idxdtime\T \obsH_\idxdtime \kfpcov_\idxdtime \\
     & = \gmpcov_\idxdtime -
    \begin{pmatrix}
      \kfpdd_\idxdtime & \kfpcov_\idxdtime \obsH_\idxdtime\T \kfgramlsqrt_\idxdtime
    \end{pmatrix}
    \begin{pmatrix}
      \kfpdd_\idxdtime & \kfpcov_\idxdtime \obsH_\idxdtime\T \kfgramlsqrt_\idxdtime
    \end{pmatrix}
    \T                                                                                                                                                                                     \\
     & = \gmpcov_\idxdtime - \kfdd_\idxdtime \kfdd_\idxdtime\T.
  \end{align*}
\end{proof}

\subsection{Smoothing}
\begin{theorem}[RTS Smoother]
  \label{thm:rts-smoother}
  Let $(\gmp, \obsrv)$ be the LGSSM from \cref{def:lgssm}.
  Then
  \begin{equation*}
    \condrv{\gmp_\idxdtime \given \obsrv_{1:\gmplen} = \obs_{1:\gmplen}}
    \sim
    \gaussian{\ksmean_\idxdtime}{\kscov_\idxdtime},
  \end{equation*}
  where $\ksmean_\gmplen = \kfmean_\gmplen$, $\kscov_\gmplen = \kfcov_\gmplen$, and
  \begin{align*}
    \gls{ks:mean} & \defeq \kfmean_\idxdtime + \ksgain_\idxdtime (\ksmean_{\idxdtime + 1} - \kfpmean_{\idxdtime + 1})                    \\
    \gls{ks:cov}  & \defeq \kfcov_\idxdtime + \ksgain_\idxdtime (\kscov_{\idxdtime + 1} - \kfpcov_{\idxdtime + 1}) (\ksgain_\idxdtime)\T
  \end{align*}
  for $\idxdtime = 1, \dotsc, \gmplen - 1$ with $\gls{ks:gain} \defeq \kfcov_\idxdtime \gmpA_\idxdtime\T (\kfpcov_{\idxdtime + 1})\inv$.
\end{theorem}

\begin{proposition}[Inverse-Free RTS Smoother]
  \label{prop:inverse-free-smoother}
  The RTS smoother moments can be equivalently computed by the recursion
  \begin{align*}
    \ksmean_\idxdtime & = \kfpmean_\idxdtime + \kfpcov_\idxdtime \ksw_\idxdtime                                         \\
    \kscov_\idxdtime  & = \kfpcov_\idxdtime - \kfpcov_\idxdtime \ksW_\idxdtime \ps{\kfpcov_\idxdtime \ksW_\idxdtime}\T,
  \end{align*}
  where $\ksw_\gmplen = \obsH_\gmplen\T \kfgram_\gmplen\inv \kfresidual_\gmplen$, $\ksW_\gmplen \ps{\ksW_\gmplen}\T = \obsH_\gmplen\T \kfgram_\gmplen\inv \obsH_\gmplen$, and
  \begin{align}
    \gls{ks:w}
     & = \obsH_\idxdtime\T \kfgram_\idxdtime\inv \kfresidual_\idxdtime + \ps{\mI - \obsH_\idxdtime\T \kfgain_\idxdtime\T} \gmpA_\idxdtime\T \ksw_{\idxdtime + 1}                                                                                                   \nonumber                                                                                \\
     & = \obsH_\idxdtime\T \kfgram_\idxdtime\inv \kfresidual_\idxdtime + \ps{\mI - \obsH_\idxdtime\T \kfgram_\idxdtime\inv \obsH_\idxdtime \kfpcov_\idxdtime} \gmpA_\idxdtime\T \ksw_{\idxdtime + 1}                                                                               \label{eq:smooth_w}                                                      \\
    \gls{ks:W} \ps{\ksW_\idxdtime}\T
     & = \obsH_\idxdtime\T \kfgram_\idxdtime\inv \obsH_\idxdtime + \ps{\mI - \obsH_\idxdtime\T \kfgain_\idxdtime\T} \gmpA_\idxdtime\T \ksW_{\idxdtime + 1} \ps[\big]{\ps{\mI - \obsH_\idxdtime\T \kfgain_\idxdtime\T} \gmpA_\idxdtime\T \ksW_{\idxdtime + 1}}\T                                    \nonumber                                                \\
     & = \obsH_\idxdtime\T \kfgram_\idxdtime\inv \obsH_\idxdtime + \ps{\mI - \obsH_\idxdtime\T \kfgram_\idxdtime\inv \obsH_\idxdtime \kfpcov_\idxdtime} \gmpA_\idxdtime\T \ksW_{\idxdtime + 1} \ps[\big]{\ps{\mI - \obsH_\idxdtime\T \kfgram_\idxdtime\inv \obsH_\idxdtime \kfpcov_\idxdtime} \gmpA_\idxdtime\T \ksW_{\idxdtime + 1}}\T \label{eq:smooth_W}
  \end{align}
  for $\idxdtime = 1, \dotsc, \gmplen - 1$.
  Moreover,
  \begin{align}
    \ksmean_\idxdtime & = \kfmean_\idxdtime + \kfcov_\idxdtime \gmpA_\idxdtime\T \ksw_{\idxdtime + 1},                                                              \\
    \kscov_\idxdtime  & = \kfcov_\idxdtime - \kfcov_\idxdtime \gmpA_\idxdtime\T \ksW_{\idxdtime + 1} \ps{\kfcov_\idxdtime \gmpA_\idxdtime\T \ksW_{\idxdtime + 1}}\T
  \end{align}
  for $\idxdtime = 1, \dotsc, \gmplen - 1$.
\end{proposition}
\begin{proof}
  For $\idxdtime = \gmplen$, we have
  \begin{equation*}
    \ksmean_\gmplen
    = \kfmean_\gmplen
    = \kfpmean_\gmplen + \kfgain_\gmplen \kfresidual_\gmplen
    = \kfpmean_\gmplen + \kfpcov_\gmplen \underbracket[0.1ex]{\obsH_\gmplen\T \kfgram_\gmplen\inv \kfresidual_\gmplen}_{= \ksw_\gmplen}
  \end{equation*}
  and
  \begin{equation*}
    \kscov_\gmplen
    = \kfcov_\gmplen
    = \kfpcov_\gmplen - \kfgain_\gmplen \kfgram_\gmplen \kfgain_\gmplen\T
    = \kfpcov_\gmplen - \kfpcov_\gmplen \underbracket[0.1ex]{\obsH_\gmplen\T \kfgram_\gmplen\inv \kfgram_\gmplen \kfgram_\gmplen\inv \obsH_\gmplen}_{= \ksW_\gmplen \ps{\ksW_\gmplen}\T} \kfpcov_\gmplen.
  \end{equation*}
  Now let $1 \le \idxdtime < \gmplen$ and assume that
  \begin{alignat*}{5}
    \ksmean_{\idxdtime + 1}
     & = \kfpmean_{\idxdtime + 1} + \kfpcov_{\idxdtime + 1} \ksw_{\idxdtime + 1}
     & \qquad \Leftrightarrow \qquad                                                                                                &
     & \ksmean_{\idxdtime + 1} - \kfpmean_{\idxdtime + 1}
     & = \kfpcov_{\idxdtime + 1} \ksw_{\idxdtime + 1},                                                                                \\
    \kscov_{\idxdtime + 1}
     & = \kfpcov_{\idxdtime + 1} - \kfpcov_{\idxdtime + 1} \ksW_{\idxdtime + 1} \ps{\kfpcov_{\idxdtime + 1} \ksW_{\idxdtime + 1}}\T
     & \qquad \Leftrightarrow \qquad                                                                                                &
     & \kscov_{\idxdtime + 1} - \kfpcov_{\idxdtime + 1}
     & = - \kfpcov_{\idxdtime + 1} \ksW_{\idxdtime + 1} \ps{\kfpcov_{\idxdtime + 1} \ksW_{\idxdtime + 1}}\T.
  \end{alignat*}
  It follows that
  \begin{align*}
    \ksmean_\idxdtime
     & = \kfmean_\idxdtime + \ksgain_\idxdtime (\ksmean_{\idxdtime + 1} - \kfpmean_{\idxdtime + 1})                                                                                                                \\
     & = \kfmean_\idxdtime + \ksgain_\idxdtime \kfpcov_{\idxdtime + 1} \ksw_{\idxdtime + 1}                                                                                                                        \\
     & = \kfmean_\idxdtime + \kfcov_\idxdtime \gmpA_\idxdtime\T \ksw_{\idxdtime + 1}                                                                                                                               \\
     & = \kfpmean_\idxdtime + \kfgain_\idxdtime \kfresidual_\idxdtime + \ps{\kfpcov_\idxdtime - \kfgain_\idxdtime \kfgram_\idxdtime \kfgain_\idxdtime\T} \gmpA_\idxdtime\T \ksw_{\idxdtime + 1}                    \\
     & = \kfpmean_\idxdtime + \kfpcov_\idxdtime \ps[\Big]{\obsH_\idxdtime\T \kfgram_\idxdtime\inv \kfresidual_\idxdtime + \ps{\mI - \obsH_\idxdtime\T \kfgain_\idxdtime\T} \gmpA_\idxdtime\T \ksw_{\idxdtime + 1}} \\
     & = \kfpmean_\idxdtime + \kfpcov_\idxdtime \ksw_\idxdtime
  \end{align*}
  and
  \begin{align*}
    \kscov_\idxdtime
     & = \kfcov_\idxdtime + \ksgain_\idxdtime (\kscov_{\idxdtime + 1} - \kfpcov_{\idxdtime + 1}) (\ksgain_\idxdtime)\T                                                                                                                                                                                                                              \\
     & = \kfcov_\idxdtime - \kfcov_\idxdtime \gmpA_\idxdtime\T (\kfpcov_{\idxdtime + 1})\inv \kfpcov_{\idxdtime + 1} \ksW_{\idxdtime + 1} \ps{\kfpcov_{\idxdtime + 1} \ksW_{\idxdtime + 1}}\T (\kfpcov_{\idxdtime + 1})\inv \gmpA_\idxdtime \kfcov_\idxdtime                                                                                        \\
     & = \kfcov_\idxdtime - \kfcov_\idxdtime \gmpA_\idxdtime\T \ksW_{\idxdtime + 1} \ps[\big]{\kfcov_\idxdtime \gmpA_\idxdtime\T \ksW_{\idxdtime + 1}}\T                                                                                                                                                                                            \\
     & = \kfpcov_\idxdtime - \kfgain_\idxdtime \kfgram_\idxdtime \kfgain_\idxdtime\T - \ps{\kfpcov_\idxdtime - \kfgain_\idxdtime \kfgram_\idxdtime \kfgain_\idxdtime\T} \gmpA_\idxdtime\T \ksW_{\idxdtime + 1} \ps[\big]{\ps{\kfpcov_\idxdtime - \kfgain_\idxdtime \kfgram_\idxdtime \kfgain_\idxdtime\T} \gmpA_\idxdtime\T \ksW_{\idxdtime + 1}}\T \\
     & = \kfpcov_\idxdtime - \kfpcov_\idxdtime \ps[\Big]{\obsH_\idxdtime\T \kfgram_\idxdtime\inv \obsH_\idxdtime + \ps{\mI - \obsH_\idxdtime\T \kfgain_\idxdtime\T} \gmpA_\idxdtime\T \ksW_{\idxdtime + 1} \ps[\big]{\ps{\mI - \obsH_\idxdtime\T \kfgain_\idxdtime\T} \gmpA_\idxdtime\T \ksW_{\idxdtime + 1}}\T} \kfpcov_\idxdtime                  \\
     & = \kfpcov_\idxdtime - \kfpcov_\idxdtime \mW_\idxdtime \kfpcov_\idxdtime.
  \end{align*}
\end{proof}

\subsection{Sampling via Matheron's Rule}
\label{sec:matheron-sampling}
The naive approach to sampling from a multivariate normal distribution (e.g., by Cholesky factorization or eigendecomposition) has cubic cost and requires storing the covariance matrix in memory, which is not possible for large state-space dimension.
We alleviate this by applying Matheron's rule \citep{Matheron1963PrinciplesGeostatistics,Wilson2020EfficientlySampling} to the (computation-aware) Kalman filter and RTS smoother recursions, making it possible to sample the filtering and smoothing posteriors by transforming samples from the prior.
\begin{lemma}[Matheron's Rule]
  \label{lem:matherons-rule}
  Let $\rvx \sim \gaussian{\vmu}{\mSigma}$, $\mA \in \R^{N_\mA \times D}$, $\mB \in \R^{N_\mB \times D}$, and $\vbeta \in \range{\mB \mSigma}$.
  Define
  \begin{equation*}
    \bm{\mathcal{M}}_{\mSigma, \mA, \mB, \vbeta} \colon \R^D \to \R^{N_\mA},
    \vxi \mapsto \expectation[\tilde{\rvx}][\gaussian{\vxi}{\mSigma}]{\condrv{\mA \tilde{\rvx} \given \mB \tilde{\rvx} = \vbeta}}.
  \end{equation*}
  Then $(\condrv{\mA \rvx \given \mB \rvx = \vbeta}) \stackrel{d}{=} \bm{\mathcal{M}}_{\mSigma, \mA, \mB, \vbeta}(\rvx)$.
\end{lemma}
\begin{proof}
  Let $\rva \defeq \mA \rvx$ and $\rvb \defeq \mB \rvx$.
  Then
  \begin{equation*}
    \begin{pmatrix}
      \rva \\
      \rvb
    \end{pmatrix}
    \sim
    \gaussian{
      \begin{pmatrix}
        \mA \vmu \\
        \mB \vmu
      \end{pmatrix}
    }{
      \begin{pmatrix}
        \mA \mSigma \mA\T & \mA \mSigma \mB\T \\
        \mB \mSigma \mA\T & \mB \mSigma \mB\T
      \end{pmatrix}
    }
  \end{equation*}
  and
  \begin{equation*}
    \bm{\mathcal{M}}_{\mSigma, \mA, \mB, \vbeta}(\vxi)
    = \expectation[\tilde{\rvx}][\gaussian{\vxi}{\mSigma}]{\condrv{\mA \tilde{\rvx} \given \mB \tilde{\rvx} = \vbeta}}
    = \mA \vxi + \mA \mSigma \mB\T (\mB \mSigma \mB\T)\inv (\vbeta - \mB \vxi).
  \end{equation*}
  Hence,
  \begin{equation*}
    \bm{\mathcal{M}}_{\mSigma, \mA, \mB, \vbeta}(\rvx)
    = \underbrace{\mA \rvx}_{= \rva} + \underbrace{\mA \mSigma \mB\T}_{= \covariance{\rva}{\rvb}} (\underbrace{\mB \mSigma \mB\T}_{= \covariance{\rvb}{\rvb}})\inv (\vbeta - \underbrace{\mB \rvx}_{= \rvb})
  \end{equation*}
  and thus the statement follows from \citet[Theorem 1]{Wilson2020EfficientlySampling}.
\end{proof}
To proceed, we assume that it is feasible to obtain an (approximate) sample from the initial state $\gmp_0 \sim \gaussian{\gmpmean_{0}}{\gmpcov_{0}}$, as well as (approximate) samples from the dynamics and observational noise $\gmpnoise_{\idxdtime-1} \sim \gaussian{\vec{0}}{\gmpnoisecov_{\idxdtime-1}}$, $\obsnoise_\idxdtime \sim \gaussian{\vec{0}}{\obsnoisecov_\idxdtime}$, $\idxdtime=1,\dots, \gmplen$.
This assumption is reasonable because the covariance matrices $\gmpcov_{0}, \gmpnoisecov_{\idxdtime-1}, \obsnoisecov_\idxdtime$ are often simple or highly structured; for example it is common for $\obsnoisecov_\idxdtime$ to be diagonal.
Moreover, for discretized spatiotemporal Gauss--Markov processes one can use function space approximations like random Fourier features (RFF) \citep{Rahimi2007RandomFeatures} to obtain approximate samples from $\gmp_0$ and $\gmpnoise_{\idxdtime-1}$ \citep[see also][]{Wilson2020EfficientlySampling}.
Finally, Krylov methods can be used to approximate matrix square roots of the covariances in a matrix-free fashion \citep[see e.g.,][]{Pleiss2020MatrixSquareRoots}.
With these samples, \cref{prop:matheron_filter_smooth} shows how Matheron sampling can be implemented for the standard Kalman filter and RTS smoother, while \cref{prop:matheron_smooth_inverse_free} gives an equivalent form of Matheron sampling for the smoother that circumvents inversion of state covariance matrices.

Each of these approaches can be applied to the modified state-space model used in the CAKF and the CAKS at low cost, again recycling computed values from the filtering pass in \cref{alg:mfkf,alg:projected_update}.
The resulting algorithm for sampling from the computation-aware posterior process $\set{\condrv{\gmp_\idxdtime \given \cakfprojobsrv_{1:\gmplen} = \cakfprojobs_{1:\gmplen}}}_{\idxdtime = 0}^{\gmplen}$ is detailed in \cref{alg:cakf-caks-sampler}.
If it is stopped early before \cref{algline:cakf-sample}, then it can also be used to compute samples from the CAKF states $\set{\condrv{\gmp_\idxdtime \given \cakfprojobsrv_{1:\idxdtime} = \cakfprojobs_{1:\idxdtime}}}_{\idxdtime = 0}^{\gmplen}$.
Also note that \cref{alg:cakf-caks-sampler} allows us to sample from the full Bayesian posterior without running the CAKS, since all quantities used above have already been computed by the filter.

\begin{theorem}
  \label{prop:matheron_filter_smooth}
  Let $(\gmp, \obsrv)$ be the LGSSM from \cref{def:lgssm}.
  Fix $\obs_\idxdtime \in \R^{\obsdim_\idxdtime}$ for $\idxdtime = 1, \dotsc, \gmplen$ and define
  \begin{align*}
    \gls{kf:pstate} & \defeq \gmpA_{\idxdtime - 1} \kfstate_{\idxdtime - 1} + \gmpb_{\idxdtime - 1} + \gmpnoise_{\idxdtime - 1} \\
    \gls{kf:pobsrv} & \defeq \obsH_\idxdtime \kfpstate_\idxdtime + \vc_\idxdtime + \obsnoise_\idxdtime                          \\
    \gls{kf:state}  & \defeq \kfpstate_\idxdtime + \kfgain_\idxdtime (\obs_\idxdtime - \kfpobsrv_\idxdtime)
  \end{align*}
  for $\idxdtime = 1, \dotsc, \gmplen$, where $\kfstate_0 \defeq \gmp_0$, as well as
  \begin{equation*}
    \gls{ks:state} \defeq \kfstate_\idxdtime + \ksgain_\idxdtime (\ksstate_{\idxdtime + 1} - \kfpstate_{\idxdtime + 1})
  \end{equation*}
  for $\idxdtime = \gmplen - 1, \dotsc, 0$, where $\ksstate_\gmplen \defeq \kfstate_\gmplen$.
  Then
  \begin{equation*}
    (\ksstate_1, \dotsc, \ksstate_\gmplen)
    \stackrel{d}{=}
    \ps[\Big]{\condrv{(\gmp_1, \dotsc, \gmp_\gmplen) \given \obsrv_1 = \obs_1, \dotsc, \obsrv_\gmplen = \obs_\gmplen}}.
  \end{equation*}
\end{theorem}
\begin{proof}
  Let $\tilde{\gmpnoise}_\idxdtime \defeq \gmpnoise_\idxdtime + \gmpb_\idxdtime$ and $\tilde{\obsnoise}_\idxdtime \defeq \obsnoise_\idxdtime + \vc_\idxdtime$.
  Then $\rvx \defeq (\gmp_0, \tilde{\gmpnoise}_0, \dotsc, \tilde{\gmpnoise}_{\gmplen - 1}, \tilde{\obsnoise}_1, \dotsc, \tilde{\obsnoise}_\gmplen)$ is jointly Gaussian with pairwise independent components and covariance matrix $\tilde{\mSigma}$.
  \begin{align*}
    \gmp_{\idxdtime + 1} & = \gmpA_\idxdtime \gmp_\idxdtime + \tilde{\gmpnoise}_\idxdtime, \qquad \text{and} \\
    \obsrv_\idxdtime     & = \obsH_\idxdtime \gmp_\idxdtime + \tilde{\obsnoise}_\idxdtime
  \end{align*}
  pointwise.
  Let $\bm{\mathcal{U}}$ and $\bm{\mathcal{Y}}$ be the linear operators defined by
  \begin{align*}
    \bm{\mathcal{U}}(\gmp_0, \tilde{\gmpnoise}_0, \dotsc, \tilde{\gmpnoise}_{\gmplen - 1}, \tilde{\obsnoise}_1, \dotsc, \tilde{\obsnoise}_\gmplen) & = (\gmp_1, \dotsc, \gmp_\gmplen), \qquad \text{and} \\
    \bm{\mathcal{Y}}(\gmp_0, \tilde{\gmpnoise}_0, \dotsc, \tilde{\gmpnoise}_{\gmplen - 1}, \tilde{\obsnoise}_1, \dotsc, \tilde{\obsnoise}_\gmplen) & = (\obsrv_1, \dotsc, \obsrv_\gmplen).
  \end{align*}
  By \cref{thm:kalman-filter,thm:rts-smoother}, the operator $\bm{\mathcal{M}}_{\tilde{\mSigma}, \bm{\mathcal{U}}, \bm{\mathcal{Y}}, \vy_{1:\gmplen}}$ from \cref{lem:matherons-rule} corresponding to this model is given by
  \begin{equation*}
    \bm{\mathcal{M}}_{\tilde{\mSigma}, \bm{\mathcal{U}}, \bm{\mathcal{Y}}, \vy_{1:\gmplen}}(\vxi)
    = (\ksmean_1(\vxi), \dotsc, \ksmean_\gmplen(\vxi)),
  \end{equation*}
  where, for $\vxi = (\vxi^\gmp_0, \vxi^{\tilde{\gmpnoise}}_0, \dotsc, \vxi^{\tilde{\gmpnoise}}_{\gmplen - 1}, \vxi^{\tilde{\obsnoise}}_1, \dotsc, \vxi^{\tilde{\obsnoise}}_\gmplen)$,
  \begin{align*}
    \kfmean_0(\vxi)          & \defeq \vxi^\gmp_0,                                                                                                                                                      \\
    \kfpmean_\idxdtime(\vxi) & \defeq \gmpA_{\idxdtime - 1} \kfmean_{\idxdtime - 1}(\vxi) + \vxi^{\tilde{\gmpnoise}}_{\idxdtime - 1},                                                                   \\
    \kfmean_\idxdtime(\vxi)  & \defeq \kfpmean_\idxdtime(\vxi) + \kfgain_\idxdtime (\vy_\idxdtime - (\obsH_\idxdtime \kfpmean_\idxdtime(\vxi) + \vxi^{\tilde{\obsnoise}}_\idxdtime)), \qquad \text{and} \\
    \ksmean_\idxdtime(\vxi)  & \defeq \kfmean_\idxdtime(\vxi) + \ksgain_\idxdtime (\ksmean_{\idxdtime + 1}(\vxi) - \kfpmean_{\idxdtime + 1}(\vxi)),
  \end{align*}
  Note that $\kfmean_0(\rvx) = \gmp_0$ and hence
  \begin{align*}
    \kfpmean_\idxdtime(\rvx) & = \gmpA_{\idxdtime - 1} \underbrace{\kfmean_{\idxdtime - 1}(\rvx)}_{= \kfstate_\idxdtime} + \tilde{\gmpnoise}_{\idxdtime - 1} = \kfpstate_\idxdtime, \qquad \text{and}                       \\
    \kfmean_\idxdtime(\rvx)  & = \kfpstate_\idxdtime + \kfgain_\idxdtime (\obs_\idxdtime - \underbrace{(\obsH_\idxdtime \kfpmean_\idxdtime(\rvx) + \tilde{\obsrv}_\idxdtime)}_{= \kfpobsrv_\idxdtime}) = \kfstate_\idxdtime
  \end{align*}
  by induction on $\idxdtime$.
  Moreover, $\ksmean_\gmplen(\rvx) = \kfmean_\gmplen(\rvx) = \kfstate_\gmplen = \ksstate_\gmplen$ and thus
  \begin{equation*}
    \ksmean_\idxdtime(\rvx)
    = \underbrace{\kfmean_\idxdtime(\rvx)}_{= \kfstate_\idxdtime} + \ksgain_\idxdtime (\underbrace{\ksmean_{\idxdtime + 1}(\rvx)}_{= \ksstate_{\idxdtime + 1}} - \underbrace{\kfpmean_{\idxdtime + 1}(\rvx)}_{= \kfpstate_\idxdtime})
    = \ksstate_\idxdtime
  \end{equation*}
  by induction on $\idxdtime$.
  Finally, by \cref{lem:matherons-rule}, we arrive at
  \begin{align*}
    (\ksstate_1, \dotsc, \ksstate_\gmplen)
    = \bm{\mathcal{M}}_{\tilde{\mSigma}, \bm{\mathcal{U}}, \bm{\mathcal{Y}}, \vy_{1:\gmplen}}(\rvx)
    \stackrel{d}{=} \ps[\Big]{\condrv{\bm{\linop{U}}(\rvx) \given \bm{\linop{Y}}(\rvx) = \obs_{1:\gmplen}}}
    = \ps[\Big]{\condrv{(\gmp_1, \dotsc, \gmp_\gmplen) \given \obsrv_1 = \obs_1, \dotsc, \obsrv_\gmplen = \obs_\gmplen}}.
  \end{align*}
\end{proof}

\begin{proposition}[Inverse-Free Posterior Sampling] \label{prop:matheron_smooth_inverse_free}
  Samples from the smoothing posterior can be equivalently computed by means of the recursion
  \begin{equation*}
    \ksstate_\idxdtime = \kfpstate_\idxdtime + \kfpcov_\idxdtime \kswrv_\idxdtime,
  \end{equation*}
  where $\kswrv_\gmplen \defeq \obsH_\gmplen\T \kfgram_\gmplen\inv (\obs_\gmplen - \kfpobsrv_\gmplen)$, and
  \begin{align*}
    \gls{ks:wrv}
     & \defeq \obsH_\idxdtime \kfgram_\idxdtime\inv (\obs_\idxdtime - \kfpobsrv_\idxdtime) + \ps{\mI - \obsH_\idxdtime\T \kfgain_\idxdtime\T} \gmpA_\idxdtime\T \kswrv_{\idxdtime + 1}                              \\
     & = \obsH_\idxdtime \kfgram_\idxdtime\inv (\obs_\idxdtime - \kfpobsrv_\idxdtime) + \ps{\mI - \obsH_\idxdtime\T \kfgram_\idxdtime\inv \obsH_\idxdtime \kfpcov_\idxdtime} \gmpA_\idxdtime\T \rvw_{\idxdtime + 1}
  \end{align*}
  for $\idxdtime = 0, \dotsc, \gmplen - 1$.
  Moreover,
  \begin{equation*}
    \ksstate_\idxdtime = \kfstate_\idxdtime + \kfcov_\idxdtime \gmpA_\idxdtime\T \kswrv_{\idxdtime + 1}
  \end{equation*}
  pointwise for $\idxdtime = 0, \dotsc, \gmplen - 1$.
\end{proposition}
\begin{proof}
  \begin{equation*}
    \ksstate_\gmplen
    = \kfstate_\gmplen
    = \kfpstate_\gmplen + \kfgain_\gmplen (\obs_\gmplen - \kfpobsrv_\gmplen)
    = \kfpstate_\gmplen + \kfpcov_\gmplen \underbracket[0.1ex]{\obsH_\gmplen\T \kfgram_\gmplen\inv (\obs_\gmplen - \kfpobsrv_\gmplen)}_{= \kswrv_\gmplen}
  \end{equation*}
  Now assume that $\ksstate_{\idxdtime + 1} = \kfpstate_{\idxdtime + 1} + \kfpcov_{\idxdtime + 1} \kswrv_{\idxdtime + 1}$, which is equivalent to $\ksstate_{\idxdtime + 1} - \kfpstate_{\idxdtime + 1} = \kfpcov_{\idxdtime + 1} \kswrv_{\idxdtime + 1}$.
  Then
  \begin{align*}
    \ksstate_\idxdtime
     & = \kfstate_\idxdtime + \ksgain_\idxdtime (\ksstate_{\idxdtime + 1} - \kfpstate_{\idxdtime + 1})                                                                                                                                                                    \\
     & = \kfstate_\idxdtime + \kfcov_\idxdtime \gmpA_\idxdtime\T (\kfpcov_{\idxdtime + 1})\inv \kfpcov_{\idxdtime + 1} \kswrv_{\idxdtime + 1}                                                                                                                             \\
     & = \kfstate_\idxdtime + \kfcov_\idxdtime \gmpA_\idxdtime\T \kswrv_{\idxdtime + 1}                                                                                                                                                                                   \\
     & = \kfpstate_\idxdtime + \kfgain_\idxdtime (\obs_\idxdtime - \kfpobsrv_\idxdtime) + \kfcov_\idxdtime \gmpA_\idxdtime\T (\kfpcov_{\idxdtime + 1})\inv \kfpcov_{\idxdtime + 1} \kswrv_{\idxdtime + 1}                                                                 \\
     & = \kfpstate_\idxdtime + \kfgain_\idxdtime (\obs_\idxdtime - \kfpobsrv_\idxdtime) + (\kfpcov_\idxdtime - \kfgain_\idxdtime \kfgram_\idxdtime \kfgain_\idxdtime\T) \gmpA_\idxdtime\T \kswrv_{\idxdtime + 1}                                                          \\
     & = \kfpstate_\idxdtime + \kfpcov_\idxdtime \underbracket[0.1ex]{\ps[\Big]{\obsH_\idxdtime\T \kfgram_\idxdtime\inv (\obs_\idxdtime - \kfpobsrv_\idxdtime) + (\mI - \obsH_\idxdtime\T \kfgain_\idxdtime\T) \gmpA_\idxdtime\T \kswrv_\idxdtime}}_{= \kswrv_\idxdtime}.
  \end{align*}
\end{proof}

\begin{algorithm}
  \caption{CAKF/CAKS Sampler}
  \label{alg:cakf-caks-sampler}
  \begin{algorithmic}[1]
    \small
    \Function{Sample}{$\{ \cdots, \cakfdd_\idxdtime, \cakfprojgramlsqrt_\idxdtime, \cakfW_\idxdtime, \dotsc \}_{\idxdtime = 0}^{\gmplen}$}
      \State $\cakfstate_0 \sim \gaussian{\gmpmean_{0}}{\matfree{\gmpcov_{0}}}$
      \For{$\idxdtime = 1, \dotsc, \gmplen$}
        \State $\just{\cakfpstate_\idxdtime}{\gmpnoise_{\idxdtime-1}} \sim \gaussian{\vec{0}}{\matfree{\gmpnoisecov_{\idxdtime-1}}}$
        \State $\just{\cakfpstate_\idxdtime}{\cakfprojobsnoise_\idxdtime} \sim \gaussian{\vec{0}}{\cakfprojobsnoisecov_\idxdtime}$
        \State $\just{\cakfpstate_\idxdtime}{\gls{cakf:pstate}} \gets \matfree{\gmpA_{\idxdtime-1}}{\cakfstate_{\idxdtime-1}} + \gmpb_{\idxdtime-1} + \gmpnoise_{\idxdtime-1}$
        \State $\just{\cakfpstate_\idxdtime}{\gls{cakf:wrv}} \gets \cakfW_\idxdtime \ps{\cakfprojgramlsqrt_\idxdtime\T \ps{\cakfprojobs_\idxdtime - \cakfprojH_\idxdtime \cakfpstate_\idxdtime - \cakfprojobsnoise_\idxdtime}}$
        \State $\just{\cakfpstate_\idxdtime}{\gls{cakf:state}} \gets \cakfpstate_\idxdtime + \matfree{\cakfpcov_\idxdtime}{\cakfwrv_\idxdtime}$
      \EndFor
      \State $\cakswrv_n = \cakfwrv_{\gmplen}$ \label{algline:cakf-sample}
      \For{$\idxdtime = \gmplen - 1, \dotsc, 0$}
        \State $\gls{caks:wrv} \gets \cakfwrv_\idxdtime + \matfree{(\mI - \cakfW_\idxdtime \cakfW_\idxdtime\T \cakfpcov_\idxdtime) \gmpA_\idxdtime\T}{\cakswrv_{\idxdtime+1}}$
      \EndFor
      \State \textbf{return} \( \{ \gls{caks:state} = \cakfstate_\idxdtime + \matfree{\cakfcov_\idxdtime \gmpA_\idxdtime\T}{\cakswrv_\idxdtime} \}_{\idxdtime = 0}^{\gmplen} \)
    \EndFunction
  \end{algorithmic}
\end{algorithm}

\subsection{Temporal Interpolation}
\label{sec:temporal-interpolation}
In practice, we are often interested in interpolating a set of discrete-time measurements.
To this end, we need to exchange the discrete-time dynamics model in \cref{def:lgssm} by a continuous-time dynamics model.
\begin{definition}[Continuous-Discrete LGSSM]
  A \emph{continuous-discrete linear-Gaussian state-space model} (CD-LGSSM) is a pair $(\set{\gmp(t)}_{t \in \mathbb{T}}, \set{\obsrv_\idxdtime}_{\idxdtime = 1}^\gmplen)$ of a continuous-time Gauss--Markov process $\gmp$ with transition kernels
  \begin{equation*}
    \prob{\gmp(t) \given \gmp(s) = \gmpval(s)} = \gaussian{\gmpA(t, s) \gmpval(s) + \gmpb(t, s)}{\gmpnoisecov(t, s)}
  \end{equation*}
  and a discrete-time Gauss--Markov process $\set{\obsrv_\idxdtime}_{\idxdtime = 1}^\gmplen$ defined by
  \begin{equation*}
    \obsrv_\idxdtime = \obsH_\idxdtime \gmp(\ttrain{\idxdtime}) + \obsnoise_\idxdtime,
  \end{equation*}
  where $\obsnoise_\idxdtime \sim \gaussian{\vec{0}}{\obsnoisecov_\idxdtime}$ for $k = 1, \dotsc, \gmplen$ and $\gmp$ are pairwise independent.
\end{definition}
We write $\gmpmean(t) \defeq \expectation{\gmp(t)}$ and $\gmpcov(t) \defeq \covariance{\gmp(t)}{\gmp(t)}$ for the mean and covariance functions of the latent continuous-time Gauss--Markov process.

In practice, we want to be able to access the interpolant efficiently (i.e., without revisiting all training data) and on-demand, as we often do not know the query locations in advance.
It is well-known that this can be achieved by running the Kalman filter and RTS smoother on the discretized LGSSM corresponding to $\gmp_\idxdtime \defeq \gmp(\ttrain{\idxdtime})$ followed by an interpolation step on the filter and/or smoother states that only involves the neighboring training time points.
The above is formalized in \cref{cor:temporal-interpolation-kf-rts}.

\begin{corollary}[Temporal Interpolation in Kalman Filter and RTS Smoother]
  \label{cor:temporal-interpolation-kf-rts}
  Let $t \in \mathbb{T}$ and $0 \le k \le \gmplen$.
  \begin{enumerate}[label=(\roman*)]
    \item If $\ttrain{\idxdtime} \le t < \ttrain{\idxdtime + 1}$ with $1 \le k < \gmplen$ or $\ttrain{\idxdtime} \le t$ with $\idxdtime = \gmplen$, then
          \begin{equation*}
            \kfstate(t)
            \defeq \ps{\condrv{\gmp(t) \given \obsrv_{1:\idxdtime} = \obs_{1:\idxdtime}}}
            \sim \gaussian{\kfmean(t)}{\kfcov(t)}
          \end{equation*}
          with
          \begin{align*}
            \kfmean(t)
             & \defeq \gmpA(t, \ttrain{\idxdtime}) \kfmean_{\idxdtime} + \gmpb(t, \ttrain{\idxdtime}), \qquad \text{and}                                         \\
            \kfcov(t)
             & \defeq \gmpA(t, \ttrain{\idxdtime}) \kfcov_{\idxdtime} \gmpA(t, \ttrain{\idxdtime})\T + \gmpnoisecov(t, \ttrain{\idxdtime})                       \\
             & = \gmpcov(t) - \underbrace{\gmpA(t, \ttrain{\idxdtime}) \kfdd_{\idxdtime}}_{\rdefeq \kfdd(t)} (\gmpA(t, \ttrain{\idxdtime}) \kfdd_{\idxdtime})\T.
          \end{align*}
          For $t < \ttrain{1}$, we extend the definition by $\kfstate(t) \defeq \gmp(t)$, i.e., $\kfmean(t) \defeq \gmpmean(t)$, $\kfcov(t) \defeq \gmpcov(t)$, and $\kfdd(t) \in \R^{\gmpdim \times 0}$.
    \item If $\ttrain{\idxdtime} \le t < \ttrain{\idxdtime + 1}$ with $1 \le k < \gmplen$ or $t < \ttrain{\idxdtime + 1}$ with $\idxdtime = 0$, then
          \begin{equation*}
            \ksstate(t)
            \defeq \ps{\condrv{\gmp(t) \given \obsrv_{1:\gmplen} = \obs_{1:\gmplen}}}
            \sim \gaussian{\ksmean(t)}{\kscov(t)}
          \end{equation*}
          with
          \begin{align*}
            \ksmean(t)
             & \defeq \kfmean(t) + \ksgain(t) (\ksmean_{\idxdtime + 1} - \kfmean(t))                                                                              \\
             & = \kfmean(t) + \kfcov(t) \mA(\ttrain{\idxdtime + 1}, t)\T \ksw_{\idxdtime + 1}, \qquad \text{and}                                                  \\
            \kscov(t)
             & \defeq \kscov(t) + \ksgain(t) (\kscov_{\idxdtime + 1} - \kfcov(t)) \ksgain(t)\T                                                                    \\
             & = \kfcov(t) - \kfcov(t) \mA(t, \ttrain{\idxdtime + 1})\T \ksW_{\idxdtime + 1} (\kfcov(t) \mA(t, \ttrain{\idxdtime + 1})\T \ksW_{\idxdtime + 1})\T,
          \end{align*}
          where $\ksgain(t) \defeq \kfcov(t) \mA(t, \ttrain{\idxdtime + 1})\T (\kfpcov_{\idxdtime + 1})\inv$.
          Again, we extend the definition by $\ksstate(t) \defeq \kfstate(t)$ for $t > \ttrain{\gmplen}$, i.e., $\ksmean(t) \defeq \kfmean(t)$ and $\kscov(t) \defeq \kfcov(t)$.
  \end{enumerate}
\end{corollary}
\begin{proof}
  This follows from \cref{thm:kalman-filter,prop:kf-ddcov,thm:rts-smoother,prop:inverse-free-smoother}.
\end{proof}

\Cref{cor:temporal-interpolation-kf-rts} also shows that the efficient interpolation capabilities extend to the downdate-form versions of the Kalman filter and RTS smoother from \cref{prop:kf-ddcov,prop:inverse-free-smoother} that form the basis of the CAKF and CAKS.
Using this result, we can derive \cref{alg:cakf-interpolation,alg:caks-interpolation}.
An efficient version of \cref{alg:cakf-caks-sampler} that allows for sampling at intermediate points can be derived analogously.
\begin{figure}
  \centering
  \begin{minipage}[t]{0.46\textwidth}
    \input{algorithms/cakf-interpolation}
  \end{minipage}
  ~
  \begin{minipage}[t]{0.52\textwidth}
    \input{algorithms/caks-interpolation}
  \end{minipage}
\end{figure}

\subsection{Iterative Version of the CAKF Update Step}

\begin{algorithm}[t]
  \caption{CAKF Update Step (Iterative Version)} \label{alg:update_pls}
  \begin{algorithmic}[1]
    \Function{Update}{$\cakfpmean_\idxdtime, \cakfpdd_\idxdtime, \gmpcov_\idxdtime, \obsH_\idxdtime, \obsnoisecov_\idxdtime, \obs_\idxdtime$}
      \State $\matfree{\cakfpcov_\idxdtime} \gets \matfree{\gmpcov_\idxdtime - \cakfpdd_\idxdtime (\cakfpdd_\idxdtime)\T} \in \R^{\gmpdim \times \gmpdim}$
      \State $\matfree{\cakfgram_\idxdtime} \gets \matfree{\obsH_\idxdtime \cakfpcov_\idxdtime \obsH_\idxdtime\T + \obsnoisecov_\idxdtime} \in \R^{\ntraindata_\idxdtime \times \ntraindata_\idxdtime}$
      \State $\cakfrepw_\idxdtime\iidx[0] \gets \vec{0} \in \R^{\ntraindata_\idxdtime}$
      \State $\cakfrepwdd_\idxdtime\iidx[0] \gets (\quad) \in \R^{\ntraindata_\idxdtime \times 0}$
      \State $\cakfresidual_\idxdtime\iidx[0] \gets \obs_\idxdtime - \matfree{\obsH_\idxdtime}{\cakfpmean_\idxdtime}$
      \While{$\neg \Call{StoppingCriterion}{i, \cakfresidual_\idxdtime\iidx[i], \dotsc}$}
        \State $\cakfact_\idxdtime\iidx \gets \Call{Policy}{i, \cakfresidual_\idxdtime\iidx[i - 1], \dotsc}$
        \State $\cakfresidual_\idxdtime\iidx \gets \cakfresidual_\idxdtime\iidx[0] - \matfree{\cakfgram_\idxdtime}{\cakfrepw_\idxdtime\iidx[i - 1]}$
        \State $\alpha_\idxdtime\iidx \gets \inprod{\cakfact_\idxdtime\iidx}{\cakfresidual_\idxdtime\iidx}$
        \State $\vd_\idxdtime\iidx \gets \matfree{\ps{\mI - \cakfrepwdd_\idxdtime\iidx[i - 1] \ps{\cakfrepwdd_\idxdtime\iidx[i - 1]}\T \cakfgram_\idxdtime}}{\cakfact_\idxdtime\iidx}$ \label{alg:line:orthogonalise}
        \State $\eta_\idxdtime\iidx \gets \inprod{\cakfact_\idxdtime\iidx}{\matfree{\cakfgram_\idxdtime}{\vd_\idxdtime\iidx}}$
        \State $\cakfrepw_\idxdtime\iidx \gets \cakfrepw_\idxdtime\iidx[i - 1] + \frac{\alpha_\idxdtime\iidx}{\eta_\idxdtime\iidx} \vd_\idxdtime\iidx$
        \State \(
        \cakfrepwdd_\idxdtime\iidx
        \gets
        \begin{pmatrix}
          \cakfrepwdd_\idxdtime\iidx[i - 1] & \frac{1}{\sqrt{\eta_\idxdtime\iidx}} \vd\iidx
        \end{pmatrix}
        \in \R^{\ntraindata_\idxdtime \times i}
        \)
      \EndWhile
      \State $\gls{cakf:w} \gets \obsH_\idxdtime\T \cakfrepw_\idxdtime\iidx$ \label{line:calc_filter_w}
      \State $\gls{cakf:W} \gets \obsH_\idxdtime\T \cakfrepwdd_\idxdtime\iidx$ \label{line:calc_filter_W}
      \State $\cakfmean_\idxdtime \gets \cakfpmean_\idxdtime + \matfree{\mP_\idxdtime^-}{\cakfw_\idxdtime}$
      \State \(
      \cakfdd_\idxdtime
      \gets
      \begin{pmatrix}
        \cakfpdd_\idxdtime & \matfree{\cakfpcov_\idxdtime}{\cakfW_\idxdtime}
      \end{pmatrix}
      \)
      \State \textbf{return} $(\cakfmean_\idxdtime, \cakfdd_\idxdtime)$
    \EndFunction
  \end{algorithmic}
\end{algorithm}

\begin{proposition}
  When an identical \textup{\textsc{Policy}} is used, \cref{alg:projected_update,alg:update_pls} are equivalent (in exact precision).
\end{proposition}

\begin{proof}
  The principal difference between the two algorithms is that the quantities $\cakfw_\idxdtime$ and $\cakfW_\idxdtime$ are calculated differently.
  To show that these actually take the same values for the same policy, first note that in \cref{alg:projected_update} we have that $\cakfW_\idxdtime \cakfW_\idxdtime\T = \cakfprojH_\idxdtime\T \cakfprojgram_\idxdtime^\dagger \cakfprojH_\idxdtime = \obsH_\idxdtime\T \cakfacts_\idxdtime (\cakfacts_\idxdtime\T \cakfgram_\idxdtime \cakfacts_\idxdtime)^\dagger \cakfacts_\idxdtime\T \obsH_\idxdtime$.
  In \cref{alg:update_pls} the matrix $\cakfrepwdd_\idxdtime^{\iidx[\cakfprojobsdim_\idxdtime]}$ has the same span as $\cakfacts_\idxdtime$, but is orthogonalised to remove the need for the matrix inversion in $\cakfW_\idxdtime \cakfW_\idxdtime\T$.
  This essentially follows from the fact that \cref{alg:line:orthogonalise} implements a version of the Gram-Schmidt procedure with an adjustment to enforce orthogonality with-respect to $\inprod{\cdot}{\cdot}[\cakfgram_k]$ rather than the standard Euclidean inner product.

  To show this we proceed by induction.
  For the base step we need only show that $(\cakfrepwdd_\idxdtime^{\iidx[1]})^\top \cakfgram_\idxdtime \cakfrepwdd_\idxdtime^{\iidx[1]} = \mI$; this follows from the fact that since $\cakfrepwdd_\idxdtime^{\iidx[0]}$ is an empty matrix, $\vd_\idxdtime^{\iidx[1]} = \cakfact_\idxdtime^{\iidx[1]}$ and therefore
  \begin{align*}
    (\vd_\idxdtime^{\iidx[1]})^\top \cakfgram_\idxdtime \vd_\idxdtime^{\iidx[1]}                            & = (\cakfact_\idxdtime^{\iidx[1]})^\top \cakfgram_\idxdtime \vd_\idxdtime^{\iidx[1]} = \eta_\idxdtime^{\iidx[1]} \\
    \implies \norm*{\frac{\vd_\idxdtime^{\iidx[1]}}{\sqrt{\eta_\idxdtime^{\iidx[1]}}}}[\cakfgram_\idxdtime] & = 1.
  \end{align*}
  For the induction step suppose that $\cakfrepwdd_\idxdtime^{\iidx[i-1]}$ is $\cakfgram_\idxdtime$-orthonormal.
  Let $\vz^{\iidx} = \frac{\vd_\idxdtime^{\iidx}}{\sqrt{\eta_\idxdtime^{\iidx}}}$ and consider the matrix $$ (\cakfrepwdd_\idxdtime^{\iidx})\T\cakfgram_\idxdtime \cakfrepwdd_\idxdtime^{\iidx} =
    \begin{pmatrix}
      (\cakfrepwdd_\idxdtime^{\iidx[i-1]})^\top \cakfgram_\idxdtime \cakfrepwdd_\idxdtime^{\iidx[i-1]} & (\cakfrepwdd_\idxdtime^{\iidx[i-1]})\T \cakfgram_\idxdtime \vz^{\iidx} \\
      (\vz^{\iidx})\T \cakfgram_\idxdtime (\cakfrepwdd_\idxdtime^{\iidx[i-1]})                         & (\vz^{\iidx})\T \cakfgram_\idxdtime \vz^{\iidx}
    \end{pmatrix}
    .
  $$
  It is straightforward to show that $(\vz^{\iidx})\T \cakfgram_\idxdtime (\cakfrepwdd_\idxdtime^{\iidx[i-1]}) = 0$, since
  \begin{align}
    (\vd_\idxdtime^{\iidx})\T \cakfgram_\idxdtime (\cakfrepwdd_\idxdtime^{\iidx[i-1]})
     & =
    (\cakfact_\idxdtime^{\iidx})\T (\mI - \cakfgram_\idxdtime \cakfrepwdd_\idxdtime\iidx[i - 1] \ps{\cakfrepwdd_\idxdtime\iidx[i - 1]}\T )\cakfgram_\idxdtime (\cakfrepwdd_\idxdtime^{\iidx[i-1]})                        \\
     & = (\cakfact_\idxdtime^{\iidx})\T \cakfgram_\idxdtime \cakfrepwdd_\idxdtime^{\iidx[i-1]}
    - (\cakfact_\idxdtime^{\iidx})\T \cakfgram_\idxdtime \cakfrepwdd_\idxdtime\iidx[i - 1] \underbracket[0.1ex]{\ps{\cakfrepwdd_\idxdtime\iidx[i - 1]}\T \cakfgram_\idxdtime (\cakfrepwdd_\idxdtime^{\iidx[i-1]})}_{=\mI} \\
     & = \vec{0}
  \end{align}
  by the inductive assumption.
  It remains to show that $(\vz^{\iidx})\T \cakfgram \vz^{\iidx} = 1$.
  This follows from observing that
  \begin{align*}
    \norm{\vd_\idxdtime^{\iidx}}[\cakfgram_\idxdtime]^2 & = (\cakfact_\idxdtime^{\iidx})\T(\mI - \ \cakfgram_\idxdtime \cakfrepwdd_\idxdtime\iidx[i - 1] \ps{\cakfrepwdd_\idxdtime\iidx[i - 1]}\T ) \cakfgram_\idxdtime \vd_\idxdtime^{\iidx}                                                                                                   \\
                                                        & = (\cakfact_\idxdtime^{\iidx})\T \cakfgram_\idxdtime \vd_\idxdtime^{\iidx} - (\cakfact_\idxdtime^{\iidx})\T \cakfgram_\idxdtime \cakfrepwdd_\idxdtime\iidx[i - 1] \underbracket[0.1ex]{\ps{\cakfrepwdd_\idxdtime\iidx[i - 1]}\T \cakfgram_\idxdtime \vd_\idxdtime^{\iidx}}_{=\vec{0}} \\
                                                        & = \eta_\idxdtime^{\iidx}
  \end{align*}
  where equality with zero is from the calculation above.
  It follows that $\cakfacts_\idxdtime(\cakfacts_\idxdtime\T \cakfgram \cakfacts_\idxdtime) \cakfacts_\idxdtime\T = (\cakfrepwdd_\idxdtime^{\iidx[\cakfprojobsdim_\idxdtime]})\T \cakfrepwdd_\idxdtime^{\iidx[\cakfprojobsdim_\idxdtime]}$, which completes the proof.
\end{proof}

\section{SPACE-TIME SEPARABLE GAUSS--MARKOV PROCESSES}
\label{sec:stsgmps}
Assume we are given a spatiotemporal regression problem over the domain \(\gls{str:inputspace} = \temporalinputspace \times \spatialinputspace\) and a Gaussian process prior
\begin{equation}
  \gls{str:gpprior} \sim \gp{\gls{str:meanfn}}{\gls{str:covfn}}
\end{equation}
for the latent function \(\gls{str:targetfn} \colon \inputspace \to \R\), where \(\meanfn : \inputspace \to \R\) and \(\covfn : \inputspace \times \inputspace \to \R\).
Our goal will be to translate this batch GP regression problem into state-space form where under suitable assumptions the state dynamics are Markovian, such that we can perform exact and importantly linear-time inference via Bayesian filtering and smoothing.

\subsection{Spatiotemporal GP Regression in State-Space Form}
\label{sec:st-gpr-ssm}
As a first step, we augment the state with a sufficient number of \(\outputdim -1\) time derivatives, i.e.,
\begin{equation}
  \stsgmp(t, \vx) =
  \begin{pmatrix}
    \stsgmp_0(t, \vx) \\
    \vdots            \\
    \stsgmp_{d'-1}(t, \vx)
  \end{pmatrix}
  \defeq
  \begin{pmatrix}
    \stsgmp(t, \vx)             \\
    \pderiv{t}{\stsgmp(t, \vx)} \\
    \vdots                      \\
    \pderiv[(\outputdim -1)]{t}{\stsgmp(t, \vx)}
  \end{pmatrix}
  \in \R^{\outputdim},
\end{equation}
and assume the resulting Gaussian process is space-time separable.
\begin{definition}[Space-Time Separable Gaussian Process]
  \label{def:space-time-separable-gp}
  A $\gls{stsgmp:outputdim}$-output Gaussian process $\gls{stsgmp} \sim \gp{\gls{stsgmp:meanfn}}{\gls{stsgmp:covfn}}$ with index set $\gls{str:temporalinputspace} \times \gls{str:spatialinputspace}$ is called \emph{space-time separable} if $\stsgmpmeanfn(t, \vx) = \gls{stsgmp:meanfntime}(t) \cdot \gls{stsgmp:meanfnspace}(\vx)$ and
  \begin{equation*}
    \stsgmpcovfn((t_1, \vx_1), (t_2, \vx_2))
    = \gls{stsgmp:covfntime}(t_1, t_2) \cdot \gls{stsgmp:covfnspace}(\vx_1, \vx_2).
  \end{equation*}
\end{definition}
Then given that the temporal process $\stsgmp^t \sim \gp{\stsgmpmeanfntime}{\stsgmpcovfntime}$ is Markovian, we obtain the desired state-space representation, which can be computed exactly in closed form under suitable assumptions on the covariance function \(\stsgmpcovfntime\) (see \cref{rem:gp-priors-as-stsgmps}).
The following result formalizing this argument has been presented previously, but without an explicit proof \citep{Sarkka2012InfiniteDimensionalKalman,Solin2016StochasticDifferential,Hamelijnck2021SpatioTemporalVariationalGaussian}.
\begin{lemma}
  \label{lem:discretized-stsgmp-is-markov}
  Let $\stsgmp \sim \gp{\stsgmpmeanfn}{\stsgmpcovfn}$ be a space-time separable $\outputdim$-output Gaussian process with index set $\temporalinputspace \times \spatialinputspace$ such that $\stsgmp^t \sim \gp{\stsgmpmeanfntime}{\stsgmpcovfntime}$ is Markov with transition densities
  \(
  p(\stsgmp^t(t) \mid \stsgmp^t(s)) = \gaussianpdf{\stsgmp^t(t)}{\gmpA^t(t, s) \stsgmp^t(s) + \gmpb^t(t, s)}{\gmpnoisecov^t(t, s)}.
  \)
  Let $\xsall \in \spatialinputspace^{\nxsall}$ and define
  \begin{equation}
    \gmp(t)
    \defeq
    \stsgmp(t, \xsall)
    =
    \begin{pmatrix}
      \stsgmp_0(t, \xsall) \\
      \vdots               \\
      \stsgmp_{d'-1}(t, \xsall)
    \end{pmatrix}
    \in \R^{\outputdim \cdot \nxsall}
  \end{equation}
  for all $t \in \temporalinputspace$.
  Then $\gmp$ is a Gauss--Markov process with transition densities
  \begin{equation*}
    p(\gmp(t) \mid \gmp(s))
    = \gaussianpdf{\gmp(t)}{\gmpA(t, s) \gmp(s) + \gmpb(t, s)}{\gmpnoisecov(t, s)},
  \end{equation*}
  where
  \begin{align*}
    \gmpA(t, s)        & \defeq \gmpA^t(t, s) \otimes \mI_{\nxsall}                                 \\
    \gmpb(t, s)        & \defeq \gmpb^t(t, s) \otimes \stsgmpmeanfnspace(\xsall), \qquad \text{and} \\
    \gmpnoisecov(t, s) & \defeq \gmpnoisecov^t(t, s) \otimes \stsgmpcovfnspace(\xsall, \xsall).
  \end{align*}
\end{lemma}

\begin{remark}
  Abusing terminology, we refer to $\stsgmp$ as a \emph{space-time separable Gauss--Markov process} if $\stsgmp$ is space-time separable and $\gp{\stsgmpmeanfntime}{\stsgmpcovfntime}$ is Markov.
\end{remark}

To prove \cref{lem:discretized-stsgmp-is-markov}, we will need the following intermediate result, which can be found in most standard textbooks, for example in Appendix B of \citet{Bishop2006PatternRecognition}.
We restate the result here for convenience:

\begin{lemma}
  \label{lem:cond-gaussian-rv}
  Let
  \begin{equation*}
    \pvec{\rvx_1 \\ \rvx_2}
    \sim
    \gaussian{
      \begin{pmatrix}
        \vmu_1 \\
        \vmu_2
      \end{pmatrix}
    }{
      \begin{pmatrix}
        \mSigma_{11} & \mSigma_{21}\T \\
        \mSigma_{21} & \mSigma_{22}
      \end{pmatrix}
    }.
  \end{equation*}
  Then
  \begin{equation*}
    \condrv{\rvx_2 \mid \rvx_1}
    \sim
    \gaussian{
      \gmpA \rvx_1 + \gmpb
    }{
      \gmpnoisecov
    },
  \end{equation*}
  where $\gmpA = \mSigma_{21} \mSigma_{11}\pinv$, $\gmpb = (\vmu_2 - \gmpA \vmu_1)$, and $\gmpnoisecov = \mSigma_{22} - \gmpA \mSigma_{11} \gmpA\T$.
\end{lemma}
We can now use \Cref{lem:cond-gaussian-rv} to show that every Gauss--Markov process has transition densities of the form required by \cref{lem:discretized-stsgmp-is-markov}.
\begin{proof}[Proof of \cref{lem:discretized-stsgmp-is-markov}]
  By definition, $\gmp$ is a $\outputdim \cdot \nxsall$-output Gaussian process with index set $\temporalinputspace$, whose mean and covariance functions are given by
  \begin{align*}
    \stsgmpmeanfn_\gmp(t)
     & \defeq \stsgmpmeanfntime(t) \otimes \stsgmpmeanfnspace(\xsall), \qquad \text{and} \\
    \stsgmpcovfn_\gmp(t_1, t_2)
     & \defeq \stsgmpcovfntime(t_1, t_2) \otimes \stsgmpcovfnspace(\xsall, \xsall),
  \end{align*}
  respectively.
  Let $t_0 \le t_1 < \dotsb < t_{\ntstrain} \le T$ and define
  \begin{align*}
    \gmpA^t_\idxdtime        & \defeq \gmpA^t(t_{\idxdtime + 1}, t_\idxdtime),                   \\
    \gmpb^t_\idxdtime        & \defeq \gmpb^t(t_{\idxdtime + 1}, t_\idxdtime), \qquad \text{and} \\
    \gmpnoisecov^t_\idxdtime & \defeq \gmpnoisecov^t(t_{\idxdtime + 1}, t_\idxdtime).
  \end{align*}
  We have
  \begin{align*}
    \stsgmpmeanfntime(t_{\idxdtime + 1})                   & = \gmpA^t_\idxdtime \stsgmpmeanfntime(t_\idxdtime) + \gmpb^t_\idxdtime,                                                            \\
    \stsgmpcovfntime(t_{\idxdtime + 1}, t_{\idxdtime + 1}) & = \gmpA^t_\idxdtime \stsgmpcovfntime(t_\idxdtime, t_\idxdtime) (\gmpA^t_\idxdtime)\T + \gmpnoisecov^t_\idxdtime, \qquad \text{and} \\
    \stsgmpcovfntime(t_\idxdtime, t_{\idxdtime + l})       & = \stsgmpcovfntime(t_\idxdtime, t_\idxdtime) \prod_{j = 0}^{l - 1} (\gmpA^t_{\idxdtime + j})\T.
  \end{align*}
  It follows that
  \begin{align*}
    \stsgmpmeanfn_\gmp(t_{\idxdtime + 1})
     & = \stsgmpmeanfntime(t_{\idxdtime + 1}) \otimes \stsgmpmeanfnspace(\xsall)                                                                                                                                                                                      \\
     & = (\gmpA^t_\idxdtime \stsgmpmeanfntime(t_\idxdtime) + \gmpb^t_\idxdtime) \otimes \stsgmpmeanfnspace(\xsall)                                                                                                                                                    \\
     & = \underbracket[0.1ex]{(\gmpA^t_\idxdtime \otimes \mI_n)}_{\rdefeq \gmpA_\idxdtime} (\stsgmpmeanfntime(t_\idxdtime) \otimes \stsgmpmeanfnspace(\xsall)) + \underbracket[0.1ex]{\gmpb^t_\idxdtime \otimes \stsgmpmeanfnspace(\xsall)}_{\rdefeq \gmpb_\idxdtime} \\
     & = \gmpA_\idxdtime \stsgmpmeanfn_\gmp(t_\idxdtime) + \gmpb_\idxdtime,
  \end{align*}
  as well as
  \begin{align*}
    \stsgmpcovfn_\gmp(t_\idxdtime, t_\idxdtime)
     & = \stsgmpcovfntime(t_\idxdtime, t_\idxdtime) \otimes \stsgmpcovfnspace(\xsall, \xsall)                                                                                                                                                                                                     \\
     & = (\gmpA^t_\idxdtime \stsgmpcovfntime(t_\idxdtime, t_\idxdtime) (\gmpA^t_\idxdtime)\T + \gmpnoisecov^t_\idxdtime) \otimes \stsgmpcovfnspace(\xsall, \xsall)                                                                                                                                \\
     & = (\gmpA^t_\idxdtime \otimes \mI_n) (\stsgmpcovfntime(t_\idxdtime, t_\idxdtime) \otimes \stsgmpcovfnspace(\xsall, \xsall)) (\gmpA^t_\idxdtime \otimes \mI_n)\T + \underbracket[0.1ex]{\gmpnoisecov^t_\idxdtime \otimes \stsgmpcovfnspace(\xsall, \xsall)}_{\rdefeq \gmpnoisecov_\idxdtime} \\
     & = \gmpA_\idxdtime \stsgmpcovfn_\gmp(t_\idxdtime, t_\idxdtime) \gmpA_\idxdtime\T + \gmpnoisecov_\idxdtime,                                                                                                                                                                                  \\
  \end{align*}
  and
  \begin{align*}
    \stsgmpcovfn_\gmp(t_\idxdtime, t_{\idxdtime + l})
     & = \stsgmpcovfntime(t_\idxdtime, t_{\idxdtime + l}) \otimes \stsgmpcovfnspace(\xsall, \xsall)                                                               \\
     & = \left( \stsgmpcovfntime(t_\idxdtime, t_\idxdtime) \prod_{j = 0}^{l - 1} (\gmpA^t_{\idxdtime + j})\T \right) \otimes \stsgmpcovfnspace(\xsall, \xsall)    \\
     & = (\stsgmpcovfntime(t_\idxdtime, t_\idxdtime) \otimes \stsgmpcovfnspace(\xsall, \xsall)) \prod_{j = 0}^{l - 1} (\gmpA^t_{\idxdtime + j} \otimes \mat{I})\T \\
     & = \stsgmpcovfn_\gmp(t_\idxdtime, t_\idxdtime)  \prod_{j = 0}^{l - 1} \gmpA_{\idxdtime + j}\T.
  \end{align*}
  Moreover, by \cref{lem:cond-gaussian-rv}, we have
  \begin{equation*}
    p(\gmp(t_{\idxdtime + 1}) \mid \gmp(t_\idxdtime))
    = \gaussianpdf{\gmp(t_{\idxdtime + 1})}{\gmpA_\idxdtime \gmp(t_\idxdtime) + \gmpb_\idxdtime}{\gmpnoisecov_\idxdtime}.
  \end{equation*}
  All in all, this shows that
  \begin{equation*}
    p(\gmp(t_1), \dotsc, \gmp(t_{\ntstrain}))
    = p(\gmp(t_1)) \prod_{\idxdtime = 2}^{\ntstrain} p(\gmp(t_\idxdtime) \mid \gmp(t_{\idxdtime - 1})).
  \end{equation*}
  The statement then follows from \citet[``Consequences of the definition'' below Definition 6.2]{LeGall2016BrownianMotion}.
\end{proof}

\begin{remark}[Converting Spatiotemporal GP Priors to Space-Time Separable Gauss--Markov Processes]
  \label{rem:gp-priors-as-stsgmps}
  Not every Gaussian process prior \(\gpprior \sim \gp{\meanfn}{\covfn}\) induces a space-time separable Gauss--Markov process, even if both \(\meanfn\) and \(\covfn\) are separable, such that \(\meanfn(\vz) = \meanfntime(t)\meanfnspace(\vx)\) and \(\covfn(\vz, \vz') = \covfntime(t, t')\covfnspace(\vx, \vx')\), e.g if \(\covfntime\) is an exponentiated quadratic kernel.
  However, if \(\covfntime\) is stationary and the spectral density of \(\covfntime\) is a rational function of the form
  \begin{equation}
    S^t(\omega) = \frac{\textrm{(constant)}}{(\textrm{polynomial in } \omega^2)}
  \end{equation}
  then a corresponding STSGMP exists.
  This is the case for example if \(\covfntime\) is a Mat\'ern(\(\nu\)) kernel with differentiability parameter \(p\) such that \(\nu= p + \frac{1}{2}\).
  See \citet[][Sec.~4]{Hartikainen2010KalmanFiltering}, \citet{Sarkka2013SpatiotemporalLearning}, and \citet[][Sec.~4.3]{Solin2016StochasticDifferential} for details.
\end{remark}

\subsection{Pointwise Error Bound}
\label{sec:error-bound}

Having formalized assumptions under which the (spatiotemporal) batch GP regression problem can be equivalently formulated in state-space form and thus solved via Bayesian filtering and smoothing, we now aim to give a relative error bound for the approximate posterior mean computed by the CAKS in terms of its approximate marginal variance.

\subsubsection{(Iteratively Approximated) Batch Gaussian Process Regression}

The CAKF and CAKS can be viewed as performing exact inference in an approximate observation model.
Therefore we can analyze its error via the corresponding iterative approximation for the batch GP regression problem as introduced by \citet{Wenger2022PosteriorComputational}.

\begin{definition}[Iteratively Approximated Batch GP Regression]
  \label{def:iter-approx-gp}
  Let \(\inputspace\) be a non-empty set, \(\gpprior \sim \gp{\meanfn}{\covfn}\) a Gaussian process prior for the latent function \(\strtargetfn \in \rkhs{\covfn}\) assumed to lie in the RKHS induced by \(\covfn\).
  Define the noise scale \(\noisescale \geq 0\), the covariance function \(\noisycovfn(\vz, \vz') \defeq \covfn(\vz, \vz') + \noisescale^2 \delta(\vz, \vz')\) of the observed process \citep[Eqn.~(32)]{Kanagawa2018GaussianProcesses} and let \(\strobsfn(\cdot) \in \rkhs{\noisycovfn}\) be the function generating the data\footnotemark.
  \footnotetext{By Section 6 of \citet{Aronszajn1950TheoryReproducing}, functions \(y \in \rkhs{\noisycovfn}\) can be written as a sum \(y(\cdot) = f(\cdot) + \varepsilon(\cdot)\) of functions \(f \in \rkhs{\covfn}\) and \(\varepsilon \in \rkhs{\noisescale^2 \delta}\).}
  Assume we've observed training data \(\ystrain = {\strobsfn}(\zstrain) = (\strobsfn(\ztrain{1}), \dots, \strobsfn(\ztrain{\ntraindata}))\T \in \R^{\ntraindata}\) at inputs \(\zstrain = (\ztrain{1}, \dots, \ztrain{\ntraindata})\T \in \inputspace^{\ntraindata}\) and let \(\gls{itergp:actions} \in \R^{\ntraindata \times \gls{itergp:nactions}}\) be a matrix with linearly independent columns.
  Following \citet{Wenger2022PosteriorComputational}, define the \emph{iteratively approximated batch GP posterior} as \((\gpprior \mid \mactions\T \ystrain) \sim \gp{\approxpostmeanfn}{\approxpostcovfn}\), with
  \begin{equation}
    \label{eqn:itergp-posterior}
    \begin{aligned}
      \gls{itergp:approxpostmeanfn}(\vz)      & = \meanfn(\vz) + \covfn(\vz, \zstrain)\mC (\ystrain - \meanfn(\zstrain)), \\
      \gls{itergp:approxpostcovfn}(\vz, \vz') & = \covfn(\vz, \vz') - \covfn(\vz, \zstrain)\mC \covfn(\zstrain, \vz'),
    \end{aligned}
  \end{equation}
  where \(\mC = \mactions (\mactions\T(\covfn(\zstrain, \zstrain) + \noisescale^2\mI)\mactions\T)^{\pinv}\mactions\T\)
\end{definition}

\begin{lemma}[Iteratively Approximated GP as Exact Inference Given a Modified Observation Model]
  \label{lem:iter-approx-gp-modified-likelihood}
  Given a Gaussian process prior \(\gpprior \sim \gp{\meanfn}{\covfn}\) and training data \((\zstrain, \ystrain)\) the iteratively approximated batch GP posterior \((\gpprior \mid \mactions\T \ystrain) \sim \gp{\approxpostmeanfn}{\approxpostcovfn}\) (see \cref{def:iter-approx-gp}) is equivalent to an exact batch GP posterior \((\gpprior \mid \projystrain)\) given observations \(\gls{itergp:projystrain} = \mactions\T\ystrain\) observed according to the modified likelihood \(\gls{itergp:projystrainrv} \mid \gpprior(\zstrain) \sim \gaussian{\mactions\T \gpprior(\zstrain)}{\noisescale^2 \mactions\T\mactions}\).
\end{lemma}

\begin{proof}
  By basic properties of Gaussian distributions, we have for arbitrary \(\zsall \in \inputspace^{\nzsall}\) that
  \begin{equation*}
    \begin{pmatrix}
      \projystrain \\
      \gpprior(\zsall)
    \end{pmatrix}
    \sim \gaussian{
      \begin{pmatrix}
        \mactions\T \meanfn(\zstrain) \\
        \meanfn(\zsall)
      \end{pmatrix}
    }{
      \begin{pmatrix}
        \mactions\T \covfn(\zstrain, \zstrain) \mactions + \noisescale^2 \mactions\T \mactions & \mactions\T \covfn(\zstrain, \zsall) \\
        \covfn(\zsall, \zstrain) \mactions                                                     & \covfn(\zsall, \zsall)               \\
      \end{pmatrix}
    }
  \end{equation*}
  is jointly Gaussian.
  Therefore by \cref{lem:cond-gaussian-rv} we have that \(\gpprior(\zsall) \mid \projystrain \sim \gaussian{\approxpostmeanfn(\zsall)}{\approxpostcovfn(\zsall, \zsall)}\) where
  \begin{align*}
    \approxpostmeanfn(\zsall)        & = \meanfn(\zsall) + \covfn(\zsall, \zstrain)\mactions (\mactions\T(\covfn(\zstrain, \zstrain) + \noisescale^2\mI)\mactions\T)^{\pinv}(\projystrain - \mactions\T \meanfn(\zstrain)), \\
    \approxpostcovfn(\zsall, \zsall) & = \covfn(\zsall, \zsall) - \covfn(\zsall, \zstrain)\mactions (\mactions\T(\covfn(\zstrain, \zstrain) + \noisescale^2\mI)\mactions\T)^{\pinv}\mactions\T \covfn(\zstrain, \zsall)
  \end{align*}
  which is equivalent to the form of iteratively approximated GP posterior in \cref{def:iter-approx-gp}.
  This proves the claim.
\end{proof}

The iteratively approximated posterior mean satisfies a pointwise worst-case error bound in a unit ball in the underlying RKHS, where the bound is given by the approximate standard deviation.
Therefore, the output of the approximate method directly provides an error bound on its prediction error, that includes any error introduced through approximation.

\begin{restatable}[Worst-Case Error of (Iteratively Approximated) Batch GP Regression \citep{Wenger2022PosteriorComputational}]{theorem}{thmWorstCaseErrorGP}
  \label{thm:worst-case-error-batch-gp}
  Consider the (iteratively approximated) GP posterior \((\gpprior \mid \mactions\T \ystrain) \sim \gp{\approxpostmeanfn}{\approxpostcovfn}\) given in \cref{def:iter-approx-gp}.
  The pointwise worst-case error of the (approximate) posterior mean \(\approxpostmeanfn\) for an arbitrary data-generating function \(\strobsfn \in \rkhs{K^\sigma}\) such that \(\norm{\strobsfn}_{\rkhs{\noisycovfn}} \leq 1\) is given by
  \begin{equation}
    \sup\limits_{\substack{y \in \rkhs{\noisycovfn} \\ \norm{y}_{\rkhs{\noisycovfn}} \leq 1}} \abs{y(\vz) - \approxpostmeanfn[y](\vz)} = \sqrt{\approxpostcovfn(\vz, \vz) + \noisescale^2}.
  \end{equation}
  for any \(\vz \in \inputspace \setminus \zstrain\) not in the training data.
  In the absence of observation noise, i.e., \(\noisescale^2=0\), this holds for all \(\vz \in \inputspace\).
  Note that this result also trivially extends to the exact batch GP posterior by choosing \(\mactions = \mI_{\ntraindata \times \ntraindata}\).
\end{restatable}

\begin{proof}
  This result and its proof are identical to Theorem 2 in \citet{Wenger2022PosteriorComputational}.
  We give a proof in our notation for completeness.
  Let \(\ztrain{0} = \vz\), \(c_0 = 1\) and \(c_j = - (\mC \noisycovfn(\zstrain, \vz))_j\) for \(j=1, \dots, \ntraindata\).
  Then by Lemma 3.9 of \citet{Kanagawa2018GaussianProcesses}, it holds that
  \begin{align*}
    \left(\sup\limits_{\substack{y \in \rkhs{\noisycovfn}                                                                                                                                                                                                                                                                                                                                                                                               \\
        \norm{y}_{\rkhs{\noisycovfn}} \leq 1}} \abs{y(\vz) - \approxpostmeanfn[y](\vz)}\right)^2
     & = \left( \sup\limits_{\substack{y \in \rkhs{\noisycovfn}                                                                                                                                                                                                                                                                                                                                                                                         \\
    \norm{y}_{\rkhs{\noisycovfn}} \leq 1}} \sum_{j=0}^{\ntraindata}c_j y(\vz_j)\right)^2 = \norm*{\sum_{j=0}^{\ntraindata}c_j\covfn(\cdot, \vz_j)}_{\rkhs{\noisycovfn}}^2                                                                                                                                                                                                                                                                               \\
     & = \norm*{\noisycovfn(\cdot, \ztrain{0}) - \sum_{j=1}^{\ntraindata} c_j \noisycovfn(\cdot, \vz_j)}_{\rkhs{\noisycovfn}}^2 = \norm*{\noisycovfn(\cdot, \vz) - \covfn^{\sigma}(\cdot, \zstrain)\mC \noisycovfn(\zstrain, \vz)}_{\rkhs{\noisycovfn}}^2                                                                                                                                                                                               \\
     & = \inprod{\noisycovfn(\cdot, \vz)}{\noisycovfn(\cdot, \vz)}[\rkhs{K^\sigma}] - 2 \inprod{\noisycovfn(\cdot, \vz)}{\covfn^{\sigma}(\cdot, \zstrain)\mC \noisycovfn(\zstrain, \vz)}[\rkhs{K^\sigma}]                                                                                                                                                                                                                                               \\
     & \quad + \inprod{\covfn^{\sigma}(\cdot, \zstrain)\mC \noisycovfn(\zstrain, \vz)}{\covfn^{\sigma}(\cdot, \zstrain)\mC \noisycovfn(\zstrain, \vz)}[\rkhs{K^\sigma}] \intertext{By the reproducing property it holds that}                             & = \noisycovfn(\vz, \vz) - 2 \covfn^{\sigma}(\vz, \zstrain)\mC \noisycovfn(\zstrain, \vz) + \covfn^{\sigma}(\vz, \zstrain)\mC \noisycovfn(\zstrain, \zstrain) \mC \noisycovfn(\zstrain, \vz) \\
     & = \noisycovfn(\vz, \vz) - \covfn^{\sigma}(\vz, \zstrain)\mC \noisycovfn(\zstrain, \vz)                                                                                                                                                                                                                                                                                                                                                           \\
     & \eqcolon {\approxpostcovfn}^\sigma(\vz, \vz)
  \end{align*}
  Now if \(\vz \notin \zstrain\) or \(\noisescale^2 = 0\) it holds that \(\covfn^{\sigma}(\vz, \zstrain) = \covfn(\vz, \zstrain)\) and therefore \({\approxpostcovfn}^\sigma(\vz, \vz) = \approxpostcovfn(\vz, \vz) + \noisescale^2\).
\end{proof}

\subsubsection{Computation-aware Filtering and Smoothing}

Having obtained an error bound for the iteratively approximated batch GP posterior, we now aim to show that the CAKF and CAKS compute precisely the same posterior marginals and thus satisfy the same error bound.
We do so by leveraging \cref{lem:discretized-stsgmp-is-markov} describing how to translate between a (batch) GP regression problem and an equivalent state-space formulation under suitable assumptions on the model.

\begin{restatable}[Connecting (Computation-Aware) Batch Spatio-temporal GP Regression and Filtering and Smoothing]{proposition}{propBatchGPandSSM}
  \label{prop:batch-gp-and-ssm}
  Consider the following spatiotemporal regression problem over the domain \(\inputspace = \temporalinputspace \times \spatialinputspace\).
  Define a space-time separable Gauss--Markov process \(\stsgmp \sim \gp{\stsgmpmeanfn}{\stsgmpcovfn}\) such that its first component \(\gpprior \defeq \stsgmp_0 \sim \gp{\meanfn}{\covfn}\) defines a Gaussian process prior for the latent function \(\strtargetfn \in \rkhs{\covfn}\), where \(\meanfn(t, \vx) = \stsgmpmeanfntime_0(t)\stsgmpmeanfnspace(\vx)\) and \(\covfn((t, \vx), (t', \vx')) = \stsgmpcovfntime_0(t, t')\stsgmpcovfnspace(\vx, \vx')\).
  Assume we are given a training dataset consisting of $\ntraindata = \sum_{\idxdtime = 1}^\ntstrain \nxstrain{\idxdtime}$ inputs \(\zstrain = ((\ttrain{1}, \xtrain{1}{1}), \dotsc, (\ttrain{1}, \xtrain{1}{\nxstrain{1}}), \dotsc, (\ttrain{\ntstrain}, \xtrain{\ntstrain}{1}), \dotsc, (\ttrain{\ntstrain}, \xtrain{\ntstrain}{\nxstrain{\ntstrain}})) \in \inputspace^{\ntraindata}\) and targets \(\ystrain \in (\obs_1, \dotsc, \obs_\gmplen) \in \R^{\ntraindata}\).
  Then for any test input \(\vz = (t, \vx) \in \inputspace\) the computation-aware smoother computes the mean \(\approxpostmeanfn(\vz)\) and variance \(\approxpostcovfn(\vz, \vz)\) of the marginal distribution of the iteratively approximated batch GP posterior \({(\gpprior \mid \mactions\T \ystrain) \sim \gp{\approxpostmeanfn}{\approxpostcovfn}}\) with
  \begin{equation}
    \label{eqn:batch-actions-from-cakf-actions}
    \mactions =
    \begin{pmatrix}
      \cakfacts_1 &        & \mat{0}               \\
                  & \ddots &                       \\
      \mat{0}     &        & \cakfacts_{\ntstrain}
    \end{pmatrix}
    \in \R^{\ntraindata \times \sum_{\idxdtime=1}^{\ntstrain}\nactions_\idxdtime}
  \end{equation}
  evaluated at the given test input $\vz$, i.e.
  \begin{align*}
    \caksmean(\vz \mid \strobsfn)_0 & = \approxpostmeanfn(\vz), \quad \text{and} \\
    \cakscov(\vz)_{0,0}             & = \approxpostcovfn(\vz, \vz).
  \end{align*}
  If \(t \geq \ttrain{\ntstrain}\) it suffices to run the computation-aware filter.
\end{restatable}

\begin{proof}
  Let \(\xsall \in \spatialinputspace^{\nxsall}\) be a vector containing all (unique) spatial training points and the spatial test point $\vx$ as its zeroth component.
  By \cref{lem:discretized-stsgmp-is-markov} a space-time separable Gauss--Markov process \(\stsgmp\) evaluated at the spatial inputs \(\xsall\), i.e., \(\gmp(t) \defeq \stsgmp(t, \xsall) \in \R^{\nxsall \outputdim}\) and $\gmp_0(t) = \gpprior(t, x)$, admits a state-space representation with dynamics
  \begin{equation}
    \label{eqn:proof-itergp-cakf-equivalence-dynamics-model}
    \gmp(t) = \gmpA(t, s) \gmp(s) + \gmpb(t, s) + \gmpnoise(t,s)
  \end{equation}
  where \(\gmpnoise(t, s) \sim \gaussian{\vec{0}}{\gmpnoisecov(t, s)}\) and the observation model is by assumption given by
  \begin{equation}
    \label{eqn:eqn:proof-itergp-cakf-equivalence-observation-model}
    \cakfprojobsrv_\idxdtime
    = \cakfacts\T_\idxdtime \stsgmp_0(\ttrain{\idxdtime}, \xstrain{\idxdtime}) + \cakfprojobsnoise_\idxdtime
    = \cakfacts\T_\idxdtime \obsH_\idxdtime \gmp(\ttrain{\idxdtime}) + \cakfprojobsnoise_\idxdtime \in \R^{\nactions_\idxdtime}
  \end{equation}
  where $\obsH_\idxdtime \in \R^{\nxstrain{\idxdtime} \times \gmpdim}$ is implicitly defined by $\obsH_\idxdtime \gmp(\ttrain{\idxdtime}) = \stsgmp_0(\ttrain{\idxdtime}, \xstrain{\idxdtime})$.

  Now, in a linear Gaussian state-space model as defined by \cref{eqn:proof-itergp-cakf-equivalence-dynamics-model} and \cref{eqn:eqn:proof-itergp-cakf-equivalence-observation-model}, the vanilla Kalman filter and smoother \citep[Alg.~3.14 \& 3.17]{Sarkka2006RecursiveBayesian} compute the posterior marginal
  \begin{equation}
    (\stsgmp(t, \xsall) \mid \cakfprojobsrv_1 = \cakfprojobs_1, \dots, \cakfprojobsrv_{\ntstrain} = \cakfprojobs_{\ntstrain}) \sim \gaussian{\stsgmpapproxpostmeanfn(t, \xsall)}{\stsgmpapproxpostcovfn((t, \xsall), (t, \xsall))}
  \end{equation}
  at an arbitrary timepoint \(t\) exactly, and if \(t \geq \ttrain{\ntstrain}\), then it suffices to run the filter \citep[Alg.~10.15 \& 10.18]{Sarkka2019AppliedStochastic}.
  By construction, the output of the computation-aware filter and smoother for given \(\cakfacts_1, \dots, \cakfacts_{\ntstrain}\) (assuming no truncation) are equivalent to applying the vanilla filter and smoother to the state-space model defined by \cref{eqn:proof-itergp-cakf-equivalence-dynamics-model} and \cref{eqn:eqn:proof-itergp-cakf-equivalence-observation-model}, i.e., $\caksmean(\vz \mid \strobsfn) = \stsgmpapproxpostmeanfn(\vz)$ and $\cakscov(\vz) = \stsgmpapproxpostcovfn(\vz, \vz)$.
  By \cref{lem:iter-approx-gp-modified-likelihood}, the zeroth components of the CAKS moments $(\caksmean(\vz))_{0}, (\cakscov(\vz))_{0, 0}$ are hence equal to the marginal moments of the iteratively approximated batch GP posterior \({(\gpprior \mid \mactions\T \ystrain) \sim \gp{\approxpostmeanfn}{\approxpostcovfn}}\) evaluated at the test point \(\vz = (t, \vx)\).
  This completes the proof.
\end{proof}

\thmErrorBoundCAKF*

\begin{proof}
  By \Cref{prop:batch-gp-and-ssm} the mean \(\caksmean(\vz \mid \strobsfn)_0\) and variance \(\cakscov(\vz, \vz)_{0,0}\) computed by the computation-aware filter and smoother for the test input \(\vz = (t, \vx)\) are equivalent to the marginal posterior mean and variance of an iteratively approximated batch GP posterior with the induced (space-time separable) prior \(\gpprior \defeq \stsgmp_0 \sim \gp{\meanfn}{\covfn}\) and actions \(\mactions\) defined as in \cref{eqn:batch-actions-from-cakf-actions}.
  Therefore by \cref{thm:worst-case-error-batch-gp} it holds that
  \begin{equation*}
    \sup\limits_{\substack{\tilde{y} \in \rkhs{\noisycovfn} \\ \norm{\tilde{y}}_{\rkhs{\noisycovfn}} \leq 1}} \abs{\tilde{y}(\vz) - \caksmean(\vz \mid \tilde{y})_0} = \sqrt{\cakscov(\vz, \vz)_{0, 0} + \noisescale^2}
  \end{equation*}
  Recognizing that the supremum is achieved on the boundary, i.e., where \(\norm{\tilde{y}}_{\rkhs{\noisycovfn}} = 1\), we can equivalently consider \(\tilde{y}(\cdot) = \frac{y(\cdot)}{\norm{y}_{\rkhs{\noisycovfn}}}\) for $y \in \rkhs{\noisycovfn} \setminus \set{0}$ arbitrary.
  Then it holds that \(\caksmean(\vz \mid \tilde{y})_0 = \frac{\caksmean(\vz \mid y)_0}{\norm{y}_{\rkhs{\noisycovfn}}}\) and therefore we have
  \begin{equation*}
    \sup\limits_{y \in \rkhs{\noisycovfn} \setminus \set{0}} \frac{\abs{y(\vz) - \caksmean(\vz \mid y)_0}}{\norm{y}_{\rkhs{\noisycovfn}}} = \sup\limits_{\substack{\tilde{y} \in \rkhs{\noisycovfn} \\ \norm{\tilde{y}}_{\rkhs{\noisycovfn}} = 1}} \abs{\tilde{y}(\vz) - \caksmean(\vz \mid \tilde{y})_0} = \sqrt{\cakscov(\vz, \vz)_{0,0} + \noisescale^2}
  \end{equation*}
  This proves the claim.
\end{proof}

\section{EXPERIMENTS}
\label{sec:appendix-experiments}
\subsection{Metrics}
\label{sec:appendix-metrics}
To assess the predictive performance of the models considered for our experiments in \cref{sec:experiments}, we mainly use two different metrics:

\textbf{Mean Squared Error}
To assess the quality of the predictive mean $\tilde{\vm}_\idxdtime$ as a point estimate for the target state $\gmptarget_\idxdtime$, we use the mean squared error (MSE)
\begin{equation*}
  \operatorname{MSE}(\gmptarget_\idxdtime, \tilde{\vm}_\idxdtime)
  = \frac{1}{\gmpdim} \norm{\gmptarget_\idxdtime - \tilde{\vm}_\idxdtime}_2^2.
\end{equation*}

\textbf{Average Negative Log Density}
All inference algorithms considered in this work provide Gaussian conditional distributions $\gaussian{\tilde{\vm}_\idxdtime}{\tilde{\mP}_\idxdtime}$ as estimates of the unknown target state $\gmptarget_\idxdtime$.
Hence, we can use the average negative log density (NLD)
\begin{align*}
  \operatorname{AvgNLD}\ps{\condrv{\gmptarget_\idxdtime \given \tilde{\vm}_\idxdtime, \tilde{\mP}_\idxdtime}}
   & = - \frac{1}{\gmpdim} \sum_{d = 1}^\gmpdim \log \gaussianpdf{\gmptarget_{\idxdtime, d}}{\tilde{\vm}_{\idxdtime, d}}{\tilde{\mP}_{\idxdtime, d,d}}                                     \\
   & = \frac{1}{2 \gmpdim} \sum_{d = 1}^\gmpdim \frac{(\gmptarget_{\idxdtime, d} - \tilde{\vm}_{\idxdtime, d})^2}{\tilde{\mP}_{\idxdtime, d,d}} + \log(2 \pi \tilde{\mP}_{\idxdtime, d,d})
\end{align*}
corresponding to this Gaussian distribution as a measure of the quality of the predictive uncertainty.
We use the average NLD as opposed to the ``full'' NLD, i.e., $-\log \gaussianpdf{\gmptarget_\idxdtime}{\tilde{\vm}_\idxdtime}{\tilde{\mP}_\idxdtime}$, since computation of the latter is infeasible for high-dimensional state spaces.

For both metrics, we choose $(\tilde{\vm}_\idxdtime, \tilde{\mP}_\idxdtime) \in \set{(\cakfmean_\idxdtime, \cakfcov_\idxdtime), (\caksmean_\idxdtime, \cakscov_\idxdtime), (\kfmean_\idxdtime, \kfcov_\idxdtime), (\ksmean_\idxdtime, \kscov_\idxdtime), \dotsc}$ as appropriate.
We aggregate the metrics across entire trajectories by averaging.
In \cref{fig:work-precision-era5,fig:work-precision-policies}, we also marginalize over the train and test parts of the states before computing the metrics to obtain a more fine-grained view of the performance of the algorithms.

\subsection{Experiment Details}
\label{sec:experiment-details}
We provide some additional details for the experiments conducted in \Cref{sec:experiments} in this section.
\subsubsection{Comparison to Other Methods}
\label{sec:appendix-on-model-data}
The domain of the problem is $[0, \tmax] \times \spatialinputspace$ with $\spatialinputspace = [0, 20]^{\spatialinputspacedim}$.

\textbf{Model}
The temporal and spatial covariance functions are given by $\covfntime = \sigma^2 \cdot \text{Matérn}(\nicefrac{3}{2}, \ell_t)$ and $\covfnspace = \text{Matérn}(\nicefrac{3}{2}, \ell_\vx)$ with lengthscales $\ell_t = \ell_\vx = 0.5$ and output scale $\sigma = 1$.

\textbf{Data}
We discretize the model on regular grids $\xsall \in \spatialinputspace^{\nxsall}$ and $\tsall \in [0, \tmax]^{\ntsall}$ of $\nxsall = 100^{\spatialinputspacedim}$ points in space and $\ntsall = 20 \cdot \tmax$ points in time.
Afterwards, we sample a ground-truth trajectory $\set{\gmptarget_\idxdtime}_{\idxdtime = 1}^{\ntsall}$ from the resulting state-space model.
Drawing samples from the SSM requires access to (left) square roots of the initial covariance matrix and the transition noise covariance matrices of the discretized model.
As shown in \cref{sec:st-gpr-ssm}, these matrices are Kronecker products of a $\R^{2 \times 2}$ matrices with the matrix $\covfnspace(\xsall, \xsall) \in \R^{\nxsall \times \nxsall}$.
Hence, we can compute the required (left) square roots by computing a (left) square root of $\covfnspace(\xsall, \xsall)$.
Since a (left) square root of $\covfnspace(\xsall, \xsall)$ is not available in closed form, we compute it by means of an eigendecomposition (on the GPU).
However, this severely limits the number of spatial points $\nxsall$ that can be considered in the experiments.
We pick random subsets $\tstrain \subset \tsall$ and $\xstrain{} \subset \xsall$ of $\ntstrain = \nicefrac{\ntsall}{10}$ temporal and $\nxstrain{\idxdtime} = 20^{\spatialinputspacedim}$ spatial points as training locations $\zsall = \tstrain \times \xstrain{}$ and perturb the corresponding entries of the ground-truth states $\gmptarget_\idxdtime$ with additive i.i.d.~Gaussian measurement noise with standard deviation $\lambda = 0.1$.

\textbf{Evaluation}
We repeat the data sampling procedure outlined above five times with different random seeds and evaluate the performance of each inference algorithm independently on each of the resulting datasets.
The solid lines in \cref{fig:ensemble-on-model-wall-time,fig:ensemble-on-model-rank,fig:error-dynamics} are obtained by taking the median over the resulting metrics (and the wall-time in case of \cref{fig:ensemble-on-model-wall-time}) across the different runs while the individual values are scattered in the background with reduced opacity.
For the CAKF/CAKS and the ensemble Kalman filters, we study the effect of the computational budget on the approximation by varying the rank parameter $r$.
For a fair comparison, we fix both the number of iterations and the maximal rank after truncation in the CAKF/CAKS to the same values used in the ensemble Kalman filters, i.e., $\cakfmaxprojobsdim_\idxdtime = \cakfmaxtddrank_\idxdtime = r$.

\textbf{Comparison to EnKF and ETKF}
\begin{wrapfigure}[28]{R}{0.5\linewidth}
  \centering
  \includegraphics[width=\linewidth]{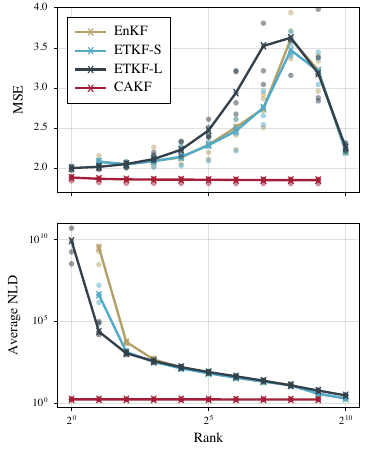}
  \caption{
    Comparison of the CAKF, the EnKF, and two variants of the ETKF on on-model data while varying the rank parameters that govern the computational budget of the algorithms.
  }
  \label{fig:ensemble-on-model-rank}
\end{wrapfigure}

In this experiment, we consider a two spatial dimensions, i.e., $\spatialinputspacedim = 2$, and set $\tmax = 5$.
This results in a state-space dimension of $\gmpdim = \num{20000}$, $\ntsall = \num{100}$ total time steps taken, and $\nxstrain{\idxdtime} = 400$ spatial observations made at each of the $\ntstrain = 10$ training time points.

The ensemble in the EnKF is initialized by drawing $r$ samples from the initial state $\gmp_1$ of the model.
In the predict step, we draw one new ensemble member from the transition model starting at the corresponding ensemble member in the previous ensemble.
The ETKF-S uses the same sample-based initialization and prediction steps as the EnKF.
Both algorithms leverage the (left) square root of $\covfnspace(\xsall, \xsall)$ computed during data generation at every step, and the wall time duration of its computation is added to the wall time cost of both algorithms in \cref{fig:ensemble-on-model-wall-time}.
The ETKF-L uses the Lanczos process to compute a low-rank approximation to the initial covariance during initialization and to the predictive covariance during the prediction step.
The random number generators in the ensemble Kalman filters is seeded with a deterministic transformation of the data seed.

\textbf{Comparison to Kalman filter and RTS smoother}
\begin{figure}[b]
  \begin{subfigure}[t]{0.49\textwidth}
    \centering
    \includegraphics[width=\textwidth]{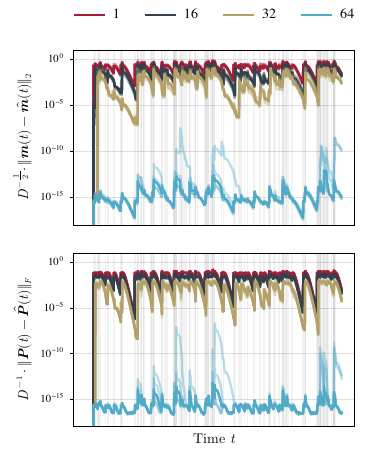}
    \caption{Filter}
    \label{fig:error-dynamics-filter}
  \end{subfigure}
  \begin{subfigure}[t]{0.49\textwidth}
    \centering
    \includegraphics[width=\textwidth]{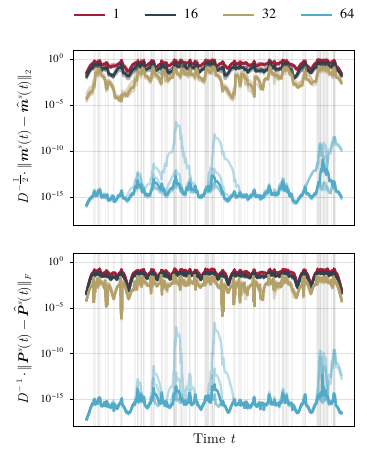}
    \caption{Smoother}
    \label{fig:error-dynamics-smoother}
  \end{subfigure}
  \caption{
    Error dynamics of the CAKF and CAKS with varying rank parameter $\cakfmaxprojobsdim_\idxdtime = \cakfmaxtddrank_\idxdtime = 1, 16, 32, 64$ compared to the Kalman filter and RTS smoother on on-model data.
    The time points at which data is observed are marked by the vertical grid lines.
  }
  \label{fig:error-dynamics}
\end{figure}

In this experiment, we consider a one spatial dimension, i.e., $\spatialinputspacedim = 1$, and set $\tmax = 50$.
This results in a state-space dimension of $\gmpdim = \num{200}$, $\ntsall = \num{1000}$ total time steps taken, and $\nxstrain{\idxdtime} = 20$ spatial observations made at each of the $\ntstrain = 100$ training time points.
Moreover, we fix one random set of training time points that is shared across runs to make the trajectories comparable.

\subsubsection{Impact of Truncation}
\label{sec:appendix-synthetic-data}

The temporal domain of the problem is $[0, 1]$, while the spatial domain is $[0, \pi]$.

\textbf{Data}
We generate the synthetic data on a regular grid of size $11 \times 16$ (in time and space) and the plots are generated on a $51 \times 158$ regular grid.
The training data is corrupted by i.i.d.~zero-mean Gaussian measurement noise with standard deviation $\lambda = 0.1$.

\textbf{Model}
The temporal and spatial covariance functions are given by $\covfntime = \sigma^2 \cdot \text{Matérn}(\nicefrac{3}{2}, \ell_t)$ and $\covfnspace = \text{Matérn}(\nicefrac{5}{2}, \ell_\vx)$ with lengthscales $\ell_t = 0.5$ and $\ell_\vx = 2$, respectively, and an output scale $\sigma = 1$.
The resulting state-space dimension of the spatially discretized model is $\gmpdim = 348$.

\subsubsection{Large-Scale Climate Dataset}
\label{sec:appendix-climate-data}
\begin{table}
  \centering
  \caption{Total number of spatial points $\nxsall=\abs{\xsall}$, state-space dimension $\gmpdim$, number of spatial training points per time step $\nxstrain{\idxdtime} = \abs{\xstrain{\idxdtime}}$, and total number of training points $\ntraindata = \sum_{\idxdtime = 1}^\ntstrain \nxstrain{\idxdtime}$ for the ERA5 experiment.}
  \small
  \begin{tabular}{r S[table-format=6.0] S[table-format=6.0] S[table-format=5.0] S[table-format=7.0]}
    \toprule
    Downsampling factor    & {$\nxsall$} & {$\gmpdim$} & {$\nxstrain{\idxdtime}$} & {$\ntraindata$} \\
    \midrule
    $(\nicefrac{1}{24})^2$ & 1860        & 3720        & 1440                     & 69120           \\
    $(\nicefrac{1}{12})^2$ & 7320        & 14640       & 5580                     & 267840          \\
    $(\nicefrac{1}{6})^2$  & 29040       & 58080       & 21960                    & 1054080         \\
    $(\nicefrac{1}{3})^2$  & 115680      & 231360      & 87120                    & 4181760         \\
    \bottomrule
  \end{tabular}
  \label{tbl:era5-sizes}
\end{table}

The domain of the problem is $[\qty{0}{h}, \qty{48}{h}] \times \spacesym{S}^2(r_\oplus)$, where $\spacesym{S}^2(r)$ denotes the two-dimensional sphere with radius $r \in \R_{\ge 0}$, and $r_\oplus \approx \qty{6371}{\kilo\meter}$ is the average Earth radius.

\textbf{Model}
The temporal covariance function is given by $\covfntime = \sigma^2 \cdot \text{Matérn}(\nicefrac{3}{2}, \ell_t)$ with lengthscale $\ell_t = \qty{3}{\hour}$ and output scale $\sigma = \qty{10}{\celsius}$.
The spatial covariance function is chosen as an extrinsic $\text{Matérn}(\nicefrac{3}{2}, \ell_\vx)$ covariance function on $\spacesym{S}^2(r_\oplus)$, i.e., a $\text{Matérn}(\nicefrac{3}{2}, \ell_\vx)$ covariance function on $\R^3$ concatenated with a coordinate transformation from geographic coordinates to $\R^3$ in both arguments.
For any given spatial downsampling factor, its lengthscale $\ell_\vx$ is set to the geodesic distance of the training points at the equator.
We assume the data is corrupted by iid Gaussian noise with standard deviation \(\qty{0.1}{\degreeCelsius}\).
These hyperparameters were chosen a priori and not tuned for the given training data.

We provide the total number of spatial points, state-space dimension, number of spatial training points per time step, and total number of training points for every downsampling factor of the ERA5 experiment in \cref{tbl:era5-sizes}.

\subsection{Policy Choice}
\label{sec:details-policy-choice}

In general, the choice of optimal policy may be highly problem-dependent, but some natural choices present themselves.

\paragraph{Coordinate Actions}
The simplest choice of policy produces a sequence of unit vectors with all zero entries, except for a single coordinate \(j(i)\), i.e., \(\cakfact_\idxdtime\iidx[i] = \ve_{j(i)}\).
Choosing \(j(i) = i\) simply corresponds to sequential conditioning on a subset of the components of $\obs_\idxdtime$, which in the spatiotemporal regression setting correspond to a subset of spatial locations.
When the data has spatial structure, e.g., when it is placed on a grid as in \Cref{sec:climate-dataset}, this structure can be leveraged by choosing a space-filling sequence of points.
\citet{Berberidis2017} similar to our work explored low-dimensional projections to accelerate Kalman filtering.
They propose several effective policy choices, among them a coordinate policy, which is based on a computable measure of the amount of information in each component of $\obs_\idxdtime$, allowing one to select informative points sequentially while omitting uninformative data points.

\paragraph{Randomized Actions}
\citet{Berberidis2017} also proposed using randomized actions inspired by \emph{sketching} techniques in randomized numerical linear algebra \citep{Martinsson2020RandomisedLinalg}.
For example, a common choice are actions with i.i.d.
sampled entries, e.g., $\cakfact_\idxdtime\iidx[i] \sim \gaussian{\vec{0}}{\mI}$.

\paragraph{Bayesian Experimental Design}
Another choice is to use Bayesian experimental design \citep[see e.g.,][]{Berger1980DecisionTheory}.
This has been explored before in the context of probabilistic linear solvers \citep{Cockayne2019BayesCG} and was found to suffer from slow convergence since these are \emph{a-priori} optimal and therefore do not adapt well to the specific problem.
Furthermore, such optimal actions are not always tractable to compute.

\paragraph{CG/Lanczos Actions}
Finally, a choice that has been repeatedly proposed in the literature on probabilistic linear solvers \citep{Hennig2015LinearSolvers,Cockayne2019BayesCG,wenger2020problinsolve,Wenger2022PosteriorComputational} is the Lanczos/CG algorithm \citep[Section 6.6]{Saad2003}.
In this approach, we obtain the vectors $\cakfact_\idxdtime\iidx[1], \dots, \cakfact_\idxdtime\iidx[\cakfprojobsdim_k]$ by appling the Lanczos procedure to the matrix $\cakfgram_\idxdtime = \obsH_\idxdtime \cakfpcov_\idxdtime \obsH_\idxdtime\T + \obsnoisecov$, with initial vector $\cakfresidual_\idxdtime^{(0)} = \obs_\idxdtime - \obsH_\idxdtime \cakfpmean_\idxdtime$.
This has the effect of ensuring that the residuals $\cakfresidual_\idxdtime\iidx \to 0$ at an exponential rate in $i$ \citep[see e.g.,][Corollary~5.6.7]{Liesen2012}.
As shown by \citet[][Cor.~S2]{Wenger2022PosteriorComputational}, this choice is equivalent to directly selecting the current residual as the next action, i.e., $\cakfact_\idxdtime\iidx[i] = \cakfresidual_\idxdtime\iidx$.

\subsubsection{Empirical Comparison of Policies}
\label{sec:empirical-policy-comparison}

\begin{figure*}
  \centering
  \includegraphics[width=\textwidth]{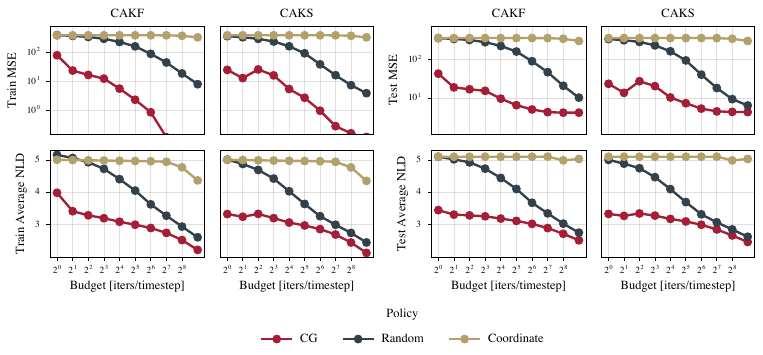}
  \caption{\emph{Comparison of different policies for the CAKF and CAKS on the ERA5 climate dataset.}
    The work-precision diagrams measuring MSE and average NLD on the train and test set universally show that CG actions achieve lower error as a function of the budget when compared to either coordinate or random actions.
  }
  \label{fig:work-precision-policies}
\end{figure*}

To empirically compare some of the proposed policy choices for the CAKF and CAKS, we rerun the experiment on the ERA5 climate dataset with a downsampling factor of 12 ($\gmpdim = \num{14640}$, orange line in \cref{fig:work-precision-era5}) using three different policies selected from the above choices.
We compare coordinate actions with coordinates corresponding to spatial locations chosen according to an (approximately) space-filling design, random actions obtained by drawing independent samples from a standard normal distribution, i.e., $\cakfact_\idxdtime\iidx \sim \gaussian{\vec{0}}{\mI}$, and finally CG/Lanczos actions given by $\cakfact_\idxdtime\iidx = \cakfresidual_\idxdtime\iidx$.
The corresponding work-precision diagrams are shown in \Cref{fig:work-precision-policies}.
It shows that CG actions are preferable to the other actions considered in terms of the MSE and average NLD achieved on both train and test sets.
Note that the blue lines in \Cref{fig:work-precision-policies} coincide with the orange lines in \Cref{fig:work-precision-era5}.
In particular, the exponential convergence rate of CG can be seen in the top-left panel (the CG MSE terminates just below $10^{-4}$ for $2^8$ iterations per time step as can be seen in \Cref{fig:work-precision-era5}).

\subsection{On Comparison with the EnKF}
\label{sec:comparison-enkf}
\begin{wrapfigure}[20]{R}{0.3\textwidth}
  \centering
  \includegraphics[width=0.3\textwidth]{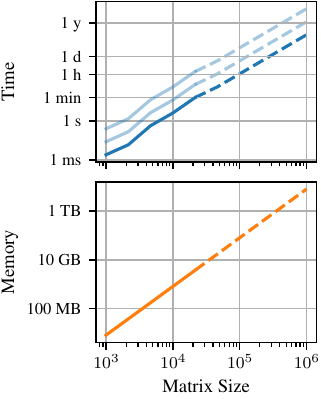}
  \caption{Computational time and memory cost of a Cholesky decomposition in double precision as a function of matrix size.}
  \label{fig:cholesky-complexity}
\end{wrapfigure}

While the ensemble Kalman filter (EnKF) \citep{Evensen1994EnKF} and variants such as the ensemble adjustment Kalman filter (EAKF) \citep{Anderson2001EnsembleAdjustment} and the ensemble transform Kalman filter (ETKF) \citep{Bishop2001AdaptiveSampling} can sometimes be used to solve filtering problems with high-dimensional state spaces, they typically make strong assumptions about the state-space models that are not fulfilled in the setting of this paper.
To be precise, ensemble Kalman filters often assume that we can sample from the initial state $\gmp_0 \sim \gaussian{\gmpmean_0}{\gmpcov_0}$ and the process noise $\gmpnoise_\idxdtime \sim \gaussian{\vec{0}}{\gmpnoisecov_\idxdtime}$ in the dynamics model \citep[][Eqn.~18]{Burgers1998AnalysisScheme} and/or that it is feasible to compute (left) square roots of the respective covariances $\gmpcov_0$ and $\gmpnoisecov_\idxdtime$.

In the data assimilation context, it is common that either:
\begin{itemize}[itemsep=0pt,topsep=0pt]
  \item $\gmpnoisecov_\idxdtime = \mat{0}$
  \item $\gmpnoisecov_\idxdtime$ is diagonal, or
  \item $\gmpnoisecov_\idxdtime$ is low-rank with given left square root.
\end{itemize}

In separable spatiotemporal GP regression, we have $\gmpnoisecov_\idxdtime = \gmpnoisecov_\idxdtime^t \otimes \covfnspace(\xsall, \xsall)$, where $\covfnspace(\xsall, \xsall)$ is a dense kernel Gram matrix, whose (left) square root is (generally) not available without allocating $\covfnspace(\xsall, \xsall)$ in memory.
This means that most ensemble Kalman filters would still require computing a Cholesky factorisation of $\covfnspace(\xsall, \xsall)$, and would thus have $\mathcal{O}(\nxsall^3)$ time and $\mathcal{O}(\nxsall^2)$ memory complexities, which are precisely the complexities that the CAKF/S are constructed to avoid.
We encounter analogous issues when considering $\gmpcov_0$.

To make this point more clearly, \Cref{fig:cholesky-complexity} shows how time and memory requirements to compute a Cholesky decomposition of $\covfnspace(\xsall, \xsall)$ scale with matrix size \(\nxsall\).
The solid lines correspond to what decompositions we are able to compute on our machines for the different problem sizes in \cref{tbl:era5-sizes}, and the dashed lines are extrapolations past the problem sizes where the decomposition crashed due to memory allocation errors.
The largest problem size we apply our methods to has $\nxsall=\num{115680}$ spatial observations, which is comfortably outside the range where a Cholesky decomposition can be computed on all but the most specialised hardware; just to allocate this matrix would require in excess of \qty{100}{\giga\byte} of working memory.

An exception to the above is the seemingly seldom-used ETKF-L algorithm described in \citep[][Sec.~3(b)]{Tippett2003EnsembleSquare} as well as in \cref{sec:experiment-ensemble-on-model,sec:appendix-on-model-data}.
The ETKF-L approximates the required (left) square roots with the Lanczos method without allocating $\covfnspace(\xsall, \xsall)$ in memory.

\stopcontents[sections]

\end{document}